\documentclass[a4paper,11pt]{article}% arXiv preprint version

\usepackage{fullpage}

\usepackage{graphicx}%
\usepackage{multirow}%
\usepackage{amsmath,amssymb,amsfonts}%
\usepackage{amsthm}%
\usepackage{mathrsfs}%
\usepackage[title]{appendix}%
\usepackage{xcolor}%
\usepackage{textcomp}%
\usepackage{manyfoot}%
\usepackage{booktabs}%
\usepackage{algorithm}%
\usepackage{algorithmicx}%
\usepackage{algpseudocode}%
\algtext*{EndFor}
\usepackage{listings}%
\usepackage{enumitem}

\usepackage{setspace}
\usepackage{adjustbox}
\usepackage{lmodern}

\usepackage{tikz} % Diagrams
\usepackage{wrapfig} % Wrapped figures
\usepackage{rotating} % Rotated figures
\usepackage{pdfpages} % Pdf insertion
\usepackage{fancybox} % Boxes
\usepackage{microtype} % Bad boxes
\usepackage{minitoc}

\usepackage[justification=centering]{caption} % Captions for figures
\usepackage[justification=centering]{subcaption} % Captions for subfigures

\usepackage{comment} % Block comment
\usepackage{epigraph} % Quotes
\usepackage{csquotes} % English quotation marks
\usepackage{lipsum}  % Nonsense
\usepackage{xspace}
\usepackage{hyphenat}  % Hyphenation
\usepackage{blindtext}

\usepackage{array} % Arrays
\usepackage{multicol} % Multiple columns
\usepackage{multirow} % Multiple rows
\usepackage{makecell} % Merge table cells
\usepackage{booktabs} % Nice tables

\usepackage{mathtools}
\usepackage{bm} % Bold math
\usepackage{bbm}
\usepackage{stmaryrd}
\usepackage{cancel} % Strike

\usepackage{xurl} % Break URLs whenever

\definecolor{juliablue}{rgb}{0.251,0.388,0.847}
\definecolor{juliagreen}{rgb}{0.220,0.596,0.149}
\definecolor{juliapurple}{rgb}{0.584,0.345,0.698}
\definecolor{juliared}{rgb}{0.796,0.235,0.200}

\usepackage{thmtools, thm-restate}

\usepackage[
	backend=biber,
	giveninits=true,
	style=authoryear,
	date=year,
	uniquename=init,
	maxbibnames=99,
	maxcitenames=1,
	url=false
]{biblatex} % Bibliography
\usepackage{bookmark}
\hypersetup{colorlinks=true, breaklinks=true, linkcolor=black, citecolor=black, urlcolor=black}

\newcommand{\ie}{\textit{i.e.},\xspace}

\newif\ifshort
\shorttrue
\ifdefined\shortbuild\shorttrue\fi

\newcommand{\one}{\mathbf{1}}
\newcommand{\oneindicator}{\mathbbm{1}}

\DeclareMathOperator{\conv}{conv}

\DeclareMathOperator*{\argmin}{argmin}
\DeclareMathOperator*{\argmax}{argmax}
\DeclareMathOperator{\dom}{dom}

\DeclareMathOperator*{\inte}{int}
\DeclareMathOperator*{\relint}{rel\,int}
\DeclareMathOperator*{\cl}{cl}
\DeclareMathOperator*{\aff}{aff}

\DeclareMathOperator*{\vect}{span}

\DeclareMathOperator*{\im}{Im}
\DeclareMathOperator{\bdry}{bdry}

\newcommand{\bfr}{\mathbf{r}}

\newcommand{\bfx}{\mathbf{x}}
\newcommand{\bfxi}{\boldsymbol{\xi}}

\newcommand{\bfy}{\mathbf{y}}

\newcommand{\bbC}{\mathbb{C}}

\newcommand{\bbE}{\mathbb{E}}

\newcommand{\bbI}{\mathbb{I}}

\newcommand{\bbN}{\mathbb{N}}

\newcommand{\bbP}{\mathbb{P}}

\newcommand{\bbR}{\mathbb{R}}

\newcommand{\bfE}{\mathbf{E}}

\newcommand{\bfZ}{\mathbf{Z}}

\newcommand{\calC}{\mathcal{C}}
\newcommand{\calD}{\mathcal{D}}

\newcommand{\calH}{\mathcal{H}}

\newcommand{\calL}{\mathcal{L}}
\newcommand{\calM}{\mathcal{M}}
\newcommand{\calN}{\mathcal{N}}
\newcommand{\calO}{\mathcal{O}}
\newcommand{\calP}{\mathcal{P}}

\newcommand{\calR}{\mathcal{R}}
\newcommand{\calS}{\mathcal{S}}

\newcommand{\calU}{\mathcal{U}}

\newcommand{\calW}{\mathcal{W}}
\newcommand{\calX}{\mathcal{X}}
\newcommand{\calY}{\mathcal{Y}}

\newtheorem{theorem}{Theorem}%  meant for continuous numbers
\newtheorem{proposition}[theorem]{Proposition}% 
\theoremstyle{plain}

\newtheorem{lem}[theorem]{Lemma}
\newtheorem{ass}{Assumption}
\newtheorem{result}{Result}

\newtheorem{example}{Example}%
\newtheorem{remark}{Remark}%

\usetikzlibrary{3d, arrows.meta, positioning, shapes.geometric}

\hypersetup{bookmarksdepth=3}

\begin{document}

\title{Primal-dual algorithm for contextual stochastic combinatorial optimization}

\author{Louis Bouvier$^1$, Thibault Prunet$^{2,3}$, Vincent Leclère$^2$, Axel Parmentier$^2$\\[0.5em]
\normalsize
$^1$ Co-innovation lab, ENPC, Institut Polytechnique de Paris, Marne-la-Vallée, France\\
$^2$ CERMICS, CNRS, ENPC, Institut Polytechnique de Paris, Marne-la-Vallée, France\\
$^3$ Université de Bordeaux, CNRS, INRIA, IMB, UMR 5251, Talence, France\\[0.5em]
\small\texttt{\{louis.bouvier,vincent.leclere,axel.parmentier\}@enpc.fr}\\
\small\texttt{thibault.prunet@u-bordeaux.fr}}
\date{\today}

\maketitle

\begin{abstract}
This paper introduces a novel approach to contextual stochastic optimization, integrating operations research and machine learning to address decision-making under uncertainty. Traditional methods often fail to leverage contextual information, which underscores the necessity for new algorithms. In this study, we utilize neural networks with combinatorial optimization layers to encode policies. Our goal is to minimize the empirical cost, which is estimated from past data on uncertain parameters and contexts. To that end, we present a surrogate learning problem and a generic primal-dual algorithm that is applicable to various combinatorial settings in stochastic optimization. Our approach extends classic Fenchel--Young loss results and introduces a new regularization method using sparse perturbations on the distribution simplex. This allows for tractable updates in the original space and can accommodate diverse objective functions.
We establish sublinear convergence for the exact linear-parametric version and provide a bound on the non-optimality of the resulting policy in terms of the empirical cost. Experiments on three contextual stochastic optimization problems show that our algorithm is efficient and scalable, achieving performance comparable to state-of-the-art baselines with significantly reduced computational requirements.
\end{abstract}


\section{Introduction}\label{sec:intro}

    Consider a decision maker whose choice is affected by some random noise $\bfxi \in \Xi$.
    The decision maker does not know $\bfxi$ when he takes his decision, but has access to a realization~$x$ of a \emph{context} variable~$\bfx$ correlated to $\bfxi$.
    The context space $\calX$ is the set of all possible context realizations.
    Based on a context realization $x \in \calX$, the decision maker takes a decision $y$ in $\calY(x)$.
    To that purpose, he chooses a policy $\pi$ that maps a context realization $x$ to a decision $y \in \calY(x)$. 
    We do not require the policy to be deterministic and can, therefore, see it as a conditional distribution $\pi(y|x)$ over $\calY(x)$ given $x$. 
    Assuming that the policy $\pi$ has to belong to some hypothesis class $\calH$, 
    our  \emph{contextual stochastic optimization} problem \parencite{sadanaSurveyContextualOptimization2024} aims at finding a policy $\pi$ that minimizes the \emph{expected cost} $\calR$, which is the expected cost under $\pi$.
    \begin{equation}
        \label{eq:contextualStochasticOptimization}
        \min_{\pi \in \calH} \calR(\pi) \quad \text{where} \quad \calR(\pi) = \bbE_{(\bfx, \bfxi), \bfy \sim \pi(\cdot|\bfx)}\big[c(\bfx, \bfy,\bfxi)\big].
    \end{equation}
    
    The expectation is taken with respect to the
    distribution over $(\bfx,\bfy,\bfxi)$ that derives from the 
    joint distribution over $(\bfx,\bfxi)$ and the policy $\pi$.
    Since the decision maker does not have access to~$\bfxi$, the decision $\bfy$ is independent of $\bfxi$ given the context $\bfx$.
    In many situations, the noise $\xi$ is observed once the decision has been taken, and the training set comes from historical data; we thus place ourselves in a \emph{learning} setting. 
    \begin{ass}\label{ass:learning_dataset}
        We do not know the joint distribution over $(\bfx,\bfxi)$. But 
    we have access to a training set $(x_1,\xi_1),\ldots,(x_N, \xi_N)$ of independent samples of $(\bfx,\bfxi)$.
    \end{ass}
    
    In this work, we focus on the combinatorial case where, for any context realization $x \in \calX$, the set of admissible decisions $\calY(x)$ is finite but potentially combinatorially large, as formalized in the following assumption that holds throughout the paper.

    \begin{ass}
    \label{rem:exposed_vertices}
    For every possible context $x$, the set of admissible decisions ${\calY(x)\subset \bbR^{d(x)}}$ is finite.
    Further, we assume that $\calY(x)$ is the set of extreme points of its convex envelope $\calC(x)= \conv \left(\calY(x)\right)$.
    \end{ass}
    {As $\calY(x)$ is the set of extreme points of a polytope, for any $\bar y$ in $\calY(x)$, there exists a $\theta \in \bbR^{d(x)}$ such that $\bar y$ is the unique argmax of $\max_{y\in \calY(x)}\langle\theta | y\rangle$, which allows us to build policies based on linear optimizers. 
    Assumption~\ref{rem:exposed_vertices} may seem restrictive at first glance, but the class of problems satisfying the conditions is actually very large. In particular, it includes all 0-1 optimization problems, which are ubiquitous in mathematical programming applications.
    We also emphasize that the set of feasible decisions $\calY(x)$ is context-dependent, a feature that few contextual stochastic methods can handle.}

    \paragraph{Stochastic optimization policies.}
    For a stochastic optimization approach, one would typically build a policy $\pi$ by solving the stochastic optimization problem that arises by taking the conditional expectation over $\bfxi$ given $\bfx = x$.
    \begin{equation}\label{eq:conditionalSto}
        \min_{y \in \calY(x)}\bbE_{\bfxi}\big[c(\bfx, y,\bfxi)\big| \bfx = x\big].
    \end{equation}
    Practical approaches typically solve a sample average approximation of Equation~\eqref{eq:conditionalSto}. 
    Decomposition-coordination methods such as progressive hedging  \parencite{RW1991}
    solve thousands of instances of a deterministic (single scenario) problem of the form
    \begin{equation}
        \label{eq:DeterministicProblem}
     \min_{y\in \calY(x)} c\big(x(\omega) , y,\xi(\omega)\big) + \langle \theta| y \rangle,
    \end{equation}
    where $\theta$ is a dual vector, such as a vector of Lagrange multipliers.
    Our combinatorial and large dimensional setting in~$\bbR^{d(x)}$ brings two challenges.
    First, we do not know the distribution over $(\bfx,\bfxi)$. We may learn a model, but large dimensional $\bfx$ and $\bfxi$ require a large training set, which we do not always have in industrial settings.
    Second, the computational burden required by such algorithms becomes significant and prevents them from being applied online in a contextual setting, where the computing time is limited.
    
    \paragraph{Parametrized family of policies based on a combinatorial optimization layer.}
    We therefore propose a change of paradigm.
    We instead define a hypothesis class $\calH_\calW$ of policies $\pi_w$ parameterized by $w$ in $\calW$.
    Policies in $\calH_\calW$ are chosen to be fast enough to be used online.

    Working with a combinatorial solution space~$\calY(x)$ makes the choice of $\pi$ challenging. 
    Indeed, there are few statistically relevant, while computationally tractable, models from a combinatorial set to another.
    We rely on a combinatorial optimization (CO) layer to build such policies. 
    We build upon recent contributions \parencite{blondelLearningFenchelYoungLosses2020a, berthetLearningDifferentiablePertubed2020, dalleLearningCombinatorialOptimization2022, parmentierLearningApproximateIndustrial2022a}
    that derive from the following regularized linear optimization problem 
    \begin{equation}\label{eq:COlayer}
        \max_{y \in \calC(x)} \langle \theta| y \rangle - \Omega_{\calC(x)}(y),
    \end{equation}
    a conditional distribution $ p_{\Omega_{\calC(x)}}(\cdot|\theta)$ on $\calY(x)$ (see Section~\ref{subsec:minimize_empirical_risk}), where~${\calC(x)=\conv(\calY(x))}$.
    Here, ${\Omega_{\calC(x)} : \dom(\Omega_{\calC(x)}) \rightarrow \bbR}$ is a proper convex lower-semicontinuous regularization function, with ${\calC(x) \subseteq \cl\big(\dom(\Omega_{\calC(x)})\big)}$.
    The simplest such regularization is $\Omega_{\calC(x)} = 0$, in which case we obtain a Dirac on the $\argmax$ (if unique).
    This model is parameterized by $\theta$, the direction of the linear term.
    In our policies, we use a statistical model $\varphi_w$, typically a neural network parameterized by $w \in \calW \subseteq \bbR^{n_w}$ to predict $\theta$ from the context $x$. 
     In other words, we seek policies in the hypothesis class 
    \begin{equation}\label{eq:dl_policy}
        \calH_{\calW}  = \Big\{\pi_w\colon w \in \calW\Big\} \quad \text{where} \quad \pi_w(y|x) = p_{\Omega_{\calC(x)}}\big(y|\varphi_w(x)\big),
    \end{equation}
    where $\varphi_w : x \in  \calX \mapsto \theta \in \bbR^{d(x)}$ is a statistical model.

    \paragraph{Empirical cost minimization.}
    { For such policies to work, we need  a learning algorithm that leverages the training set to find a policy $\pi_w$ in $\calH_\calW$ with a low expected cost $\calR(\pi_w)$}.
    Practically we solve the \emph{empirical cost minimization problem},
    \begin{equation}
        \label{eq:Empiricalrisk}
        \min_{w \in \calW} R_N( \pi_w) \quad \text{where} \quad R_N(\pi_w):= \frac{1}{N}\sum_{i=1}^N \bbE_{\bfy \sim \pi_w(\cdot|x_i)}\Big[c(x_i, \bfy, \xi_i)\Big].
    \end{equation}
    {which is just the empirical version on the training set of~\eqref{eq:conditionalSto} where we restrict ourselves to policies in $\calH_{\calW}$.
    For several regularization functions~$\Omega$, stochastic gradients can be computed, and the problem~\eqref{eq:Empiricalrisk} is amenable to stochastic gradient descent. However, the results obtained tend to be poor as first the gradient estimates are noisy and the objective non-convex~\parencite{dalleLearningCombinatorialOptimization2022}. To that purpose, we introduce a surrogate problem that approximates~\eqref{eq:Empiricalrisk} and for which we can derive a better behaved minimization algorithm.
    This algorithm relies on the following assumption---which is common in (contextual) stochastic optimization---to derive a tractable alternating minimization algorithm.
    \begin{ass}\label{ass:oracle}
        We suppose having a reasonably efficient algorithm to solve the deterministic single scenario problem~\eqref{eq:DeterministicProblem}.
    \end{ass}
    Before introducing our algorithm, let us first focus on the main approach in the literature to train policies~\eqref{eq:dl_policy}:  supervised learning. This will both motivate the use of empirical cost minimization and allow us to introduce a mathematical tool needed to define our surrogate.
    }

    \paragraph{Supervised learning and Fenchel--Young losses.}
    {
    Supervised learning requires a training set $(x_1,\bar y_1),\ldots,(x_N,\bar y_N)$ with a target decision $\bar y_i$ to imitate, and minimizes the expectation of a loss $\calL$ that evaluates how far the prediction $\pi_w(x_i)$ is from the target $\bar y_i$. If using a standard loss such as the squared Euclidean distance, stochastic gradient descent on the supervised learning problem typically suffers from the same non-convexity and noisy gradients.
    Several losses have been proposed to address these issues, including Fenchel--Young losses.}

    Given a regularization function~$\Omega : \bbR^d \rightarrow \bbR\cup\{+\infty\}$,
    the Fenchel--Young loss
     $\calL_\Omega(\theta;\bar y)$ generated by $\Omega$ \parencite{blondelLearningFenchelYoungLosses2020a} is defined over $\dom(\Omega^*) \times \dom(\Omega) $ as
     \begin{align}\label{eq:FYloss_def}
        \calL_\Omega(\theta;\bar y) &:= \Omega^*(\theta) + \Omega(\bar y) - \langle \theta| \bar y \rangle \nonumber\\
        &= \sup_{y \in \dom(\Omega)}\big(\langle \theta | y \rangle - \Omega(y)\big) - \big(\langle \theta | \bar y \rangle - \Omega(\bar y)\big).
    \end{align}
    It measures the difference between the solution~$y_\theta$ of Equation~\eqref{eq:COlayer} and a target $\bar y \in \calY(x) \subset \dom(\Omega)$, as the non-optimality of the target for this problem. 
    Such a loss is typically convex in~$\theta$, nonnegative, and equal to~$0$ if and only if $p_{\Omega}(\cdot|\theta)$ is a Dirac in $\bar y$.
    {Note that $\calL_\Omega(\theta;\bar y)$ is the gap in the Fenchel--Young inequality of convex analysis}.
    During the last few years, they have become the main approach for supervised training of policies of the form~\eqref{eq:dl_policy} as they lead to a tractable {(with low variance pathwise gradient estimates)} and convex learning problem. 
    Under some hypotheses, they happen to coincide with Bregman divergences and are a key element in our expected cost minimization algorithm.

    {
    In our contextual setting, we may lack good targets $\bar y_i$ to imitate. One natural approach under Assumption~\ref{ass:oracle} is to use our deterministic oracle to get an anticipative decision $\bar y_i \in \argmin_{y \in \calY(x_i)}c(x_i,y,\xi_i)$. While the approach has been successful on some problems~\parencite{batyCombinatorialOptimizationEnrichedMachine2024}, it is well known in stochastic optimization that such decisions can be arbitrarily far from the best non-anticipative decisions. Our goal is to provide a better learning approach based on empirical cost minimization for such problems.
    }
    
    \paragraph{Alternating minimization algorithm.}
    
    {We now have all the tools to introduce our \emph{surrogate problem} to Problem~\eqref{eq:Empiricalrisk}}
    \begin{equation}
        \calS\big(w,y_{\otimes}\big) = \frac1N \sum_{i=1}^N c(x_i,y_i,\xi_i) + \kappa \calL_\Omega(\theta_i,y_i), \quad \text{with} \quad \begin{cases}
            \theta_i = \varphi_w(x_i), \\
            y_\otimes= (y_i)_{i \in [N]} \in \calY_\otimes,
        \end{cases}
    \end{equation}
    where $[N]$ denotes the set~$\{1, \ldots, N\}$, $\kappa > 0$ is a positive constant, and $\calL$ is a Fenchel--Young loss
    and 
    \[ \calY_\otimes=\Big\{(y_i)_{i \in [N]}\colon y_i \in \calY(x_i) \text{ for each }i\Big\}. \]
    {When $\kappa \rightarrow \infty$, minimizing over $y_\otimes$ leads to taking $y_i = \argmax \langle \theta_i,y_i\rangle$, and we fall back on our empirical cost minimization problem~\eqref{eq:Empiricalrisk}. When $\kappa$ is finite, we get a relaxation whose error to $\calR(w)$ is in $\frac{1}{\kappa}$.}

    Our learning algorithm minimizes~$\calS\big(w,y_{\otimes}\big)$ using alternating minimization. 
    \begin{subequations}
        \begin{align}
             y_{\otimes}^{(t+1)}&= \argmin_{y_{\otimes}} \calS(w^{(t)},y_{\otimes}),  && \text{(Decomposition)}, \label{eq:intro_deco_step} \\
             w^{(t+1)} &\in \argmin_{w \in \calW} \calS(w,y_{\otimes}^{(t+1)}), && \text{(Coordination)}, \label{eq:intro_step_coord}
        \end{align}
    \end{subequations}
    where $y_{\otimes}^{(t+1)}= (y_i^{(t+1)})_{i \in [N]}$.
    In~\eqref{eq:intro_deco_step}, we do not ask that $y_{\otimes}$ belongs to $\calY_\otimes$ as we optimize in practice on a continuous space that contains $\calY_\otimes$.
    Indeed, to make this algorithm practical on combinatorial spaces, we need to work on the space of distribution over~$\calY(x_i)$, which requires some technical preliminaries.
    We therefore postpone the precise definition to Section~\ref{sec:main_results}.
    Suffice it to say at this point that Step~\eqref{eq:intro_deco_step} decomposes per scenario and requires solving deterministic single-scenario problems of the form~\eqref{eq:DeterministicProblem}, and that the coordination step~\eqref{eq:intro_step_coord} amounts to a supervised learning problem with a Fenchel--Young loss, for which efficient algorithms exist.
    
    \paragraph{Related works.}
    
{
Our approach can be related to the proximal point algorithm (PPA), a classical method for finding zeros of maximal monotone operators popularized by the seminal work of \textcite{rockafellar1976monotone}. \textcite{eckstein1993nonlinear} extended the PPA by replacing the Euclidean distance penalty with a Bregman divergence, giving rise to the \emph{Bregman PPA}, and later refined this framework to allow inexact subproblems and double regularization \parencite{Eckstein2009}. More recently, \textcite{Zaslavski2011} established convergence of the (Bregman) PPA even in the presence of computational errors at each iteration, a setting close to ours since our coordination step is itself solved only approximately by SGD, and \textcite{Jiang2022} combined the Bregman PPA with operator splitting to handle composite objectives, in the same spirit as our decomposition-coordination scheme. We make this connection precise in Section~\ref{sec:main_results}.}

{
We also rely on alternating minimization algorithms for Equations~\eqref{eq:intro_deco_step}-\eqref{eq:intro_step_coord}, where a two-variable function $\phi : (y,z) \mapsto \phi(y,z)$ is iteratively minimized along one of its coordinates while the other is fixed. 
    The convergence of such algorithms is not a new topic, but the wide range of applications in machine learning and signal processing leads to a revived interest in the recent years \parencite{wrightCoordinateDescentAlgorithms2015}. 
    \textcite{beckConvergenceBlockCoordinate2013} were the first to prove a sublinear rate of convergence in a Euclidean setting when $\phi$ is assumed convex and L-smooth. 
    In the present case, our function $\phi$ is neither convex nor L-smooth. 
    A more general setting, without any structural assumption on $Y$, $Z$ nor $\phi$ is studied by \textcite{legerGradientDescentGeneral2023}. 
    They proved the convergence of the alternating minimization algorithm when $\phi$ satisfies the five-point property, that is a non-local inequality involving $\phi$ evaluated at different points. }
    
    {Alternative proof strategies have been developed based on abstract convergence theorems \parencite{attouch_convergence_2013}. In this work, the authors only assume a function to optimize, and a sequence generated by a descent algorithm, without any assumption on the algorithm used to generate the sequence. \textcite{attouch_convergence_2013} prove the convergence of the sequence toward a critical point with finite length, provided that {\em i.} the sequence follows some descent properties, and {\em ii.} the function satisfies the so called {\em Kurdyka–Łojasiewicz (KL)} property, which ensures its variations are tame in the neighborhoods of critical points.
    We leverage this literature in Appendix~\ref{sec:convergenceProof}.}
    
    \paragraph{Contributions and outline.}
    Our main contribution is to introduce an alternating minimization algorithm for contextual stochastic combinatorial optimization, which has several nice properties.
    \begin{enumerate}
        \item 
        This algorithm relies on sampling, stochastic gradient descent, and automatic differentiation to update a model~$\varphi_w$, and is therefore deep learning-compatible. 
        \item It is \emph{generic}, and can be applied to any setting where assumptions~\ref{ass:learning_dataset}-\ref{ass:oracle} hold.
        It notably provides a generic algorithm to train policies based on neural networks with combinatorial optimization layers for contextual stochastic optimization problems. This notably allows us to deal with problems where the set of feasible solutions $\calY(x)$, and even its dimension, depend on $x$, which is a known difficulty for most contextual stochastic optimization methods in the literature.
        \item {
        We bound the difference between the empirical cost of the solution to the surrogate problem and the optimum of the empirical cost, and thus the non-optimality of the policy returned for the initial problem. 
        When $\varphi_w$ is linear in $w$, we prove the convergence of the exact alternating scheme to a stationary point of the surrogate problem, as well as a convergence speed.
        Reformulating this exact scheme as the proximal point algorithm, we identify when it converges to a global optimum and when it might end-up in a local minimum due to non-convexity.}
        \item {Our numerical experiments, using an approximate algorithm, on three different applications show that our algorithm is practically efficient and scalable in the size of the statistical model~$\varphi_w$, the size~$N$ of the training set, and the dimension of the combinatorial optimization problem. The limiting factor is the size of the deterministic instances that can be solved in Equation~\eqref{eq:DeterministicProblem}. In particular, while its strength is to scale to large-dimensional, context-dependent sets~$\calY(x)$, it reaches the performance of state-of-the-art baselines in contextual stochastic optimization on problems whose size is compatible with those baselines, while being orders of magnitude faster.}
    \end{enumerate}
    The key challenge we face in defining practical versions of these algorithms is to develop tractable regularizations on non-full-dimensional polytopes~$\calC(x)$ and on the distribution simplex over~$\calY(x)$.
    These have practical relevance for supervised learning with Fenchel--Young losses beyond our method.
    \begin{enumerate}[resume]
        \item Based on the work of~\textcite{berthetLearningDifferentiablePertubed2020}, we introduce a new sparse regularization by perturbation on the distribution simplex over~$\calY(x)$. 
        This new regularization is perhaps the key element to obtain a tractable and generic learning algorithm for large combinatorial problems. 
        \item We show that structured supervised learning with a Fenchel--Young loss using a generalized linear model can be seen as directly minimizing a single Fenchel--Young loss on the parameter space. Hence, only the projection of the imitated decisions onto the feature space matters.
        \item We highlight several results on Fenchel--Young losses~\parencite{blondelLearningFenchelYoungLosses2020a} on non-full-dimensional polytopes~$\calC(x)$, and on the distribution simplex over~$\calY(x)$.
        We analyze their links with Legendre-type functions, mirror maps and regularizers.
    \end{enumerate}

    The remainder of the paper is organized as follows. Section~\ref{sec:main_results} presents the primal-dual algorithm and its properties, including tractability, a bound on the surrogate error, and a convergence analysis; proofs of the convergence results are gathered in Appendix~\ref{sec:convergenceProof}.
    Section~\ref{sec:new_structured_prediction} introduces two new theoretical contributions: a sparse perturbation directly on the distribution space~$\Delta^{\calY}$ (Section~\ref{sec:structured_perturbation}), and a characterisation of the geometry induced by aggregating Fenchel--Young losses (Section~\ref{sec:loss_parameter_space}); their proofs are in Appendices~\ref{subsec:appendix_structured_perturbation} and~\ref{subsec:appendix_loss_parameter_space}, respectively.
    Section~\ref{sec:computational_exp_primal_dual} details numerical experiments and Section~\ref{sec:conclusion_primal_dual} concludes.
    Appendix~\ref{sec:regularization_on_distributions} gathers the background on regularization on non-full-dimensional spaces used throughout the paper.
    
    \paragraph{General notations.}
    We denote by $\bbR$ the set of real numbers, and by $\bbR_{++}$ the set of positive real numbers. 
    Let $E$ be an Euclidean space, and $\calX \subset E$ be a set. We denote by $\vect(\calX)$ the span of $\calX$, $\aff(\calX)$ its affine hull, $\inte(\calX)$ its interior, $\cl(\calX)$ its closure, $\bdry(\calX)$ its boundary, and $\relint(\calX)$ its relative interior. 
    We introduce $\bbI_\calX$ the indicator function of the set $\calX$, with value $0$ over $\calX$ and $+\infty$ elsewhere. 
    For two sets $\calX_1$ and $\calX_2$, we denote by $\calX_1 \times \calX_2$ their Cartesian product space, and $\calX_1+\calX_2$ their Minkowski sum. 
    In addition, when $\calX_1$ and $\calX_2$ are vector subspaces of $E$, with $\calX_1 \cap \calX_2 = \{0\}$, we have a direct sum written as $\calX_1 \oplus \calX_2$. 
    We extend this notation to $S_1 \oplus S_2$ to denote $\big\{s_1+s_2\colon s_1\in S_1,\,s_2 \in S_2\}$ given two subsets $S_1$ and $S_2$ of $E$ (not necessarily vector spaces) such that $\langle s_1| s_2 \rangle =0$ for any $s_1 \in S_1$ and $s_2\in S_2$.
    By default $\|\cdot\|$ is the euclidean norm. 
    
    For $E$ an Euclidean space with inner product $\langle \cdot | \cdot \rangle$ and associated norm $||\cdot ||$, we denote by $\Gamma_0(E)$ the set of proper lower-semicontinuous (l.s.c.) convex functions from $E$ to $(-\infty, + \infty]$. For a function $\Psi \in \Gamma_0(E)$, we denote by $\dom(\Psi)$ the domain of $\Psi$, by $\argmin \Psi$ and $\argmax \Psi$ the sets of global  minimizers and maximizers of $\Psi$ (possibly empty), by $\Psi^*:E \to (-\infty, +\infty]$ its Fenchel-conjugate function, $ \Psi^*:y\mapsto \sup_{x \in E} \{\langle x| y \rangle - \Psi(x)\}$, and by $\partial \Psi$ its subdifferential, $\partial \Psi: x\mapsto \{g \in E \, | \, \forall y \in E, \langle y-x| g \rangle + \Psi(x) \leq \Psi(y) \}$.
    
    Given $x \in E$, the mapping $\Psi$ is subdifferentiable at $x$ if $\partial \Psi(x) \neq \emptyset$; the elements of $\partial \Psi(x)$ are the subgradients of $\Psi$ at $x$. If $\Psi$ is differentiable at $x$, we name $\nabla \Psi(x)$ the gradient of $\Psi$ at $x$.
    
    {For a continuously differentiable strictly convex function $F:E\mapsto \bar \bbR$ with closed domain, we denote by $D_F$ the Bregman divergence associated with $F$, defined as $D_F(x,y) = F(x) - F(y) - \langle \nabla F(y)| x-y\rangle$ for $x,y \in \dom(F)$.}
    
    \paragraph{Moment polytope.}
    Let $\calY$ be a finite combinatorial set in $\bbR^d$, and ${\calC = \conv(\calY)}$ be its convex hull. We sometimes refer to~$\calC$ as the moment polytope. As stated above, we assume that no element of $\calY$ is a strict convex combination of other elements of $\calY$.
    In other words, $\calY$ is the set of vertices of the polytope~$\calC$. We denote by $H = \aff(\calY)$ the affine hull of $\calY$, and by $V$ the direction of $H$, a sub-vector space in $\bbR^d$. 
    We have the orthogonal sum $\bbR^d = V \oplus V^\perp$, and we denote by $\Pi_V$ the orthogonal projection onto $V$ in $\bbR^d$. We name $Y$ the wide matrix with vectors $y \in \calY$ as columns. 
    
    \paragraph{Distribution polytope.}
    Let $\Delta^{\calY} := \{q \in \bbR^{\calY}, q \geq 0, \sum_{y \in \calY} q_y = 1\}$ be the probability simplex whose vertices are indexed by $\calY$, and $H_{\Delta}$ its affine hull $H_{\Delta} = \aff(\Delta^{\calY})$. We denote by $V_{\Delta}$ the vector subspace (hyperplane) in $\bbR^{\calY}$ that is the direction of $H_{\Delta}$. As previously, we rely on the orthogonal sum $\bbR^{\calY} = V_\Delta \oplus V_\Delta^\perp$, where here $V_\Delta^\perp = \vect(\mathbf{1})$.
    Let $\theta \in \bbR^d$ be a cost vector and $q \in \Delta^{\calY}$ be a probability distribution, then $s_\theta = Y^\top \theta \in \bbR^{\calY}$ is the vector $(y^\top \theta)_{y \in \calY}$, and $\mu_q = Yq = \sum_y q_y y = \bbE(\bfy|q)$ is the moment vector of the random variable $\bfy$ on $\calY$ with distribution $q$. 
    
    \section{Primal-dual algorithm and its properties}\label{sec:main_results}

    {This section presents the main results of the paper on learning structured policies for contextual stochastic combinatorial optimization.
    Section~\ref{subsec:minimize_empirical_risk} defines the policy class~$\pi_w$ and the empirical cost minimization problem.
    Section~\ref{subsec:surrogate_and_guarantee} introduces a tractable surrogate problem and bounds the error incurred by optimizing it in place of the empirical cost.
    Section~\ref{subsec:alternating_min_tractability} presents the alternating minimization algorithm and establishes its tractability via a moment-space reformulation.
    Section~\ref{subsec:convergence} analyzes convergence of the algorithm when $\varphi_w$ is linear in $w$; the supporting proofs are gathered in Appendix~\ref{sec:convergenceProof}.}
    
    \subsection{Policies over combinatorial spaces and empirical cost minimization}
    \label{subsec:minimize_empirical_risk}

    {Let us now formally introduce our policies $\pi_w$, which are based on Legendre-type functions on the probability simplex $\Delta^{\calY(x)}$ over $\calY(x)$.}
    
    \paragraph{\texorpdfstring{Legendre-type functions (\textcite[Section 26]{Rockafellar+1970})}{Legendre-type functions (Rockafellar, 1970)}.}
    A function ${\Psi : \bbR^d \to \bbR\cup\{+\infty\}}$ is \emph{Legendre-type} if it is strictly convex on $\inte(\dom(\Psi))$ and \emph{essentially smooth}, i.e., 
            i) $\inte(\dom(\Psi))$ is non-empty;
            ii) $\Psi$ is differentiable through $\inte(\dom(\Psi))$;
            iii) ${\lim_{\mu \to \bdry(\dom(\Psi))} ||\nabla \Psi(\mu)|| = + \infty}$.
    
    \begin{remark}\label{rem:theorem_rockafellar_legendre}
        Let $\Psi \in \Gamma_0(\bbR^d)$ be a proper convex l.s.c. function with Fenchel conjugate $\Psi^*$. 
    Then $\Psi$ is a convex function of Legendre type if and only if $\Psi^*$ is a convex function of Legendre type. 
      When these conditions hold, the gradient mapping $\nabla \Psi$ is one-to-one from the open convex set~$\inte\big(\dom(\Psi)\big)$ onto the open convex set~$\inte\big(\dom(\Psi^*)\big)$, continuous in both directions, and $\nabla \Psi^* = (\nabla \Psi)^{-1}.$ 
    As a consequence, the gradient of Legendre-type functions can be used as one-to-one mapping from the primal to the dual space.
    \end{remark}
    
    \paragraph{Mapping scores to distributions.}
    Recall that $\Delta^\calY:=\{q\in[0,1]^\calY \;|\; \sum_{y\in\calY} q_y =1 \}$ is the distribution simplex over $\calY$.
    Let $\Omega_{\Delta^{\calY}}$ be a proper l.s.c.~convex function with domain $\Delta^\calY$ whose restriction to $H_\Delta$ (the affine hull of $\Delta^{\calY}$) is Legendre-type.
    Then 
    \[ \nabla \Omega_{\Delta^\calY}^* : s \in \bbR^{\calY} \mapsto \argmax_{q \in \Delta^\calY}\{ s^\top q - \Omega_{\Delta^{\calY}}(q) \}\]
    maps any score vector $s \in \bbR^{\calY}$ to a (unique) probability distribution over $\calY$. We refer to Proposition~\ref{prop:sub_dimensioal_domain_cvx_analysis} for Fenchel duality results that underpin this definition. 
    We call~$\bbR^\calY$ the \emph{score space}, which is nothing but the space of cost functions from $\calY$ to $\bbR$, but seen as vectors indexed by $\calY$. 
    Furthermore, the Fenchel--Young loss associated with such an $\Omega_{\Delta^\calY}$
    \[ \calL_{\Omega_{\Delta^{\calY}}}(s,q) = \Omega_{\Delta^{\calY}}(q) + \Omega_{\Delta^{\calY}}^*(s) - \langle s | q \rangle  \]
    quantifies how far distribution $q \in \Delta^{\calY}$ is from $\nabla\Omega_{\Delta^{\calY}}^*(s) \in \Delta^{\calY}$.
    
    \begin{example}[Negentropy Regularization]\label{ex:negentropy}
The most classic regularization $\Omega_{\Delta^{\calY}}$ is arguably the \emph{negentropy} $\Omega_{\Delta^{\calY}}(q) = \sum_{y \in \calY}q_y \log(q_y) + \bbI_{\Delta^\calY}(q)$. 
In that case, $\nabla \Omega^*_{\Delta^\calY}(s)$ maps the score~$s$ to its \emph{softmax}, which is the exponential family on $\calY$ parameterized by $s$:
\begin{equation}\label{eq:exponentialFamily}
    \nabla \Omega^*_{\Delta^\calY}(s) = \big(e^{s_y- A_{\Delta^{\calY}}(s)}\big)_{y \in \calY}\quad \text{where}\quad A_{\Delta^{\calY}}(s) = \log\Big(\sum_{y' \in \calY}\exp(s_{y'})\Big).
\end{equation}
With this regularization, the Fenchel--Young loss coincides with the Kullback--Leibler divergence $ \calL_{\Omega_{\Delta^{\calY}}}(s,q) = D_{\mathrm{KL}}(q\|\nabla\Omega^*_{\Delta^{\calY}}(s))$ between $q$ and $\nabla\Omega^*_{\Delta^{\calY}}(s)$ \parencite{blondelLearningFenchelYoungLosses2020a}.

When $\calY$ is combinatorially large, exact probability computations become intractable and require approximations such as variational inference or Markov Chain Monte Carlo (MCMC) methods~\parencite{wainwright2008graphical,vivierardisson2025learninglocalsearchmcmc}. Beyond graphical models, negentropy regularization has seen wide application in combinatorial optimization to smooth discrete landscapes and formulate continuous, differentiable approximations. Building upon these concepts, \textcite{mensch2018differentiable} employed entropy-regularized formulations to enable differentiable dynamic programming for structured prediction. More recently, this framework has proven instrumental in advancing learning-based combinatorial solvers: \textcite{sanokowski2023variational} leverage variational annealing on graphs to improve optimization trajectories, and \textcite{chen2025monte} rely on entropy-regularized Monte Carlo policy gradients to reliably navigate binary optimization spaces. Furthermore, these regularization schemes naturally integrate into deep unified architectures that reduce diverse combinatorial problems into matrix-encoded generalizations, ensuring robust exploration and stable convergence \parencite{pan2025unico}.
\end{example}

    \begin{example}[Sparse perturbation]\label{ex:perturbation} In Section~\ref{sec:structured_perturbation}, we extend the work of~\textcite{berthetLearningDifferentiablePertubed2020} and introduce a new regularization function $\Omega_{\Delta^\calY}$ based on the sparse perturbation function ${F_{\varepsilon,\Delta}(s) = \bbE[\max_{y\in \calY}s_y + \varepsilon \bfZ^\top y] = \bbE[\max_{q\in \Delta^{\calY}}(s + \varepsilon Y^\top \bfZ)^\top q]}$ where $\bfZ$ is a random variable, typically a standard Gaussian. 
    More precisely, we define $\Omega_{\Delta^\calY}$ as the Fenchel conjugate of $F_{\varepsilon,\Delta}$.
    This regularization enjoys convenient properties that we describe in Section~\ref{sec:structured_perturbation}. 
    In particular,   
    \[ \nabla \Omega_{\Delta^\calY}^*(s) = \nabla F_{\varepsilon,\Delta}(s) = \bbE[\argmax_{q\in \Delta^{\calY}}(s + \varepsilon Y^\top \bfZ)^\top q], \]
    which allows us to compute stochastic gradients using Monte Carlo approaches by sampling $\bfZ$, which is particularly convenient for the alternating minimization algorithms presented in Section~\ref{subsec:alternating_min_tractability}.
    \end{example}
    
    \paragraph{Policies over combinatorial sets.}
    A policy maps a context value $x \in \calX$ to a distribution over the corresponding combinatorial set $q \in \Delta^{\calY(x)}$. 
    To define such policies, we map a context $x$ to a direction vector $\theta = \varphi_w(x)$; then lift it to the score space $s_\theta = Y(x)^\top \theta = (\langle \theta | y\rangle)_{y \in \calY(x)}$, where $Y(x) = (y)_{y \in \calY(x)}$ is the wide matrix of solutions in $\calY(x)$; and finally to a distribution~$q = \nabla \Omega_{\Delta^\calY(x)}^*(s_\theta)$.
    Thus, the policy parameterized by $w$ is defined as
    \[\pi_w(\cdot|x) = \argmax_{q \in \Delta^{\calY(x)}} \{ \langle \underbrace{Y(x)^\top \overbrace{\varphi_w(x)}^{\theta \in \bbR^{d(x)}}}_{s_\theta \in \bbR^{\calY(x)}} |q \rangle - \Omega_{\Delta^{\calY(x)}}(q) \} = \nabla \Omega_{\Delta^{\calY(x)}}^*\big(Y(x)^\top \varphi_w(x)\big).\]
    This connection is further detailed in Appendix~\ref{subsec:distribution_space}.
    Such a policy depends on the weights $w$ and the choice of the regularization function $\Omega_{\Delta^{\calY(x)}}$.
    We do not explicitly show this second dependency to alleviate notation since the regularization $\Omega_{\Delta^{\calY(x)}}$ is chosen once and for all.

    \subsection{The empirical cost minimization problem and its surrogate}
    \label{subsec:surrogate_and_guarantee}
    \paragraph{The empirical cost minimization problem.} 
    We return to the setting described in Section~\ref{sec:intro} and denote by~$(x, \xi)$ a context--noise pair. Recall that we have access to a dataset~$(x_1, \xi_1), \ldots, (x_N, \xi_N)$ of such pairs.
    Our goal is to find $w$ values that lead to low empirical cost $R_N(\pi_w)$.
    \begin{equation}
    \label{eq:R_N_pi_w}
        \min_{w \in \calW} R_N(\pi_w) = 
        \min_{w \in \calW} \frac{1}{N}\sum_{i=1}^N \bbE_{\bfy \sim \pi_w(\cdot|x_i)}\Big[c(x_i, \bfy, \xi_i)\Big]
    \end{equation}
    With the $\Omega_\Delta$ previously introduced, the empirical cost is differentiable with respect to $w$, and stochastic gradients can be computed using score function estimators \parencite{williamsSimpleStatisticalGradientfollowing1992}.
    We could therefore directly minimize the empirical cost using stochastic gradient descent (SGD). This SGD performs poorly as
    the score function estimator suffers from a high variance, and $R_N$ is highly non-convex as it is a smoothed piecewise constant function~\parencite{parmentier2021learning}.
    We therefore follow a different approach based on (less noisy) pathwise estimators for gradients and a convexified problem.
    
    \paragraph{Reformulation as a linear problem on the distribution space.}
    To that purpose, we reformulate the empirical cost minimization as a linear problem on the distribution space.
    Let ${\gamma(x,\xi) = \big(c(x,y,\xi)\big)_{y\in \calY(x)} \in \bbR^{\calY(x)}}$ be the score vector corresponding to the cost function $c(x, \cdot, \xi)$ on the (finite) combinatorial space $\calY(x)$. 
    In general, $\gamma$ does not belong to the image of $Y(x)^\top$.
    Given a distribution $q \in \Delta^\calY$ on $\calY$, let us define $R_{\Delta}(q;x,\xi)$ as the expected cost under $q$. We can recast it as
    \begin{equation}
    \label{eq:R_Delta}
        R_{\Delta}(q;x,\xi) = \bbE_{\bfy \sim q}\Big[c(x,\bfy,\xi)\Big] = \langle \gamma(x, \xi) | q \rangle.
    \end{equation}
    We can rewrite the empirical cost as 
    $$ 
    \begin{aligned}
    R_N(\pi_w) &= 
        \frac{1}{N}\sum_{i=1}^N \bbE_{\bfy \sim \pi_w(\cdot|x_i)}\Big[c(x_i, \bfy, \xi_i)\Big] \\
        &= 
        \frac1N \sum_{i=1}^N R_\Delta\big( \underbrace{\nabla \Omega^*_{\Delta^{\calY(x_i)}}\big(\overbrace{Y(x_i)^\top\varphi_w(x_i)}^{s_i}\big)}_{\pi_w(\cdot|x_i)}, x_i,\xi_i\big)
    \end{aligned}
        $$
    We introduce the notation~$\calR_{\Omega_\Delta,N}$ for the expected cost as a function of the score $s_{\otimes}$, relying on the following product spaces
    \begin{subequations}
    \begin{align}
        S_{\otimes} &= \{s_\otimes = (s_i)_{i \in [N]} \,|\, \forall i \in [N], s_i \in \bbR^{\calY(x_i)} \},\\ 
        \Delta_{\otimes} &= \{q_\otimes = (q_i)_{i \in [N]} \,|\, \forall i \in [N], q_i \in \Delta^{\calY(x_i)} \}, \\
        \calR_{\Omega_\Delta,N}(s_{\otimes}) &= \frac{1}{N} \sum_{i=1}^N R_{\Delta}(\nabla\Omega_{\Delta^{\calY(x_i)}}^*(s_i);x_i,\xi_i) \quad \text{where} \quad s_{\otimes} \in S_{\otimes}.\label{eq:empirical_risk_product}
    \end{align}
    \end{subequations}
    The expression of $R_N(\pi_w)$  above allows us to rewrite the empirical cost minimization problem~\eqref{eq:Empiricalrisk} as
    \begin{equation}
    \label{eq:dist_regret_min_pb}
    \min_{w \in \calW} R_N(\pi_w) = \min_{w \in \calW} \calR_{\Omega_\Delta,N}\Big(\big(Y(x_i)^\top\varphi_w(x_i)\big)_{i \in [N]}\Big).
    \end{equation}
    
    \paragraph{Surrogate problem.}
    We recall that the Fenchel--Young loss generated by~$\Omega_{\Delta^{\calY(x)}}$ is defined over ${\bbR^{\calY(x)} \times \Delta^{\calY(x)}}$ by~$\calL_{\Omega_{\Delta^{\calY(x)}}}(s;q) = \Omega_{\Delta^{\calY(x)}}(q) + \Omega_{\Delta^{\calY(x)}}^*(s) - \langle s | q \rangle$.
    Let $\kappa >0$ be a positive constant, we introduce the following surrogate functions for a single observation and for the full dataset:
    \begin{subequations}\label{eq:surrogate_risk_delta}
    \begin{align}
            S_{\Omega_\Delta}(s,q;x,\xi) &= \langle\gamma(x,\xi)| q \rangle + \kappa \calL_{\Omega_{\Delta^{\calY(x)}}}\big(s, q\big), \\
            \calS_{\Omega_\Delta, N}\big(s_{\otimes}, q_{\otimes}\big) &= \frac{1}{N} \sum_{i=1}^N S_{\Omega_\Delta}\big(s_i,q_i;x_i,\xi_i\big),
            \label{eq:surrogate_risk_product}
    \end{align}
    \end{subequations}
    where $s_\otimes = (s_i)_{i\in [N]}$, $q_\otimes = (q_i)_{i\in [N]}$, and $\gamma(x,\xi)$ corresponds to the cost vector $(c(x,y,\xi))_{y \in \calY(x)}$.
    Notice that, by the Fenchel--Young inequality, $S_{\Omega_\Delta} \geq R_\Delta$, with equality holding if and only if $s$ and $q$ are a dual pair matching the policy, \ie $q = \nabla\Omega_{\Delta^{\calY(x)}}^*(s)$.
    Instead of directly minimizing the non-convex empirical cost, we aim to solve the following \emph{surrogate learning problem}:
    \begin{equation}\label{eq:surrogate_learning_pb}
        \min_{\substack{w \in \calW,\\ q_{\otimes} \in \Delta_\otimes}} \calS_{\Omega_\Delta, N}\Big(\big(Y(x_i)^\top \varphi_w(x_i)\big)_{i \in [N]}, q_{\otimes}\Big).
    \end{equation}
    
    \paragraph{Bounding the error between the surrogate and the initial problem.}
    We introduce~\eqref{eq:surrogate_learning_pb} for algorithmic reasons. But before diving into algorithms, a natural concern is the quality of this surrogate object for solving the original empirical cost minimization~\eqref{eq:dist_regret_min_pb}. Let us introduce the partial minimizer of the surrogate cost for a given fixed scenario:
    \begin{subequations}
       \begin{align}\label{eq:partial_surrogate}
        \underline{\calS_{\Omega_\Delta}}(\theta;x,\xi) &:= \min_{q \in \Delta^\calY(x)} \calS_{\Omega_\Delta}\big(Y(x)^\top \theta, q;x,\xi\big),\\
        \underline{\calS_{\Omega_\Delta,N}}(\varphi_w) &:= \frac{1}{N} \sum_{i=1}^N \underline{\calS_{\Omega_\Delta}}(\varphi_w(x_i);x_i,\xi_i) \nonumber\\
        &= \min_{q_{\otimes} \in \Delta_\otimes} \calS_{\Omega_\Delta, N}\big((Y(x_i)^\top \varphi_w(x_i))_{i \in [N]}, q_{\otimes}\big),
    \end{align} 
    \end{subequations}

The following theorem shows that the partial surrogate is a lower
approximation of the empirical cost and gives an exact expression of the
approximation error. 
We use the standard notation \(D_F(\cdot \mid \cdot)\)
for the Bregman divergence associated with a differentiable convex function \(F\).

\begin{restatable}{theorem}{thmboundrisk}
\label{thm:bound_risk}
Let \(x\in\calX\), \(\xi\in\Xi\), and \(\theta\in\bbR^{d(x)}\), and set
\(s:=Y(x)^\top\theta\) and \(\gamma:=\gamma(x,\xi)\).
Assume that \(\nabla\Omega_{\Delta^{\calY(x)}}^*\) is
\(1/L_x\)-Lipschitz-continuous. Then
\begin{equation}
\label{eq:bound_risk_single}
0
\leq
R_\Delta\big(
\nabla\Omega_{\Delta^{\calY(x)}}^*(s);x,\xi
\big)
-
\underline{\calS_{\Omega_\Delta}}(\theta;x,\xi)
=
\kappa
D_{\Omega_{\Delta^{\calY(x)}}^*}
\left(
s-\frac{\gamma}{\kappa}\mid s
\right)
\leq
\frac{\|\gamma\|^2}{2L_x\kappa}.
\end{equation}

For the full dataset, let
\(s_i(w):=Y(x_i)^\top\varphi_w(x_i)\),
\(\gamma_i:=\gamma(x_i,\xi_i)\), and
\(\Omega_i:=\Omega_{\Delta^{\calY(x_i)}}\).
If \(\nabla\Omega_i^*\) is \(1/L_i\)-Lipschitz-continuous for every
\(i\in[N]\), then, for every \(w\in\calW\),
\begin{equation}
\label{eq:bound_risk_sum}
0
\leq
\calR_{\Omega_\Delta,N}(\varphi_w)
-
\underline{\calS_{\Omega_\Delta,N}}(\varphi_w)
=
\frac{\kappa}{N}
\sum_{i=1}^N
D_{\Omega_i^*}
\left(
s_i(w)-\frac{\gamma_i}{\kappa}\mid s_i(w)
\right)
\leq
\frac{1}{2\kappa N}
\sum_{i=1}^N
\frac{\|\gamma_i\|^2}{L_i}.
\end{equation}

Finally, suppose that the minima are attained, and let
\(w_{\calS}\in\argmin_{w\in\calW}
\underline{\calS_{\Omega_\Delta,N}}(\varphi_w)\) and
\(w_{\calR}\in\argmin_{w\in\calW}
\calR_{\Omega_\Delta,N}(\varphi_w)\). Then
\begin{equation}
\label{eq:bound_empirical_risk}
\calR_{\Omega_\Delta,N}(\varphi_{w_{\calS}})
-
\calR_{\Omega_\Delta,N}(\varphi_{w_{\calR}})
\leq
\frac{1}{2\kappa N}
\sum_{i=1}^N
\frac{\|\gamma_i\|^2}{L_i}.
\end{equation}
In particular, the right-hand sides of
\eqref{eq:bound_risk_sum} and \eqref{eq:bound_empirical_risk} are bounded
above by
\(\frac{1}{2L\kappa N}\sum_{i=1}^N\|\gamma_i\|^2\), where
\(L:=\min_{i\in[N]}L_i\).
\end{restatable}

The proof is provided in Appendix~\ref{sec:bound_empirical_risk}.

\begin{remark}
Both the Euclidean regularization and the negentropy have
\(1\)-Lipschitz conjugate gradients on the simplex. We show in
Proposition~\ref{prop:strongConvexitySparsePerturbation} that the sparse
perturbation regularization also has a Lipschitz-continuous conjugate
gradient.
\end{remark}

    \subsection{Alternating minimization algorithm and tractability}
    \label{subsec:alternating_min_tractability}

    We propose the following \emph{primal-dual alternating minimization scheme} for Problem~\eqref{eq:surrogate_learning_pb}.
    \begin{subequations}\label{eq:alternating_product}
        \begin{align}
            q_i^{(t+1)} &= \argmin_{q_i \in \Delta^{\calY(x_i)}} S_{\Omega_\Delta}\big(Y(x_i)^\top \varphi_{\bar w^{(t)}}(x_i), q_i; x_i, \xi_i\big),
            \forall i \in [N], && \text{(decomposition)} \label{eq:decomposition_product_w}\\
            \bar w^{(t+1)} &\in \argmin_{w \in \calW} \calS_{\Omega_\Delta, N}\Big(\big(Y(x_i)^\top \varphi_w(x_i)\big)_{i \in [N]},
            q_{\otimes}^{(t+1)}\Big). && \text{(coordination)} \label{eq:coordination_product_w}
        \end{align}
    \end{subequations}
    
        By construction of alternating minimization, the sequence of evaluated surrogate values $\calS_{\Omega_\Delta, N}$ monotonically decreases. 
    To get better guarantees, one needs to assume a generalized linear structure mapping, such as $\varphi_w(x) = \phi(x)^\top w$. But before delving into convergence, let us start with the tractability of the iterates~\eqref{eq:alternating_product}.

    \paragraph{Tractable updates via the moment space.}

    Algorithm~\eqref{eq:alternating_product} looks intractable at first glance, as working directly with full distributions $q_i \in \Delta^{\calY(x_i)}$ is computationally prohibitive for combinatorial $\calY(x_i)$. 
    However, the following results show that we can work with moments instead of full distributions. The proof is given in Appendix~\ref{subsec:tractable_proofs}.
    \begin{restatable}{proposition}{propcomputationsprimaldualdist}\label{prop:computations_primal_dual_dist}
        Let $\mu_i^{(t+1)} = \bfE_{\bfy\sim q_i^{(t+1)}}[\bfy] = Y(x_i)q_i^{(t+1)}$ be the moment of $\bfy_i$ according to $q_i^{(t+1)}$. Given $\bar w^{(t)}$, the next iterate of~\eqref{eq:alternating_product} can be computed through moments:
        \begin{subequations}
            \begin{align}
                \mu_i^{(t+1)} &= \bbE_{\bfy \sim q_i^{(t+1)}}[\bfy],
                 \quad \text{where} \quad
                q_i^{(t+1)}=\nabla \Omega_{\Delta^{\calY(x_i)}}^*\Big(Y(x_i)^\top\varphi_{\bar w^{(t)}}(x_i) - \frac{1}{\kappa}\gamma_i\Big), \label{eq:decomposition_w_dist} \\
                \bar w^{(t+1)} &\in \argmin_{w \in \calW} \frac{1}{N} \sum_{i=1}^N \calL_{\Omega_{\calC(x_i)}}\big(\varphi_w(x_i); \mu_i^{(t+1)}\big), \label{eq:coordination_w_dist}
            \end{align}
            where $\calC(x_i) = \conv(\calY(x_i))$ is the moment polytope and $\Omega_{\calC(x_i)}(\mu) = \min_{q \in \Delta^{\calY(x_i)}}\{\Omega_{\Delta^{\calY(x_i)}}(q) \,|\, Y(x_i)q = \mu\}$.
        \end{subequations}
    \end{restatable}
    Dual coordination~\eqref{eq:coordination_w_dist} reduces to supervised learning over the low-dimensional moment space $\calC$ with $\mu_i^{(t)}$ as targets. It can be solved easily using stochastic gradient descent with well-chosen regularizations~\parencite{blondelLearningFenchelYoungLosses2020a,berthetLearningDifferentiablePertubed2020}.
    Second, $q_i^{(t+1)}$ admits a characterization that makes its moment $\mu_i^{(t+1)} $ tractable for well-chosen regularization $\Omega_{\Delta^{\calY(x)}}$.
    Let us now show that we can compute moments of~$\nabla \Omega_{\Delta^{\calY(x_i)}}^*\Big(Y(x_i)^\top\varphi_{\bar w^{(t)}}(x_i) - \frac{1}{\kappa}\gamma_i\Big)$ for our two main regularization functions.

    \paragraph{Sparse perturbation.}
    Under a sparse perturbation formulation, Monte Carlo estimates of the primal moment can be computed using the deterministic combinatorial oracle.
    \begin{restatable}{proposition}{propprimaldualperturbation}\label{prop:primal_dual_perturbation}
        Let~$\varepsilon>0$ be a positive constant.
        When the regularization functions~$\Omega_{\calC(x)}$ and $\Omega_{\Delta^{\calY(x)}}$ are defined as the Fenchel conjugates of the perturbed maxima~$\Omega_{\calC(x)} := F_{\varepsilon, \calC(x)}^*$ and $\Omega_{\Delta^{\calY(x)}} := F_{\varepsilon, \Delta(x)}^*$ introduced in Equations~\eqref{eq:perturbation_moment}-\eqref{eq:perturbation_distribution} for some $\varepsilon>0$,
        \begin{equation}\label{eq:primalUpdatePerturbation}
            \mu_i^{(t+1)} = \bbE_\bfZ\Big[\argmin_{y \in \calY(x_i)} c(x_i,y,\xi_i) - \kappa\big(\varphi_{\bar w^{(t)}}(x_i) + \varepsilon \bfZ\big)^\top y \Big],
        \end{equation}
        where $\bfZ$ is standard multivariate Gaussian noise.
    \end{restatable}
    The proof is given in Appendix~\ref{subsec:tractable_proofs}. 
    Remark that $\kappa$ and $\varepsilon$ appear only through their product $\kappa\varepsilon$ in~\eqref{eq:primalUpdatePerturbation}; in particular, fixing one and tuning the other explores exactly the same family of algorithm trajectories, in the non-contextual case.

    \paragraph{Negentropic regularization.}
    Under a negentropy formulation, computing $\mu_i^{(t+1)}$ amounts to inference in an exponential family over $\calY$, and we therefore have the following well-known result~\parencite{wainwright2008graphical}.
    \begin{proposition}\label{prop:primal_dual_negentropy}
        When the regularization functions~$\Omega_{\calC(x)}$ and $\Omega_{\Delta^{\calY(x)}}$ are defined using the negentropy $\Omega_{\Delta^{\calY(x)}}(q) = \sum_{y \in \calY(x)} q_y\log(q_y) + \bbI_{\Delta^{\calY(x)}}(q)$, the primal moment update becomes:
        \begin{multline}
            \mu_i^{(t+1)} = \sum_{y \in \calY(x_i)} y \exp\Big( \varphi_{\bar w^{(t)}}(x_i)^\top y - \frac{1}{\kappa} c(x_i,y,\xi_i) \\
            - A_{\Delta^{\calY(x_i)}}\Big(Y(x_i)^\top\varphi_{\bar w^{(t)}}(x_i) - \frac{1}{\kappa}\gamma_i\Big)\Big),
        \end{multline}
        where $A_{\Delta^{\calY(x_i)}}$ is the log-partition function of the corresponding exponential family.
    \end{proposition}
    If the exact inference problem is generally intractable for large $\calY(x_i)$, we can perform approximate inference using Metropolis-Hastings Markov Chain Monte Carlo (MCMC) methods \parencite{wainwright2008graphical}. In practice, sampling this exponential family yields an algorithm closely mirroring classic simulated annealing for the non-perturbed version of~\eqref{eq:primalUpdatePerturbation} \parencite{vivierardisson2025learninglocalsearchmcmc}.
    \subsection{\texorpdfstring{Convergence when $\varphi_w$ is linear in $w$}{Convergence when phi\_w is linear in w}}
    \label{subsec:convergence}
    Our convergence proof does not focus on the non-linearity in the neural network. Let us assume for this subsection that the statistical model is linear, \ie $\varphi_w(x_i) = \phi_i^\top w$ for some matrix $\phi_i$, and $\calW = \bbR^{n_{\calW}}$. 

    \paragraph{\texorpdfstring{Identifiable parameter $\bar w$}{Identifiable parameter w-bar}.}
    Let us start with the orthogonal decomposition of $\calW$ into identifiable parameters and their orthogonal. 
    Let 
    $$ \calM = \left\{ \frac{1}{N} \sum_{i=1}^N \phi_i Y(x_i) q_i \mid q_i \in \Delta^{\calY(x_i)} \right\} = \left\{ \frac{1}{N} \sum_{i=1}^N \phi_i \mu_i \mid \mu_i \in \calC(x_i) \right\}, $$ 
    $H_\calM := \aff(\calM)$ be the affine hull of $\calM$, $\bar \calW$ the direction of $H_\calM$ in $\calW$, and $\bar \calW^{\perp}$ be its orthogonal so that~$\calW = \bar \calW \oplus \bar \calW ^{\perp}$. We define $\bar\calM = \Pi_{\bar \calW}(\calM)$.
    
    \begin{proposition}
    \label{prop:identifiableParameter}
        Let $w\in \calW$, $q_\otimes \in \Delta_{\otimes}$, and $Y_i := Y(x_i)$.
        \begin{enumerate}
            \item For any $i$, the result of our policy depends only on $\Pi_{\bar \calW}(w)$:\\   $\nabla \Omega_{\Delta^{\calY(x_i)}}^*(Y_i^\top \phi_i^\top w) = \nabla \Omega_{\Delta^{\calY(x_i)}}^*(Y_i^\top \phi_i^\top \Pi_{\bar \calW}(w)) $ 
            \item The value of the surrogate depends only on $\Pi_{\bar \calW}(w)$: \\ $S_{\Omega_{\Delta}, N}\big(w,q_{\otimes}\big) = S_{\Omega_{\Delta}, N}\big(\Pi_{\bar\calW}(w),q_{\otimes}\big) $.
            \item When learning a $w$, only the component in $\bar\calW$ is \emph{identifiable}:\\
            Let $w^*\in \argmin\frac{1}{N}\sum_{i=1}^N\calL_{\Omega_{\calC(x_i)}}(\phi_i^\top w, Y_iq_i)$, then $\argmin\frac{1}{N}\sum_{i=1}^N\calL_{\Omega_{\calC(x_i)}}(\phi_i^\top w, Y_iq_i) = \Pi_{\bar \calW}(w^*) + \bar\calW^\perp$. 
        \end{enumerate}
    \end{proposition}
    The proof of Proposition~\ref{prop:identifiableParameter} is a direct consequence of Propositions~\ref{prop:sub_dimensioal_domain_cvx_analysis} and \ref{prop:structuredPredictionWithGeneralizedNegentropy}, with the orthogonal of a sum of subspaces being the intersection of the orthogonals of each subspace.
    As a consequence, we can focus on the identifiable part of the surrogate problem
    \begin{align}
    \min_{\bar w \in \bar \calW} \underline{S_{\Delta,N}}(\bar w) \quad \text{where} \quad
    \underline{S_{\Delta,N}}(\bar w) &:= \min_{q_\otimes \in \Delta_\otimes} S_{\Omega_{\Delta},N}\Big((Y(x_i)^\top \phi_i^\top \bar w)_{i \in [N]},q_i\Big),
    \end{align}
    where we have omitted the canonical inclusion from $\bar \calW$ to $\calW$ in $Y(x_i)^\top \phi_i^\top w$ for clarity.

    \paragraph{\texorpdfstring{Conjugate $\bar \nu$ of parameter $\bar w$}{Conjugate nu-bar of parameter w-bar}.}
    Our convergence results requires some form of convexity.
    To that purpose, we need to move from objective $\underline{S_{\Delta,N}}(\bar w) $ in variable $\bar w$, whose geometry is the one of a regularized piecewise constant function on the normal fan of $\bar \calM$, to an objective expressed on the ``Fenchel conjugate'' of $\bar w$, which follows the geometry of $\bar \calM$.
    We define $\Omega_{\bar{\calM}}(\bar{\nu})$ as the Fenchel conjugate of the average of the dual regularization functions evaluated on the identifiable score space:
    \begin{equation}
        \Omega_{\bar{\calM}} := \bar{F}^*, \quad \text{where} \quad \bar{F}(\bar{w}) = \frac{1}{N} \sum_{i=1}^N \Omega_{\Delta^{\calY(x_i)}}^*\big(Y(x_i)^\top \phi_i^\top J_{\Pi_{\bar \calW}} \bar{w}\big).
    \end{equation}

    where $J_{\Pi_{\bar \calW}}$ is the canonical injection that maps identifiable parameter $\bar{w}$ to the full parameter space~$\calW$, \ie $w = J_{\Pi_{\bar \calW}} \bar{w}$. Figure~\ref{fig:mirror_descent_structure} illustrates the properties of $\Omega_{\bar\calM}$ described in the following proposition.
    
    \begin{restatable}{proposition}{propcontextualregularization}\label{prop:contextualRegularization}
        The function $\bar{F}$ is of Legendre-type, hence $\bar{F} = \Omega_{\bar{\calM}}^*$. Furthermore, $\Omega_{\bar{\calM}}$ is given by the infimal convolution:
        \begin{equation}\label{eq:contextualRegularizationProp}
            \Omega_{\bar{\calM}}(\bar{\nu}) = \inf_{(q_i)_{i \in [N]}} \left\{ \frac{1}{N} \sum_{i=1}^N \Omega_{\Delta^{\calY(x_i)}}(q_i) \;\middle|\;
            q_i \in \Delta^{\calY(x_i)}, \, \frac{1}{N} \sum_{i=1}^N \Pi_{\bar \calW}\big( \phi_i Y(x_i) q_i\big) = \bar{\nu} \right\}.
        \end{equation}
        The domain of $\Omega_{\bar{\calM}}$ is the full-dimensional identifiable aggregated moment space $\bar{\calM}$.
        For any $\bar{\nu} \in \relint(\bar{\calM})$, the minimum in~\eqref{eq:contextualRegularizationProp} is attained at $$q_i = \nabla \Omega_{\Delta^{\calY(x_i)}}^*\big(Y(x_i)^\top \phi_i^\top J_{\Pi_{\bar \calW}} \nabla \Omega_{\bar{\calM}}(\bar{\nu})\big).$$ 
        Given $q_\otimes \in \Delta_{\otimes}$, let $\begin{cases}
               \bar w = \Pi_{\bar \calW}(w) &\text{ for } w \in \argmin \frac{1}{N} \sum_{i=1}^N \calL_{\Omega_{\calC(x_i)}}(\phi_i^\top w,Y_iq_i), \\
               \bar \nu = \Pi_{\bar \calW}(\nu) &\text{ for } \nu = \frac{1}{N} \sum_{i=1}^N \phi_i Y_i q_i,
            \end{cases} $ \\
            the regularization function~$\Omega_{\bar \calM}$ has the following properties, which are useful below
        \begin{align}
            \label{eq:FYlandCalM}
            \Pi_{\bar \calW}(w) &= \nabla \Omega_{\bar \calM}\big(\Pi_{\bar \calW}(\nu)\big) \\
            \frac{1}{N}\sum_{i=1}^N\big( \Omega_{\Delta^{\calY(x_i)}}(q_i) - \Omega_{\bar \calM}(\bar \nu)\big) 
            &= \frac{1}{N} \sum_{i=1}^N \calL_{\Omega_{\Delta^{\calY(x_i)}}}\big(Y(x_i)^\top \phi_i^\top J_{\Pi_{\bar \calW}}\bar w, q_i\big).
            \label{eq:crossJensenGapInCalM}
        \end{align}
    \end{restatable}
    The detailed statement and proof are given in Section~\ref{sec:loss_parameter_space} and Appendix~\ref{subsec:appendix_loss_parameter_space}, respectively.
    We call $q_{\otimes} \mapsto \frac{1}{N}\sum_{i=1}^N\big( \Omega_{\Delta^{\calY(x_i)}}(q_i) - \Omega_{\bar \calM}(\bar \nu)\big)$ the Cross Jensen gap in $q_{\otimes}$. We use this terminology when the problem is non-structured ($Y = I$) and non-contextual ($\Phi = I$), we fall back on the usual \emph{Jensen gap} $\frac{1}{N}\sum_{i=1}^N \Omega_{\Delta}(q_i) - \Omega_{\Delta}(\frac1N\sum_{i=1}^N q_i)$.
    
    \begin{figure}[ht]
        \centering
        \resizebox{\linewidth}{!}{
        \begin{tikzpicture}[
            prop/.style={->, >=latex, thick},
            map/.style={dashed, ->, >=latex, bend left=15},
            leg/.style={->, >=latex, bend left=45},
            scale=0.85,
            every node/.style={transform shape}
        ]
        
        \begin{scope}[xshift=-4cm, canvas is xy plane at z=0]
            \node[align=center] (ident_param_label) at (1, 5) {\textbf{Identifiable}\\ \textbf{Parameter} ($\bar{w}$)};
            \draw[->,draw=purple] (0, 0.5) -- (0, 4) node[above] {$\bbR^{\bar{d}}$};
            \draw[->,draw=purple] (0, 0.5) -- (3, 0.5);
            \node[circle, fill, inner sep=1.5pt, label={right:$\bar{w}$}] (wbar) at (1.5, 1.5) {};
            
            \node[align=center] (ident_agg_label) at (1, -4.5) {\textbf{Identifiable}\\ \textbf{Moment Space}};
            \node[regular polygon, regular polygon sides=6, minimum size=2.5cm, draw=purple, thick, rotate=15] (agg_poly_bar) at (1.5, -2) {};
            \node[circle, fill=purple, inner sep=1.5pt, label={below:$\bar{\nu} = \Pi_{\bar \calW} (\nu)$}] (nubar) at (1.5, -2) {};
    
            \draw[leg] (wbar) to node[right, font=\small, pos=0.4] {$\nabla \Omega_{\bar{\calM}}^*$} (nubar);
            \draw[leg, bend left=45] (nubar) to node[left, font=\small, pos=0.6] {$\nabla \Omega_{\bar{\calM}}$} (wbar);
        \end{scope}

        \begin{scope}[xshift=1.5cm]
            \node[align=center] (param_label) at (0, 5, 0) {\textbf{Parameter}\\ \textbf{Space} ($w$)};
            \draw[->,draw=blue] (0, 0.5, 0) -- (0, 4, 0) node[above] {$\bbR^d$};
            \draw[->,draw=blue] (0, 0.5, 0) -- (3, 0.5, 0);
            \node[circle, fill, inner sep=1.5pt, label={right:$w$}] (w) at (1.5, 1.5, 0) {};
            
            \node[align=center] (agg_label) at (0, -4.5, 0) {\textbf{Aggregated}\\ \textbf{Moment Space}};
            \node[regular polygon, regular polygon sides=6, minimum size=2.5cm, draw=blue, thick, rotate=15] (agg_poly) at (1.5, -2, 0) {};
            \node[circle, fill=blue, inner sep=1.5pt, label={below:$\nu = \frac{1}{N}\sum \phi_i \mu_i$}] (nu) at (1.5, -2, 0) {};
        
            \draw[leg] (w) to node[right, font=\small, pos=0.4] {$\nabla \Omega_{\calM}^*$} (nu);
            \draw[leg, bend left=45] (nu) to node[left, font=\small, pos=0.6] {$\partial \Omega_{\calM}$} (w);
        \end{scope}
    
        \begin{scope}[xshift=7cm, canvas is xy plane at z=0]
            \node[align=center] (lin_label) at (1, 5) {\textbf{Linear Objective}\\ \textbf{Space}};
            \draw[green!60!black, ->] (0, 0.5) -- (0, 4) node[above] {$\bbR^{d_i}$};
            \draw[green!60!black, ->] (0, 0.5) -- (3, 0.5);
            \node[circle, fill, inner sep=1.5pt, label={right:$\theta_i = \phi_i^\top w$}] (theta) at (1.5, 1.5) {};
            
            \node[align=center] (mom_label) at (1, -4.5) {\textbf{Moment}\\ \textbf{Space}};
            \node[regular polygon, regular polygon sides=6, minimum size=2cm, draw=green!60!black, thick] (mom_poly) at (1.5, -2) {};
            \node[circle, fill=green!60!black, inner sep=1.5pt, label={below:$\mu_i = Y(x_i) q_i$}] (mu) at (1.5, -2) {};
    
            \draw[leg] (theta) to node[right, font=\small, pos=0.4] {$\nabla \Omega_{\calC_i}^*$} (mu);
            \draw[leg, bend left=45] (mu) to node[left, font=\small, pos=0.6] {$\partial \Omega_{\calC_i}$} (theta);
        \end{scope}
    
        \begin{scope}[xshift=11.5cm, canvas is xy plane at z=0]
            \node[align=center] (cost_label) at (1, 5) {\textbf{Cost Function}\\ \textbf{Space}};
            \draw[red!80!black, ->] (0, 0.5) -- (0, 4) node[above] {$\bbR^{\calY(x_i)}$};
            \draw[red!80!black, ->] (0, 0.5) -- (3, 0.5);
            \node[circle, fill, inner sep=1.5pt, label={right:$s_i = Y(x_i)^\top \theta_i$}] (c) at (1.5, 1.5) {};
            
            \node[align=center] (dist_label) at (1, -4.5) {\textbf{Distribution}\\ \textbf{Space}};
            \draw[red, thick] (0.5, -3) -- (2.5, -3) -- (1.5, -1) -- cycle; % Simplex
            \node at (2.8, -3.2) {$\Delta^{\calY(x_i)}$};
            \node[circle, fill=red, inner sep=1.5pt, label={below right:$q_i$}] (q) at (1.5, -2.2) {};
    
            \draw[leg] (c) to node[right, font=\small, pos=0.4] {$\nabla \Omega_{\Delta^{\calY(x_i)}}^*$} (q);
            \draw[leg, bend left=45] (q) to node[left, font=\small, pos=0.6] {$\partial \Omega_{\Delta^{\calY(x_i)}}$} (c);
        \end{scope}
    
        \begin{scope}[xshift=7.3cm, yshift=0.3cm, opacity=0.3, canvas is xy plane at z=-2]
             \draw[green!60!black] (0, 0.5) -- (0, 4); \draw[green!60!black] (0, 0.5) -- (3, 0.5);
             \node[regular polygon, regular polygon sides=6, minimum size=2cm, draw=green!60!black] at (1.5, -2) {};
        \end{scope}
        \begin{scope}[xshift=11.8cm, yshift=0.3cm, opacity=0.3, canvas is xy plane at z=-2]
             \draw[red!80!black] (0, 0.5) -- (0, 4); \draw[red!80!black] (0, 0.5) -- (3, 0.5);
             \draw[red, thick] (0.5, -3) -- (2.5, -3) -- (1.5, -1) -- cycle;
        \end{scope}
    
        \draw[map] (wbar) to node[above] {$J_{\Pi_{\bar \calW}}$} (w);
        \draw[map] (nu) to node[above] {$\Pi_{\bar \calW}$} (nubar);
        \draw[map] (w) to node[above] {$\phi_i^\top$} (theta);
        \draw[map] (theta) to node[above] {$Y(x_i)^\top$} (c);
        \draw[map] (q) to node[above] {$Y(x_i)$} (mu);
        \draw[map] (mu) to node[above] {$\phi_i$} (nu);
    
        \node[font=\Large\bfseries] at (16, 2.5, 0) {Dual};
        \node[font=\Large\bfseries] at (16, -2, 0) {Primal};
    
        \end{tikzpicture}
        }
        \caption{Illustration of the structural relationships between parameter, moment, and distribution spaces. The primal maps $\phi_i$ and $Y(x_i)$ aggregate moments, while dual maps $\phi_i^\top$ and $Y(x_i)^\top$ transport costs. The canonical injection $J_{\Pi_{\bar \calW}}$ and projection $\Pi_{\bar \calW}$ map to the identifiable parameter and moment space. Mirror maps link dual and primal spaces. Due to dimensions, non-full sets use subdifferentials $\partial \Omega$.}
        \label{fig:mirror_descent_structure}
    \end{figure}
    
      \paragraph{Alternating minimization as the Bregman proximal point algorithm.}

    Consider the iterates $q_\otimes^{(t)}$ and $w^{(t)}$ of algorithm~\eqref{eq:alternating_product}, we can define~${\bar \nu^{(t)}= \Pi_{\bar \calW}\big(\frac1N\sum_{i=1}^N\phi_iY_iq_i^{(t)}}\big)$ and $\bar w^{(t)} = \Pi_{\bar \calW}(w^{(t)})$. Equation~\eqref{eq:FYlandCalM} shows that we do not need the full details of the moments $\mu_i^{(t+1)}$ to compute $\bar w^{(t+1)}$, but only the aggregate moment $\bar \nu^{(t+1)}$ as $\bar w^{(t+1)} = \nabla\Omega_{\bar \calM}(\bar \nu^{t+1})$.

    Let us now reformulate algorithm~\eqref{eq:alternating_product} in $\bar \nu$.
    Consider the following relaxed objective function $f_\kappa(\bar \nu)$ over the aggregated moment space $\bar \calM$:
    \begin{equation}\label{eq:definitionOfFkappa}
    \begin{aligned}
        f_\kappa(\bar \nu) = \min_{q_{\otimes} \in \Delta_{\otimes}} \Bigg\{ \frac{1}{N} \sum_{i=1}^N \Big[ \langle \gamma_i | q_i \rangle + \kappa \big(\Omega_{\Delta^{\calY(x_i)}}(q_i) - \Omega_{\bar \calM}(\bar \nu)\big) \Big] \;\Bigg|\; \\
        \frac{1}{N} \sum_{i=1}^N \Pi_{\bar \calW}\big( \phi_i Y(x_i) q_i \big) = \bar \nu \Bigg\}.
    \end{aligned}
    \end{equation}

    Equation~\eqref{eq:crossJensenGapInCalM} highlights the difference between $f_{\kappa}$ and our algorithm iterates, we would need to split the minimization in~\eqref{eq:definitionOfFkappa} into two successive steps: first minimize the objective without constraint but fixing $\bar \nu = \bar \nu^{(t)}$, then compute $\bar \nu^{(t+1)}$ using the constraint.
    This decoupling is actually obtained in the proximal operator.
    
    \begin{restatable}{theorem}{theoproximalpointoperator}\label{theo:proximalPointOperator}
        Let $\bar \nu^{(0)} = \nabla \Omega_{\bar \calM}^*(\bar w^{(0)})$. The sequence $\bar \nu^{(t)}$ defined by the iterations of the Bregman proximal point algorithm on $f_\kappa$:
        \begin{equation}\label{eq:proximal_point_f_kappa}
            \bar \nu^{(t+1)} = \argmin_{\bar \nu \in \bar \calM} \left\{ f_\kappa(\bar \nu) + \kappa D_{\Omega_{\bar \calM}}\big(\bar \nu \mid \bar \nu^{(t)}\big) \right\}
        \end{equation}
        matches the iterates of the alternating minimization algorithm~\eqref{eq:alternating_product}, such that for all $t$, we have the correspondence $\bar \nu^{(t)} = \nabla \Omega_{\bar \calM}^*\big(\Pi_{\bar \calW}(w^{(t)})\big)$.
    \end{restatable}
    Given Theorem~\ref{theo:proximalPointOperator}, we can use \textcite{eckstein1993nonlinear} to get the convergence of the PPA towards the global minimum if $f_{\kappa}$ is convex, and derive a proof of convergence to a stationary point from~\parencite{attouch_convergence_2013} when it is not. The convexity of $f_\kappa$ is implied by the convexity of the \emph{cross Jensen gap}.
    \begin{restatable}{proposition}{propconvexityfkappa}\label{prop:convexityfkappa}
        The mapping $\bar \nu \mapsto f_{\kappa}(\bar \nu)$ is convex if the cross Jensen gap $q_\otimes \mapsto \frac{1}{N}\sum_{i=1}^N \Omega_{\Delta^{\calY(x_i)}}(q_i) - \Omega_{\bar \calM}(
            \bar \nu(q_\otimes) )$, where $\bar \nu(q_\otimes):= \frac{1}{N} \Pi_{\bar{\calW}} \sum_{i=1}^N \phi_i Y(x_i) q_i $, is convex.
    \end{restatable}
    Theorem~\ref{theo:proximalPointOperator} and Proposition~\ref{prop:convexityfkappa} are proved in Appendix~\ref{subsec:appendix_loss_parameter_space}.
    The rest of the section first highlights when the cross Jensen gap is convex and when it isn't, and then states the convergence result to a stationary point in the general case.

    \begin{remark}
        We can rewrite $f_\kappa$ as the difference of convex function $G_\kappa-\kappa\Omega_{\bar{\calM}}$ where
        \begin{equation}
           \label{eq:Gkappa} 
    G_\kappa(\bar\nu)
    :=
    \min_{q_\otimes\in\Delta_\otimes}
    \left\{
        \frac1N\sum_{i=1}^N\big(\langle\gamma_i,q_i\rangle+\kappa\Omega_i(q_i)\big)
        \;\middle|\;
        \frac1N\sum_{i=1}^N\Pi_{\bar{\calW}}\phi_iY(x_i)q_i=\bar\nu
    \right\}.
        \end{equation}
        Applying the difference of convex algorithm~\parencite{phamdinh1997convex} with this decomposition again leads to our alternating minimization algorithm.
    \end{remark}

    \paragraph{Convexity of the Jensen gap in the non-structured, non-contextual case.}
    
    In the non-structured case, the moment polytope $\calC$ is the simplex and is identical to the distribution polytope. In this case, both terms of the cross Jensen gap live in the distribution space. Using the interpretation of the Fenchel--Young loss as a primal--dual Bregman divergence \parencite{blondelLearningFenchelYoungLosses2020a}, we can express the Jensen gap as
    
    $$q_\otimes \mapsto \frac{1}{N}\sum_{i=1}^N \big(\Psi_\Delta(q_i) - \Psi_\Delta(\frac{1}{N}\sum_{i=1}^N q_i)\big) =\frac{1}{N}\sum_{i=1}^N D_{\Psi_{\Delta}}\big(q_i \mid \frac{1}{N} \sum_{i=1}^N q_i\big),$$

    where we used Proposition~\ref{prop:sum_sub_dim} to write $\Omega_{\Delta} = \Psi_{\Delta} + \bbI_{\Delta}$ with $\Psi_{\Delta}$ a Legendre-type function.
    This new expression reduces the convexity of the Jensen gap to the joint convexity of $D_{\Psi_{\Delta}}$. The joint convexity of a Bregman divergence is a long-standing and delicate question that has received substantial attention from the literature (see \textcite{bauschke2001joint} for a thorough analysis on the topic). Theorem 6.1 of \textcite{bauschke2001joint} provides the necessary and sufficient condition for $D_{\Psi_{\Delta}}$ to be convex, that is $(\nabla^2\Psi_\Delta)^{-1}$ is matrix-concave. This condition is known to hold for the negentropy and square-norm regularizers \parencite{bauschke2001joint}, but remains an open question for the sparse perturbation.

    \paragraph{Example of non-convexity in the structured case.}

Figure~\ref{fig:counter_example_overview} presents a counter-example showing that the use of a linear optimization layer on \(\calC\), together with a nonlinear cost over the original decisions, can destroy global convexity in the structured case. The instance has four feasible decisions \(\calY=\{y_1,y_2,y_3,y_4\}\) and three scenarios. It is constructed so that two adverse effects appear simultaneously. First, the scenario-wise anticipative minimizers are \(y_2,y_3,y_4\), whereas the unique optimal non-anticipative deterministic decision is \(y_1\), which is never anticipatively optimal. Consequently, in the anticipative limit \(\kappa\to0\), the relaxed objective \(f_\kappa\) is minimized at the average moment of \(y_2,y_3,y_4\), which lies on the side of the polytope associated with \(y_3\), rather than on the side associated with the true optimum \(y_1\). Second, when \(\beta>1\), the horizontal regions associated with \(y_1\) and \(y_3\) are separated by the vertical regions associated with \(y_2\) and \(y_4\), whose expected cost \((2\beta-1)/3\) creates a barrier controlled by \(\beta\). This makes the empirical risk \(w\mapsto R_N(\pi_w)\) non-convex (in this single-context instance, the parameter \(w\) is directly the score vector \(\theta\)), and an analogous strict convexity violation holds for \(f_\kappa\) for all sufficiently large \(\kappa\).  The corresponding computations (limiting forms of \(f_\kappa\), non-convexity of the structural Jensen gap, and the trajectory of the exact alternating minimization algorithm on the invariant horizontal subspace) are not included in the paper for brevity. 
    
    \begin{figure}[htbp]
        \centering
        \begin{minipage}{0.55\textwidth}
            \centering
            \resizebox{\textwidth}{!}{
                        \begin{tikzpicture}[scale=1.3]
            \draw[->, gray!70, dashed] (-1.5, 0) -- (1.5, 0);
            \draw[->, gray!70, dashed] (0, -1.5) -- (0, 1.5);

    \draw[thick, black!80] (-1, 0) -- (0, -1) -- (1, 0) -- (0, 1) -- cycle;

    \draw[->, very thick, blue] (0,0) -- (0.6, 0.2) node[anchor=north west] {$\theta$};

    \coordinate (Y1) at (-1, 0);
    \fill[green!60!black] (Y1) circle (1.5pt);
    \node[green!60!black, above left] at (Y1) {$y_1$};
    \node[green!60!black, align=right, below left] at (Y1) {
        $c(y_1, \xi) = (0, 0, 0)$ \\ 
        $\bar{c}(y_1) = 0$
    };

    \coordinate (Y2) at (0, -1);
    \fill[red] (Y2) circle (1.5pt);
    \node[red, below right] at (Y2) {$y_2$};
    \node[red, align=left, below right, yshift=-0.2cm] at (Y2) {
        $c(y_2, \xi) = (-1, \beta, \beta)$ \\ 
        $\bar{c}(y_2) = \frac{2\beta - 1}{3}$
    };

    \coordinate (Y3) at (1, 0);
    \fill[orange!90!black] (Y3) circle (1.5pt);
    \node[orange!90!black, above right] at (Y3) {$y_3$};
    \node[orange!90!black, align=left, below right] at (Y3) {
        $c(y_3, \xi) = (1, -1, 1)$ \\ 
        $\bar{c}(y_3) = 1/3$
    };
    
    \coordinate (Y4) at (0, 1);
    \fill[red] (Y4) circle (1.5pt);
    \node[red, above right] at (Y4) {$y_4$};
    \node[red, align=left, above right, yshift=0.2cm] at (Y4) {
        $c(y_4, \xi) = (\beta, \beta, -1)$ \\ 
        $\bar{c}(y_4) = \frac{2\beta - 1}{3}$
    };

    \fill[black] (0,0) circle (1pt) node[below right] {$0$};
\end{tikzpicture}
            }
        \end{minipage}
        \hfill
        \begin{minipage}{0.35\textwidth}
            \centering
            \includegraphics[width=\textwidth]{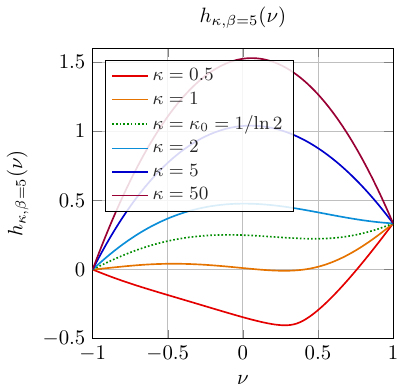}
        \end{minipage}
        \caption{Overview of the constructed counter-example. \emph{Left:} description of the problem with solutions, scenarios and costs \emph{Right:} value of $f_\kappa$ on the $x$-axis evaluated for $\beta=5$. 
        For small $\kappa$, the landscape is convex but biases towards the sub-optimal solution; for large $\kappa$, an explicit convexity violation appears and the dynamics can be directed toward different boundary regimes.}
        \label{fig:counter_example_overview}
    \end{figure}

    \paragraph{Convergence to a stationary point.}
In Appendix~\ref{sec:convergenceProof}, we adapt the proof of~\textcite{attouch_convergence_2013} to show that, under some conditions on the Bregman function and on the function minimized, the Bregman proximal point algorithm converges to a stationary point (Theorem~\ref{thm:generic_primal_convergence}).
One of the key conditions is that the function minimized satisfies the Kurdyka–\L ojasiewicz (KL) inequality.
In this paper, we make assumptions that our functions are real analytic, which is stronger and implies the (KL) inequality~\parencite{lojasiewicz1963propriete}.
This allows us to obtain the following convergence result.

\begin{restatable}{theorem}{thmconvergence}\label{thm:convergence:speed}
    Let $(w^{(t)})_{t \geq 0}$ be the sequence generated by Algorithm~\eqref{eq:alternating_product}. Additionally, assume that:
    \begin{itemize}
        \item ($A_1$) For every $i\in[N]$ and every
$\gamma_i\in\bbR^{\calY(x_i)}$, the Hessian of the mapping
\[\theta\mapsto\Omega_{\Delta_i}^*\big(Y(x_i)^\top\theta-\gamma_i\big)\]
exists and is positive definite on $V_i=\vect(\calY(x_i)-\calY(x_i))$.
        \item ($A_2$) $\nabla\Omega^*_{\Delta_i}$ is $L_i$-lipschitz continuous for all $i \in [N]$.
        \item ($A_3$) for all $i \in [N]$ and all $\gamma_i \in \bbR^{\calY(x_i)}$, the map $\theta \mapsto \Omega^*_{\Delta_i}\big(Y(x_i)^\top\theta - \gamma_i\big)$ is real analytic.
    \end{itemize}
    Then the identifiable trajectory $\bar w^{(t)}$ does exactly one of the two following.
    \begin{itemize}
        \item \textbf{(C) Confined regime.} The identifiable sequence converges toward a single stationary point $\bar w^*$ of $\underline{S_{\Delta,N}}$ with finite length, i.e., $\sum_{t > 0} ||\bar w^{(t+1)} - \bar w^{(t)}|| < +\infty$, and $\nabla \underline{S_{\Delta,N}}(\bar w^*) = 0$.
        Additionally, the function value converges at rate $\underline{S_{\Delta,N}}(\bar w^{(t)}) - \underline{S_{\Delta,N}}(\bar w^*) = \calO(1/t)$. On the primal side, $\bar \nu^{(t)}$ converges with finite length toward a stationary point $\bar{\nu}^*$ of $f_\kappa$, and $f_\kappa(\bar{\nu}^{(t)}) - f_\kappa(\bar{\nu}^*) = \calO (1/t)$
        \item \textbf{(E) Escape regime.} The identifiable sequence diverges, i.e., $||\bar w^{(t)}|| \rightarrow + \infty$. On the primal side, $\operatorname{dist}\big(\bar\nu^{(t)},\operatorname{rbd}(\bar{\calM})\big)\to0$.
    \end{itemize}
\end{restatable}

    Assumptions ($A_1$), ($A_2$), and ($A_3$) are mild and satisfied by most reasonable choices of regularization, as evidenced by Proposition~\ref{prop:regsatisfyconv}, whose proof is postponed to Appendix~\ref{sec:convergenceProof}.

    \begin{restatable}{proposition}{regularizationconv}\label{prop:regsatisfyconv}
    The negentropy regularization and the sparse perturbation with Gaussian noise satisfy assumptions ($A_1$), ($A_2$), and ($A_3$).
    \end{restatable}

    \section{New results in structured prediction}\label{sec:new_structured_prediction}

    {
    Introducing our algorithm required many new results on structured prediction
    with Fenchel--Young losses in the combinatorial setting of this paper. Some of them are technical: the theory of Fenchel--Young losses has been developed for full-dimensional polytopes. This assumption is satisfied neither by the polytopes we use in operations research applications nor by the probability simplex. We deal with this issue in Appendix~\ref{sec:regularization_on_distributions}.
    In this section we focus on the new results that are more original and that we believe may have an impact beyond our alternating minimization algorithm.
    Section~\ref{sec:structured_perturbation} introduces a sparse perturbation
    directly on the distribution space~$\Delta^\calY$, extending the framework
    of~\textcite{berthetLearningDifferentiablePertubed2020} to the case where the
    combinatorial space~$\calY$ is not a continuous polytope.
    Section~\ref{sec:loss_parameter_space} studies the geometry induced by aggregating
    Fenchel--Young losses across training scenarios, characterizing the identifiable
    aggregate moment space~$\bar{\calM}$, and the associated aggregate regularizer
    $\bar{\Omega}_{\bar{\calM}}$.
    Proofs for both subsections are deferred to Appendices~\ref{subsec:appendix_structured_perturbation} and~\ref{subsec:appendix_loss_parameter_space}, respectively.}

    \subsection{Sparse perturbation on the distribution space}\label{sec:structured_perturbation}
    We use the notations defined in Section~\ref{sec:intro} for both the variable and distribution spaces. Explicitly defining a proper l.s.c. convex regularization function $\Omega \in \Gamma_0(\bbR^d)$ with domain $\dom(\Omega) = \calC$, and computing the regularized predictions $\hat{y}_\Omega(\theta)$ defined in Equation~\eqref{eq:COlayer} can be challenging. It may rely on Frank-Wolfe~\parencite{frankAlgorithmQuadraticProgramming1956} algorithm in practice. 
    We follow another approach pioneered by \textcite{berthetLearningDifferentiablePertubed2020}, defining instead $\Omega_\calC^*$ and $\Omega_\Delta^*$ directly.  
    More precisely, let $\varepsilon \in \bbR_{++}$, we introduce:
    \begin{equation}\label{eq:perturbation_moment}
        F_{\varepsilon, \calC}(\theta) = \bbE[\max_{y \in \calY}(\theta + \varepsilon \bfZ)^\top y] = \bbE[\max_{y \in \calC}(\theta + \varepsilon \bfZ)^\top y],
    \end{equation}
    \begin{equation}\label{eq:perturbation_distribution}
        F_{\varepsilon,\Delta}(s) = \bbE[\max_{y\in \calY}s(y) + \varepsilon \bfZ^\top y] = \bbE[\max_{q\in \Delta^{\calY}}(s + \varepsilon Y^\top \bfZ)^\top q], 
    \end{equation}
    where $\bfZ$ is a centred random variable on $\bbR^d$ from an exponential family with positive density, typically a standard multivariate normal distribution. 
    The perturbed linear program in Equation~\eqref{eq:perturbation_moment} is introduced by \textcite{berthetLearningDifferentiablePertubed2020}, while Equation~\eqref{eq:perturbation_distribution} is new to the best of our knowledge. 
    We denote by $\Omega_{\varepsilon,  \calC}$ and $\Omega_{\varepsilon, \Delta}$ their respective Fenchel conjugates.
    We extend from the work of~\textcite[Proposition 2.2]{berthetLearningDifferentiablePertubed2020} the following properties for $F_{\varepsilon, \calC}$ to the case when $\calC$ is not full-dimensional.
    
    \begin{restatable}{proposition}{propperturbation}\label{prop:Perturbation}
        Let $\varepsilon \in \bbR_{++}$, the function $F_{\varepsilon, \calC}$ defined above has the following properties:
        \begin{enumerate}
            \item $F_{\varepsilon, \calC}$ is a convex finite valued function of $\bbR^d$, and in particular belongs to $\Gamma_0(\bbR^d)$.
            \item $F_{\varepsilon, \calC}$ is strictly convex over $V$, and affine over $V^\perp$. Let $\theta \in \bbR^d$, such that ${\theta = \theta_V + \theta_{V^\perp}}$, where $\theta_V = \Pi_V(\theta)$ and $\theta_{V^\perp} = \theta - \theta_V$, and $y_0 \in \calC$,
            \[F_{\varepsilon, \calC}(\theta) = \langle y_0 |\theta_{V^\perp} \rangle + F_{\varepsilon, \calC}(\theta_V).\]
            \item $F_{\varepsilon, \calC}$ is twice differentiable, with gradient given by:
            \begin{equation}\label{eq:grad_F_calC}
                \nabla_\theta F_{\varepsilon, \calC}(\theta) = \bbE[\argmax_{y \in \calC}(\theta + \varepsilon \bfZ)^\top y].
            \end{equation}
            \item The Fenchel conjugate $\Omega_{\varepsilon, \calC} := F_{\varepsilon, \calC}^*$ has domain $\calC$, and its restriction to $H$ is a Legendre-type function.
        \end{enumerate}
    \end{restatable}
    Contrary to the full dimension case considered by \textcite{berthetLearningDifferentiablePertubed2020}, $F_{\varepsilon, \calC}$ is not strictly convex over the whole space~$\bbR^d$. Therefore, it is not a Legendre-type function, but its restriction to~$V$ is. Besides, its conjugate is not Legendre-type, but the restriction of its conjugate to the affine subspace~$H$ is.
    The proof is provided in Appendix~\ref{subsec:appendix_structured_perturbation}.

    The perturbation in the definition of $F_{\varepsilon, \Delta}$ (see Equation~\eqref{eq:perturbation_distribution}) spans $\im(Y^\top)$, which is a subspace of dimension $d' \ll |\calY|$. 
    The proofs of \textcite{berthetLearningDifferentiablePertubed2020}, or the change of variable in the paper by \textcite{abernethyPerturbationTechniquesOnline2016}, no longer hold.
    However, perhaps surprisingly, many properties remain valid.
    
    \begin{restatable}{theorem}{thmsparseperturbation}
        \label{prop:SparsePerturbation}
        Let $\varepsilon \in \bbR_{++}$, the function $F_{\varepsilon, \Delta}$ defined above has the following properties:
        \begin{enumerate}
            \item The function $F_{\varepsilon, \Delta}$ is convex,  Lipschitz continuous (and in particular in~$\Gamma_0(\bbR^{\calY})$). 
            \item $F_{\varepsilon, \Delta}$ is strictly convex over $V_\Delta$, and affine over $V_\Delta^\perp = \vect(\mathbf{1})$. More precisely, let $s \in \bbR^{\calY}$, decomposed as $s = s_{V_\Delta} + s_{V_\Delta^{\perp}}$, where $s_{V_\Delta} = \Pi_{V_{\Delta}}(s)$ and $s_{V_\Delta^{\perp}} = s - s_{V_\Delta}$, and $q_0 \in \Delta^{\calY}$ be any point in $\Delta^{\calY}$,
            \[F_{\varepsilon, \Delta}(s) = \langle s_{V_\Delta^\perp} | q_0 \rangle + F_{\varepsilon, \Delta}(s_{V_\Delta}).\]
            \item $F_{\varepsilon, \Delta}$ is differentiable over~$\bbR^{\calY}$, with gradient given by:
            \begin{equation}\label{eq:grad_F_Delta}
                \nabla_s F_{\varepsilon, \Delta}(s) = \bbE[\argmax_{q \in \Delta^{\calY}} (s +\varepsilon Y^\top \bfZ)^\top q].
            \end{equation}
            \item The Fenchel conjugate $\Omega_{\varepsilon, \Delta} := F_{\varepsilon, \Delta}^*$ has domain $\Delta^{\calY}$, and its restriction to $H_\Delta$ is Legendre-type.
        \end{enumerate}
    \end{restatable}
    \begin{proof}
        See Appendix~\ref{subsec:appendix_structured_perturbation}.
    \end{proof}

    Now, we show another property of~$F_{\varepsilon, \Delta}$, useful to get the performance bound in Theorem~\ref{thm:bound_risk}.

    \begin{restatable}{proposition}{propstrongconvexitysparseperturbation}\label{prop:strongConvexitySparsePerturbation}
    Let $\bfZ \sim \calN(0, I_d)$ be a standard Gaussian. Then $\nabla_s F_{\varepsilon, \Delta}(s)$ is Lipschitz continuous with constant
    $ L = \tfrac{2|\calY|^2}{\varepsilon m_\calY \sqrt{\pi}} $
    where $m_\calY := \min_{(y, y') \in \calY^2, y \neq y'}||y-y'||_2$ is the minimum distance between two distinct points in $\calY$. Consequently, the Fenchel conjugate $\Omega_{\varepsilon, \Delta}$ is strongly convex on $H_\Delta$ with parameter
    $ \mu = \tfrac{\varepsilon m_\calY \sqrt{\pi}}{2|\calY|^{2}} .$
\end{restatable}

    \begin{proof}
        See Appendix~\ref{subsec:appendix_structured_perturbation}.
    \end{proof}

    Finally, we show that $\Omega_{\varepsilon, \calC}$ is the structured regularization corresponding to $\Omega_{\varepsilon, \Delta}$.
    \begin{restatable}{proposition}{propstructuredpredictionsparseperturbation}
        \label{prop:structuredPredictionSparsePerturbation}
        $\Omega_{\varepsilon, \calC}(\mu) = \min_{q \colon Yq=\mu}\Omega_{\varepsilon, \Delta}(q)$, and hence all the properties of Proposition~\ref{prop:structuredPredictionWithGeneralizedNegentropy} are true in the sparse perturbation case.
    \end{restatable}
    \begin{proof}
        See Appendix~\ref{subsec:appendix_structured_perturbation}.
    \end{proof}

    \subsection{Aggregating Fenchel--Young losses: the geometry of linear composition}\label{sec:loss_parameter_space}

Let us suppose that we have a collection of polytopes $\calP_1, \ldots, \calP_N$, where each $\calP_i \subset \bbR^{d_i}$ is not necessarily full-dimensional. 
We denote by $H_i = \aff(\calP_i)$ the affine hull of $\calP_i$, and by $V_i$ the corresponding parallel linear subspace. 
We consider a collection of regularization functions $\Omega_{\calP_i}$ such that $\dom(\Omega_{\calP_i}) = \calP_i$. 
Since $\calP_i$ is not full-dimensional, $\Omega_{\calP_i}$ cannot be of Legendre-type over $\bbR^{d_i}$. 
Therefore, we assume that the restriction of $\Omega_{\calP_i}$ to $H_i=\aff(\calP_i)$ is of Legendre type.

We introduce a parameter vector $w \in \bbR^d$ and linear maps denoted by matrices $A_i \colon \bbR^{d_i} \to \bbR^d$. The parameter affects the $i$-th polytope through $A_i^\top w$. The aggregate moment space is defined as $\calM = \left\{ \frac{1}{N} \sum_{i=1}^N A_i \pi_i \mid \pi_i \in \calP_i \right\} \subset \bbR^d$. Just like the individual sets $\calP_i$, the space $\calM$ is not necessarily full-dimensional. Let us define $H_\calM = \aff(\calM)$, and by $\bar \calW$ the parallel linear subspace, and by $\bar \calM$ the full-dimensional restriction of $\calM$ to $H_\calM$.
Since the affine hull of a Minkowski sum is the Minkowski sum of the affine hulls, we have $\bar \calW = \sum_{i=1}^N A_i V_i$.
We denote by $J_{\Pi_{\bar\calW}}$ the canonical injection mapping an identifiable parameter $\bar w \in \bar \calW$ to the full parameter space $\bbR^d$, meaning $w = J_{\Pi_{\bar\calW}} \bar w$. The adjoint operator $\Pi_{\bar\calW}$ projects the aggregated moments into the identifiable structure, creating the full-dimensional restricted space $\bar \calM = \Pi_{\bar\calW} \calM$. 

Let us define the sum of dual regularizers over the identifiable parameter $\bar w \in \bar \calW$, and its corresponding conjugate:
\begin{equation}\label{eq:OmegaCalMstarDefinition}
    \bar F(\bar w) = \frac{1}{N} \sum_{i=1}^N \Omega_{\calP_i}^*(A_i^\top J_{\Pi_{\bar\calW}} \bar w), \quad \text{and} \quad \Omega_{\bar \calM}(\bar \nu) = \bar F^*(\bar \nu).
\end{equation}

\begin{restatable}{proposition}{propbaromegaproperties}\label{prop:barOmegaProperties}

    $\Omega_{\bar \calM}$ has the following properties:
\begin{enumerate}
    \item $\bar F$ is of Legendre type, hence $\bar F = \Omega_{\bar \calM}^*$. 
    \item $\Omega_{\bar \calM}$ admits the following infimal-projection representation.
    \begin{equation}
          \label{eq:expressionOfBarOmega}
          \Omega_{\bar \calM}(\bar \nu) = \inf_{(\pi_i)_{i=1}^N} \Big\{\frac{1}{N}\sum_{i=1}^N\Omega_{\calP_i}(\pi_i)\colon \pi_i \in \calP_i \text{ and } \frac{1}{N}\sum_{i=1}^N \Pi_{\bar\calW} A_i\pi_i=\bar \nu\Big\}.
      \end{equation}
    \item The domain of $\Omega_{\bar \calM}$ is $\bar \calM = \left\{ \bar \nu = \frac{1}{N} \sum_{i=1}^N \Pi_{\bar\calW} A_i \pi_i \mid \pi_i \in \dom \Omega_{\calP_i} \right\}$.
    \item For any $\bar\nu\in\relint(\bar\calM)$, the minimum in~\eqref{eq:expressionOfBarOmega} is attained in $(\pi_i)_{i=1}^N$ defined by $\pi_i = \nabla \Omega_{\calP_i}^*\big(A_i^\top J_{\Pi_{\bar\calW}} \nabla \Omega_{\bar \calM}(\bar \nu)\big)$.
\end{enumerate}
\end{restatable}

    \begin{proof}
        See Appendix~\ref{subsec:appendix_loss_parameter_space}.
    \end{proof}

\begin{restatable}{proposition}{propfylandcalm}\label{prop:FYLandCalM}
    Given $\pi_1,\ldots,\pi_N$ with $\pi_i$ in $\calP_i$ and $w$ in $\calW$, let $$\begin{cases}
               \bar w = \Pi_{\bar \calW}(w) &\text{ for } w \in \argmin \frac{1}{N} \sum_{i=1}^N \calL_{\Omega_{\calP_i}}(A_i^\top w,\pi_i), \\
               \bar \nu = \Pi_{\bar \calW}(\nu) &\text{ for } \nu = \frac{1}{N} \sum_{i=1}^N A_i \pi_i,
            \end{cases} $$ \\
            we have
        \begin{align}
            \label{eq:FYlandCalM_regsec}
            \Pi_{\bar \calW}(w) &= \nabla \Omega_{\bar \calM}\big(\Pi_{\bar \calW}(\nu)\big) \\
            \frac{1}{N}\sum_{i=1}^N\big( \Omega_{\calP_i}(\pi_i) - \Omega_{\bar \calM}(\bar \nu)\big) 
            &= \frac{1}{N} \sum_{i=1}^N \calL_{\Omega_{\calP_i}}\big(A_i^\top J_{\Pi_{\bar \calW}}\bar w, \pi_i\big).
            \label{eq:crossJensenGapInCalM_regsec}
        \end{align}
\end{restatable}

    \begin{proof}
        See Appendix~\ref{subsec:appendix_loss_parameter_space}.
    \end{proof}
\section{Computational experiments}\label{sec:computational_exp_primal_dual}

The goal of our numerical experiments is to evaluate when the proposed primal-dual algorithm improves the learning of policies encoded as neural networks with combinatorial optimization layers. We benchmark against state-of-the-art baselines from the literature, including imitation-based CO-layer methods~\parencite{dalleLearningCombinatorialOptimization2022, batyCombinatorialOptimizationEnrichedMachine2024} and contextual stochastic optimization approaches~\parencite{bertsimas2020predictive}.
We first consider a contextual two-stage minimum-weight spanning-tree problem in Section~\ref{subsec:two_stage_max_weight_spanning_tree_md}, where the primal-dual method recovers the quality of a computationally heavy fully-coordinated scheme with a lightweight coordination procedure. 
We then evaluate the method on a stochastic vehicle-scheduling benchmark in Section~\ref{subsec:stovsp}, showing that it scales to a real-world application and remains competitive in a regime where uncoordinated imitation is already strong. 
Finally, we turn to a contextual assortment problem in Section~\ref{subsec:contextual_assortment}, a genuinely contextual stochastic optimization setting where decisions must adapt to the observed customer context; there, the learned policy is competitive with contextual SAA methods while requiring only a small fraction of their inference time.

\begin{algorithm}
\caption{Implemented primal-dual algorithm (sparse perturbation version)}
\begin{algorithmic}[1]
\Require $\calD_{\texttt{train}}=((x_i,\xi_i))_{i\in[n_{\texttt{train}}]}$, $\texttt{nb\_iterations}$, $\kappa$, $\varepsilon$, $\texttt{nb\_samples}$, $\texttt{nb\_epochs}$, $\texttt{lr\_init}$, $\varphi_{\cdot}$;
\State $w = 0_{\calW}$;
\State $\texttt{loss} = \texttt{Fenchel\_Young\_Loss}(\texttt{oracle}, \varepsilon, \texttt{nb\_samples})$; \Comment{Generated by \eqref{eq:perturbation_moment}}

\For{$t = 1$ to $\texttt{nb\_iterations}$}
    \For{$(x_i,\xi_i)$ in $\calD_{\texttt{train}}$}
        \If{$t=1$}
            \State $\hat\mu_i = \texttt{oracle}(x_i,\xi_i)$; \Comment{Uncoordinated anticipative target}
        \Else
            \State $\hat\mu_i =
            \texttt{perturbed}(\texttt{oracle}, \varphi_w, x_i, \xi_i, \kappa, \varepsilon, \texttt{nb\_samples})$;
            \Comment{MC for \eqref{eq:primalUpdatePerturbation}}
        \EndIf
    \EndFor
    \State $\texttt{Optimizer} = \texttt{Adam}(\texttt{lr\_init})$;
    \For{$\texttt{epoch} = 1$ to $\texttt{nb\_epochs}$}
        \Comment{Coordination \eqref{eq:coordination_w_dist}}
        \For{$(x_i,\_)$ in $\calD_{\texttt{train}}$}
            \State $\texttt{grads} =
            \texttt{compute\_gradient}(\texttt{loss}(\varphi_w(x_i), \hat\mu_i))$;
            \State $w = \texttt{update\_model}(\varphi_w, \texttt{grads}, \texttt{Optimizer})$;
        \EndFor
    \EndFor
\EndFor

\end{algorithmic}
\label{algo:primal_dual}
\end{algorithm}

{In the practical implementation of the primal-dual described in Algorithm~\ref{algo:primal_dual}, \texttt{oracle} denotes the deterministic single-scenario solver of Assumption~\ref{ass:oracle}, called with a linear perturbation of the objective.
The routine \texttt{perturbed} implements the sparse-perturbation decomposition step~\eqref{eq:primalUpdatePerturbation}: it draws $\texttt{nb\_samples}$ independent perturbations, solves the resulting deterministic problems, and returns a Monte Carlo estimate of the moment target $\mu_i$. 
The object \texttt{Fenchel\_Young\_Loss} is the sparse-perturbation Fenchel--Young loss generated by~\eqref{eq:perturbation_moment}.
In the coordination step~\eqref{eq:coordination_w_dist}, \texttt{compute\_gradient} differentiates this loss with respect to the model output $\varphi_w(x_i)$, and Adam updates the parameters~$w$.
Across the experiments, the training data set $\calD_{\texttt{train}}$ contains context--noise observations used to train the policy, while validation and test data sets are used only to tune hyperparameters and report out-of-sample performance.

    The first step of Algorithm~\ref{algo:primal_dual} is a classic uncoordinated imitation-learning framework with Fenchel--Young losses~\parencite{batyCombinatorialOptimizationEnrichedMachine2024}, which we use as a benchmark below.
The main differences with the exact primal-dual algorithm of Eq~\eqref{eq:alternating_product} are that: i) we approximate the expectation in~\eqref{eq:primalUpdatePerturbation} with a Monte Carlo estimate, and ii) we approximate the coordination step~\eqref{eq:coordination_w_dist} with a stochastic gradient descent.
}

    \subsection{Two-stage minimum weight spanning tree}\label{subsec:two_stage_max_weight_spanning_tree_md}
    
    \paragraph{Problem and data.}
    {We consider the contextual two-stage minimum weight spanning tree problem of \textcite{dalleLearningCombinatorialOptimization2022}: given an undirected graph $G=(V,E)$, the goal is to build a spanning tree over two stages at minimum cost, where second-stage edge costs depend on an exogenous noise $\bfxi$ unknown at the first stage, but a context $\bfx$ correlated to $\bfxi$ is observed. We refer to \textcite{dalleLearningCombinatorialOptimization2022} for the full problem formulation and to the open-source Julia package\footnote{\url{https://batyleo.github.io/TwoStageSpanningTree.jl/stable/}} for the implementation. Instances are defined on $20\times 20$ grid graphs; we use $\texttt{train\_size}=50$, $\texttt{val\_size}=50$, and $\texttt{test\_size}=50$ instances, each with $20$ scenarios ($n_{\texttt{train}}=1000$ context--noise observations). The full mathematical formulation, the policy model, and the algorithm hyperparameters are detailed in Appendix~\ref{sec:appendix_spanning_tree_details}.}

    \paragraph{Policies and benchmarks.}
    {We derive three benchmark policies for this problem.
    The median policy~$\pi_\texttt{median}$ is not learned: it solves the deterministic single-scenario problem~\eqref{eq:DeterministicProblem} after replacing the unknown cost vector by an estimator of its median. 
    The uncoordinated imitation policy~$\pi_{w^{(1)}}$ solves one deterministic problem~\eqref{eq:DeterministicProblem} per scenario and then imitates the resulting solutions with a Fenchel--Young loss~\parencite{batyCombinatorialOptimizationEnrichedMachine2024}.
    The fully-coordinated imitation policy~$\pi_{w^L}$ imitates first-stage solutions~$y_i^L$ obtained by solving, for each context~$x_i$, a sample average approximation problem with several noise realizations. We compute these targets with a Lagrangian heuristic as in~\textcite[Section~6.5]{dalleLearningCombinatorialOptimization2022}. Both imitation benchmarks use the same neural-network architecture as our primal-dual policy, but generating the fully-coordinated training set is much more computationally intensive.}
    
    \paragraph{Results.}
    We plot validation and test estimated average gaps over iterations in Figure~\ref{fig:mirror_descent_2stage_st_md}. 
    {The median policy~$\pi_{\texttt{median}}$ has roughly $12\%$ average validation and test gaps, while the fully-coordinated policy~$\pi_{w^L}$ reaches average gaps close to~$2\%$. The uncoordinated policy~$\pi_{w^{(1)}}$ reaches $4.3\%$ average validation and test gaps. Along the outer iterations, the current-weight policy~$\pi_{w^{(t)}}$ shows small oscillations, while the averaged policy~$\pi_{\bar w}$, based on $\bar w^{(t)} = \frac{1}{t} \sum_{t' \leq t} w^{(t')}$, converges more smoothly and reaches the performance of the fully-coordinated benchmark.}

    \begin{figure}[ht]
        \centering
        \includegraphics[width=0.99\textwidth]{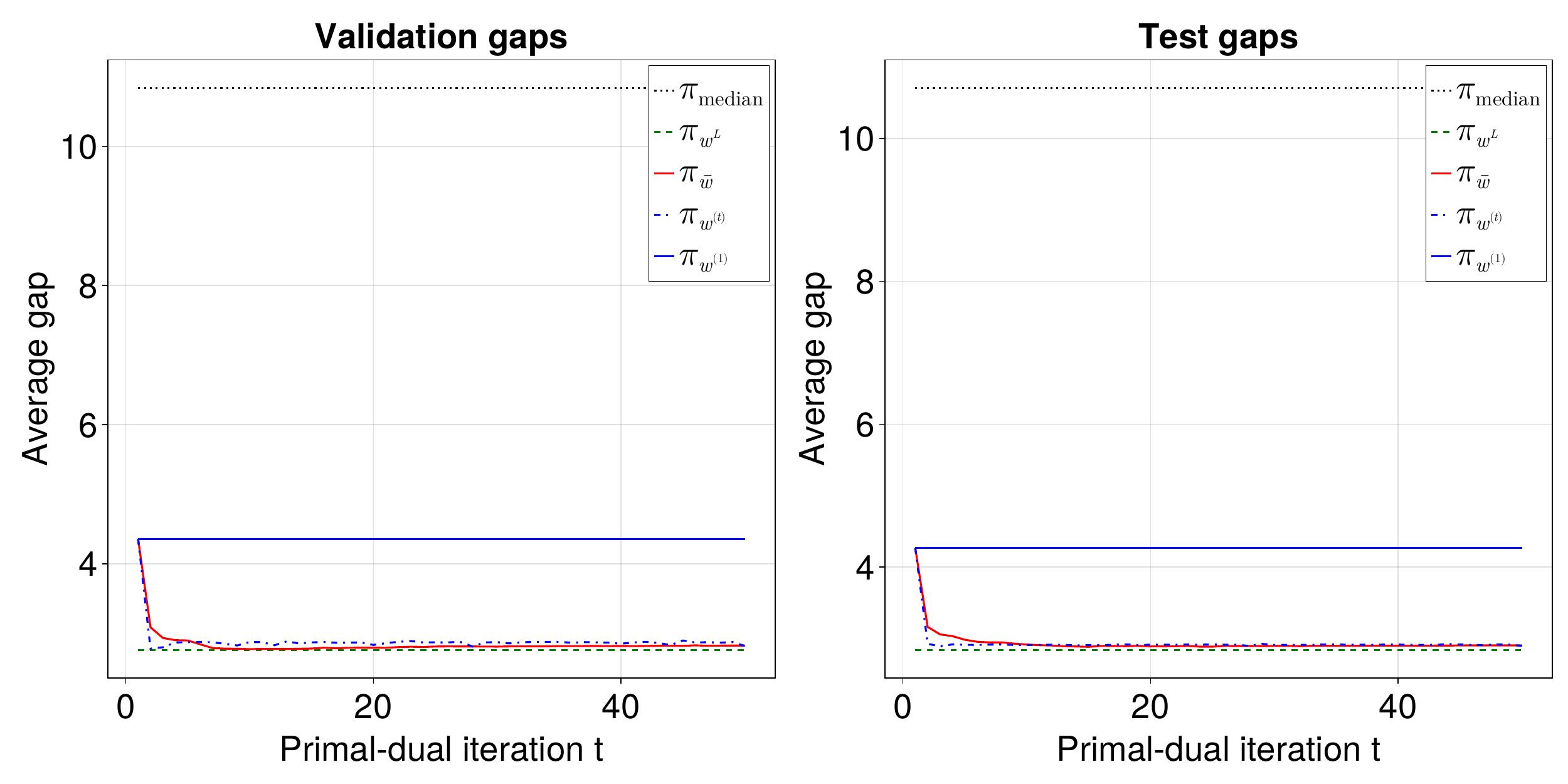}
        \caption{Validation and test average gaps of policies over the iterations of our primal-dual algorithm, for the two-stage minimum weight spanning tree.}
        \label{fig:mirror_descent_2stage_st_md}
    \end{figure}
    
    \begin{result}
        The averaged primal-dual policy~$\pi_{\bar w}$ improves over the uncoordinated imitation benchmark~$\pi_{w^{(1)}}$ and reaches the performance of the computationally demanding fully-coordinated benchmark~$\pi_{w^L}$, using the same input dataset and assumptions as the uncoordinated policy.
    \end{result}

    \subsection{Stochastic vehicle scheduling}\label{subsec:stovsp}

We now turn to a more complex problem, the stochastic vehicle scheduling problem (StoVSP) of \textcite{parmentierLearningApproximateIndustrial2022a}. We study the role of the perturbation scale~$\varepsilon$ across the three policies, and show that the averaged primal-dual policy improves over both imitation baselines at the optimal scale while remaining robust to large values of~$\varepsilon$ where the baselines degrade.

    \paragraph{Problem and data.}
    A first-stage decision assigns a fleet of vehicles to a set of tasks, represented as disjoint paths covering every task exactly once in a task graph; a vehicle performing task $v$ immediately after task $u$ corresponds to a binary variable $y_{u,v}$, so that $\calY(x)$ is again a $\{0,1\}$-polytope satisfying Assumption~\ref{rem:exposed_vertices}. In the second stage, random intrinsic delays~$\gamma_v(\xi)$ at each task propagate along the chosen vehicle routes through the recursion $d_v(\xi) = \gamma_v(\xi) + \max\big(d_u(\xi) - \delta_{u,v}(\xi), 0\big)$, where $\delta_{u,v}(\xi)$ is the slack time available between consecutive tasks $u$ and $v$. The cost~$c(x,y,\xi)$ sums the fixed routing cost of~$y$ and the resulting delay cost. We rely on the open-source implementation of this benchmark\footnote{\url{https://juliadecisionfocusedlearning.github.io/DecisionFocusedLearningBenchmarks.jl/stable/benchmarks/stochastic/02_vsp/}} to generate $200$ instances, with $\texttt{nb\_scenarios}=25$ delay scenarios drawn per instance and  $\texttt{nb\_tasks}=20$ tasks each. We split into $\texttt{train\_samples}=120 \times 25$, $\texttt{val\_samples}=40 \times 25$, and $\texttt{test\_samples}=40 \times 25$ samples.

    \paragraph{Policies and benchmarks.}
    For each instance~$x$, a GLM~$\varphi_w$ maps task features to a score vector~$\theta=\varphi_w(x)$ indexed by the arcs of the task graph. The corresponding combinatorial layer solves the deterministic vehicle-scheduling problem with~$\theta$ as the \texttt{oracle} in Algorithm~\ref{algo:primal_dual}. We use the same primal-dual implementation as in the spanning-tree experiment, fix the regularization weight~$\kappa=1$, and tune the perturbation scale~$\varepsilon$.

    Because the StoVSP instances are small enough, we solve the sample average approximation (SAA) problem of each instance to optimality with a compact mixed-integer program. 
    This gives two reference quantities: an exact fully-coordinated target solution for training~$\pi_{w^L}$, and the exact SAA cost evaluated on the otherwise unobserved test scenarios, which we use as an oracle bound in Table~\ref{tab:stovsp_epsilon_tuning}.
    We compare three learned policies: the uncoordinated imitation policy~$\pi_{w^{(1)}}$, the fully-coordinated imitation policy~$\pi_{w^L}$, and the averaged primal-dual policy~$\pi_{\bar w}$.

    \paragraph{Results.}
    For each value of~$\varepsilon$ on a log-spaced grid from~$10^{-2}$ to~$10^{2}$, we train three policies sharing the same perturbation scale: $\pi_{w^L}$, $\pi_{w^{(1)}}$, and~$\pi_{\bar w}$ with its iteration selected by validation cost (out of $T=20$ outer iterations). Table~\ref{tab:stovsp_epsilon_tuning} reports the resulting average test costs alongside the SAA oracle ($5646.4$).

    \begin{table}[ht]
    \centering
    \caption{Test cost as a function of the perturbation scale~$\varepsilon$ for the stochastic vehicle scheduling problem (SAA oracle: $5646.4$). For~$\pi_{\bar{w}}$, the iteration is selected by validation cost; brackets show the selected iteration out of $T=20$.}
    \label{tab:stovsp_epsilon_tuning}
    \begin{tabular}{l ccc}
    \toprule
    $\varepsilon$ & $\pi_{w^L}$ & $\pi_{w^{(1)}}$ & $\pi_{\bar{w}}$ (best iter.) \\
    \midrule
    $0.01$  & $5672.7\ (+0.47\%)$ & $5691.5\ (+0.80\%)$ & $5672.4\ (+0.46\%)\ [4]$  \\
    $0.1$   & $5661.5\ (+0.27\%)$ & $5663.5\ (+0.30\%)$ & $5663.5\ (+0.30\%)\ [1]$  \\
    $1$     & $5660.7\ (+0.25\%)$ & $5661.5\ (+0.27\%)$ & $5664.0\ (+0.31\%)\ [2]$  \\
    $5$     & $5663.6\ (+0.31\%)$ & $5663.4\ (+0.30\%)$ & $\mathbf{5659.0}\ (+0.22\%)\ [20]$ \\
    $10$    & $5667.5\ (+0.37\%)$ & $5661.7\ (+0.27\%)$ & $\mathbf{5658.5}\ (+0.22\%)\ [19]$ \\
    $50$    & $5687.5\ (+0.73\%)$ & $5690.2\ (+0.78\%)$ & $5664.6\ (+0.32\%)\ [20]$ \\
    $100$   & $5713.5\ (+1.19\%)$ & $5719.5\ (+1.30\%)$ & $5672.7\ (+0.47\%)\ [20]$ \\
    \bottomrule
    \end{tabular}
    \end{table}

At small scales ($\varepsilon \leq 1$), the algorithm converges in one or two iterations, so~$\pi_{\bar w}$ matches~$\pi_{w^{(1)}}$ and all policies are within $0.3\%$ of the oracle.
At the sweet spot ($\varepsilon \in \{5, 10\}$), $\pi_{\bar w}$ runs the full budget and improves over both~$\pi_{w^{(1)}}$ and the computationally expensive~$\pi_{w^L}$ ($+0.22\%$ vs.\ $+0.27$--$0.37\%$).
At large scales ($\varepsilon \geq 50$), both baselines degrade significantly while~$\pi_{\bar w}$ recovers via model averaging, staying within $0.5\%$ of the oracle.

\begin{result}
    On the stochastic vehicle scheduling benchmark, uncoordinated imitation is already strong at small perturbation scales. 
    At the optimal scale ($\varepsilon \in \{5, 10\}$), the averaged primal-dual policy~$\pi_{\bar w}$ outperforms both imitation baselines, and remains robust at large scales where baselines degrade.
\end{result}

\subsection{Contextual assortment}\label{subsec:contextual_assortment}

We evaluate Algorithm~\ref{algo:primal_dual} in a contextual stochastic assortment setting where the customer context changes the optimal assortment~\parencite{chen2020dynamic}.
We compare with contextual SAA baselines, assess scalability to larger catalogs, and show that the learned policy transfers to new catalogs unseen at training time.
Instance parameters, model architecture, and hyperparameters are detailed in Appendix~\ref{sec:appendix_assortment_details}.

\paragraph{Problem and data.}
Each product $i$ has a feature vector $x_i^p$ (first coordinate: price), and the seller observes a customer context $x^c$ before selecting an assortment $y\in\{0,1\}^N$ with at most $3$ offered products.
Customer utility for an offered product is $u_i=(x_i^p)^\top B x^c+\varepsilon_i$ with i.i.d.\ Gumbel noise~$\varepsilon_i$; the seller receives the price of the chosen product, or zero if the customer takes the outside option.

\paragraph{Policies and benchmarks.}
We use a learned bilinear score model over customer and product features; implementation details are given in Appendix~\ref{sec:appendix_assortment_details}. 
We compare against non-contextual SAA (using historical scenarios in the training set), nearest-neighbor (kNN-SAA), Gaussian-kernel (Gaussian-SAA), and random-forest (RF-SAA) contextual baselines~\parencite{bertsimas2020predictive}, and the uncoordinated imitation policy~$\pi_{w^{(1)}}$.
Hyperparameters are selected by validation performance.

\paragraph{Results.}
\textbf{Small catalog.}
For the small catalog setting $N=10$, Table~\ref{tab:contextual_assortment_small_benchmark} reports mean values by instance seed; all entries except SAA are relative improvements over SAA.

The primal-dual iterations turn a poor uncoordinated imitation policy into a strong contextual policy. The validation-selected policy $\bar{\pi}^{50}$ always improves over SAA, is most of the time better than the contextual-SAA prescriptions, and gives the best average performance.

\begin{table}[!ht]
  \centering
  \resizebox{\textwidth}{!}{%
  \begin{tabular}{lcccccccc}
    \toprule
    \textbf{Seed} & \textbf{SAA} & \textbf{kNN} & \textbf{Gaussian} & \textbf{RF} & $\boldsymbol{\pi_{w^{(1)}}}$ & $\boldsymbol{\pi^{50}}$ & $\boldsymbol{\bar{\pi}^{50}}$ & $\boldsymbol{\bar{\pi}^{100}}$ \\
    \midrule
    1 & 9.07 & -0.3\% & -0.1\% & -0.1\% & -34.3\% & +0.2\% & \textbf{+0.2\%} & -0.0\% \\
    2 & 6.39 & +3.8\% & \textbf{+4.5\%} & +4.5\% & -47.0\% & -1.1\% & +1.5\% & +1.8\% \\
    3 & 5.65 & +11.0\% & +11.6\% & +11.6\% & -30.9\% & +13.9\% & \textbf{+14.4\%} & \textbf{+14.4\%} \\
    4 & 7.25 & +9.8\% & +8.7\% & +9.3\% & -40.4\% & +8.7\% & \textbf{+10.7\%} & +10.6\% \\
    5 & 7.34 & +0.4\% & +0.9\% & +0.8\% & -30.0\% & -0.6\% & +1.2\% & \textbf{+1.3\%} \\
    \midrule
    Average & 7.14 & +4.4\% & +4.6\% & +4.7\% & -36.4\% & +3.7\% & \textbf{+5.6\%} & \textbf{+5.6\%} \\
    \bottomrule
  \end{tabular}%
  }
  \caption{Small fixed-catalog benchmark. The SAA column reports mean test value; all other entries report relative improvement over SAA. Entries are averaged over data seeds; $\pi^{50}$ is the last iterate after 50 outer iterations, while $\bar{\pi}^{50}$ and $\bar{\pi}^{100}$ are validation-selected iterates. Primal-dual entries are additionally averaged over ten algorithmic replications per instance--data pair.}
  \label{tab:contextual_assortment_small_benchmark}
\end{table}

\textbf{Scaling.}
The larger-scale runs show the same overall pattern as the small benchmark, and in this pooled comparison the primal-dual policies improve over the contextual-SAA prescription. At $N=25$ (resp. $N=50$), the best primal-dual policy improves over SAA by $2.09\%$ (resp. $1.62\%$), compared with $0.80\%$ (resp. $0.45\%$) for kNN-10.

As table~\ref{tab:contextual_assortment_scaling} shows, contextual-SAA methods have little or no training cost, but solve a new weighted SAA problem at each test context. The learned primal-dual policies have a heavier offline training phase, but their online evaluation is essentially a forward pass, leading to orders of magnitudes faster inference.

\begin{table}[!ht]
  \centering
  \small
  \setlength{\tabcolsep}{4pt}
  \begin{tabular}{lcccccc}
    \toprule
    & \multicolumn{2}{c}{$N=25$}
    & \multicolumn{2}{c}{$N=50$}
    & \multicolumn{2}{c}{$N=100$} \\
    \cmidrule(lr){2-3}\cmidrule(lr){4-5}\cmidrule(lr){6-7}
    \textbf{Method}
    & \textbf{Train (s)} & \textbf{Test (ms)}
    & \textbf{Train (s)} & \textbf{Test (ms)}
    & \textbf{Train (s)} & \textbf{Test (ms)} \\
    \midrule
    SAA
      & 3.0 & --
      & 11 & --
      & 81 & -- \\
    kNN-10
      & 0.0 & 179
      & 0.0 & 372
      & 0.0 & 1587 \\
    Gaussian-SAA
      & 0.0 & 2198
      & 0.0 & 10696
      & 0.0 & 61361 \\
    RF-SAA
      & 0.4 & 1225
      & 0.5 & 5724
      & 0.8 & 36558 \\
    $\bar{\pi}^{50}$
      & 274 & 0.013
      & 453 & 0.024
      & 1405 & 0.077 \\
    $\bar{\pi}^{100}$
      & 552 & 0.016
      & 898 & 0.028
      & 2772 & 0.074 \\
    \bottomrule
  \end{tabular}
  \caption{Training time and average inference time, pooled over $25$ instances.}
  \label{tab:contextual_assortment_scaling}
\end{table}

\begin{result}
The primal-dual algorithm outperforms the contextual-SAA baseline in value while reducing inference time by several orders of magnitude, at the price of a heavier training phase.
\end{result}

\begin{remark}
Because the primal-dual algorithm learns a feature-based scoring rule over products, the resulting policy can be evaluated directly on catalogs containing products unseen during training (in which case $\calY(x)$ depends on $x$).
This is not the case for empirical prescriptions such as SAA or kNN-SAA, which are tied to the products observed in the source catalog.
In our transfer experiments (Appendix~\ref{sec:appendix_assortment_details}), this allows the primal-dual policy to remain competitive under catalog shift, and to become particularly advantageous when a large share of target products is new.
\end{remark}

\section{Conclusion and perspectives}\label{sec:conclusion_primal_dual}

    Our numerical experiments show that using policies based on combinatorial optimization layers is orders of magnitude faster at inference time and yields competitive results compared to other contextual stochastic optimization methods on contextual assortment. Furthermore, our approach is scalable to large-scale problems with context-dependent feasible sets $\calY(x)$. The experiments also demonstrate that our deep learning-compatible empirical cost minimization improves upon the literature benchmark for training combinatorial optimization layers based on imitation learning. This advantage is particularly evident on assortment problems where anticipative decisions perform poorly. 
    
    On the theoretical side, even though we have shown that the exact linear-parametric version of the alternating scheme may converge to a local minimum of the surrogate problem, the practical algorithm achieves strong empirical performance in practice. This is backed by two theoretical elements: first, the exact scheme is equivalent to a proximal point algorithm on a function that becomes convex in the small $\kappa$ regime; second, our extension of the regularization by perturbation to the distribution case enables us to obtain low-variance gradient estimates. We believe these tools might be relevant beyond our specific application. The convergence theorem does not cover the neural-network training and Monte Carlo/SGD approximations used in the experiments. 

    Looking ahead, it seems possible to extend our convergence analysis in three ways. 
    First, it relies on exact evaluations of the updates of the algorithm. 
    When we use a regularization by perturbation, we have a Monte Carlo estimate of the decomposition step and we solve the coordination step via stochastic gradient descent. It would be interesting to extend the convergence analysis to that setting.
    Second, a promising research direction would be to leverage the literature on generalization bounds to derive guarantees for the true expected cost minimization problem \parencite{aubin-frankowski_generalization_2024}. Third, some recent works study the case when Assumption~\ref{ass:oracle} is relaxed when learning with Fenchel-Young losses \parencite{vivierardisson2025learninglocalsearchmcmc,vivierardisson2026regularizedlargeneighborhoodsearch}, meaning when we only have a heuristic solver (local search or large neighborhood search) at our disposal. The precise implications for our primal-dual algorithm, and convergence analyses are left as future works.
    Finally, an interesting direction would be to leverage the results on Section~\ref{sec:new_structured_prediction} to propose a new empirical cost minimization algorithm that does not suffer from the non-convexity of the Jensen gap.
\paragraph{Acknowledgements.}

We deeply thank Pierre-Cyril Aubin Frankowski, Jérôme Malick and Yohann De Castro for their feedback on this manuscript, and advice on the mirror descent and alternating minimization literature. Besides, we are grateful to Solène Delannoy Pavy
for sharing her implementation of the contextual stochastic assortment problem.

\section*{Declarations}

\paragraph{Funding and conflict of interest/competing interests.}
This work has been partially supported by Renault Group through the Ph.D. studentship of Louis Bouvier. This support is gratefully acknowledged. The authors have no other competing interests to disclose.

\paragraph{LLM usage.} The authors used large language models (LLMs) to assist in auditing proofs, exploring alternative proof strategies, and identifying relevant literature. All mathematical content was manually edited and thoroughly checked by the authors.

\printbibliography

\newpage

\begin{appendices}

The appendices are organized as follows. Appendix~\ref{sec:regularization_on_distributions} gathers background on regularization over non-full-dimensional convex domains, used throughout the paper. Appendix~\ref{sec:appendix_B} contains the proofs of the tractability results of Section~\ref{subsec:alternating_min_tractability}, of the two new structured-prediction contributions of Section~\ref{sec:new_structured_prediction}, of the surrogate error bound (Theorem~\ref{thm:bound_risk}), and of the equivalence between tuning~$\kappa$ and tuning~$\varepsilon$ (Appendix~\ref{sec:kappa_eps_equivalence_appendix}). Appendix~\ref{sec:convergenceProof} proves the convergence result of Section~\ref{subsec:convergence} (Theorem~\ref{thm:convergence:speed}), building on a generic convergence result for the Bregman proximal point algorithm. Finally, Appendix~\ref{sec:appendix_numerical_exp} details the implementation and hyperparameters of the numerical experiments of Section~\ref{sec:computational_exp_primal_dual}.

\section{Regularization on non-full-dimensional spaces}\label{sec:appendix_A}\label{sec:regularization_on_distributions}

This appendix provides the background on regularization with non-full-dimensional
convex domains that is used throughout the paper. We begin with the definitions of
mirror maps and regularizers, then study the key properties of convex analysis on
non-full-dimensional sets, and close with the link between regularization on the
distribution space~$\Delta^\calY$ and on the moment space~$\calC$.

\subsection{Regularizers, Legendre-type functions, and mirror maps}
\label{sec:regularizer_legendre_mirror}
    We detail here the definitions of mirror maps and regularizers, and study their connections with Legendre-type functions.

    \paragraph{Mirror maps \textcite[Definition 2.1]{juditskyUnifyingMirrorDescent2023}.}
    Let $\Psi: \bbR^d \rightarrow \bbR \cup \{+ \infty\}$ be a function and $\calC \subset \bbR^d$ be a closed convex set. We say that $\Psi$ is a $\calC$\emph{-compatible mirror map} if
    \begin{enumerate}
        \item $\Psi$ is lower-semicontinuous and strictly convex,
        \item $\Psi$ is differentiable on $\inte(\dom(\Psi))$,
        \item the gradient of $\Psi$ takes all possible values, \ie{} $\nabla \Psi\big(\inte(\dom(\Psi))\big) = \bbR^d$.
        \item $\calC \subset \cl\big(\inte(\dom(\Psi))\big)$,
        \item $\inte(\dom(\Psi)) \cap \calC \neq \emptyset$.
    \end{enumerate}
    \begin{remark}
        Let $\calC \subset \bbR^d$ be closed convex set, and~$\Psi$ be a Legendre-type function such that $\calC \subset \cl\big(\inte(\dom(\Psi))\big)$, $\inte(\dom(\Psi)) \cap \calC \neq \emptyset$, and $\dom(\Psi^*) = \bbR^d$. Then $\Psi$ is a $\calC$-compatible mirror map.
    \end{remark}

    \begin{remark}
        Sometimes $\Psi$ is called a Bregman potential, and the term of mirror map is used for its gradient~$\nabla \Psi$.
    \end{remark}

    \paragraph{Regularizers \textcite[Definition 2.8]{juditskyUnifyingMirrorDescent2023}.}
    Let $\calC \subset \bbR^d$ be a closed convex set.
    A function $\Omega: \bbR^d \rightarrow \bbR \cup \{+ \infty\}$ is a $\calC$-\emph{pre-regularizer} if it is strictly convex, lower-semicontinuous, and if $\cl(\dom(\Omega)) = \calC$.
    If in addition $\dom(\Omega^*) = \bbR^d$, then $\Omega$ is said to be a $\calC$-\emph{regularizer}.
    \begin{remark}
        The previous definition is less restrictive than the one of mirror maps. In particular, regularizers are not necessarily differentiable, and their domains may be sub-dimensional.
    \end{remark}

    Let $\Omega \in \Gamma_0(\bbR^d)$ be a proper l.s.c. convex function mapping $\bbR^d$ to $(-\infty,+\infty]$.
    In this section, we consider the \emph{regularized prediction} problem defined as
    \begin{equation}\label{eq:regularized_pred}
        \hat{y}_\Omega(\theta) \in \argmax_{\mu \in \dom(\Omega)} \langle \theta| \mu \rangle - \Omega(\mu),
    \end{equation}
    and introduce some new geometric results related to it, which are useful for the study of our algorithm. 
    
    \subsection{Regularized prediction on non-full-dimensional spaces}\label{subsec:reg_subdim}
    To the best of our knowledge, most of the theory of Fenchel--Young losses has been developed either by introducing a regularization function~$\Omega$ with full-dimensional domain~$\calC$, where~$\Omega$ is Legendre-type, or by using a decomposition~$\Omega := \Psi + \bbI_{\calC}$, where~$\Psi$ is Legendre-type.
    In many applications in operations research, we consider polytopes $\calC$ that are not full-dimensional (the simplex is a case in point).
    When defining~$\Omega$ directly on the polytope (using a perturbation as in \textcite{berthetLearningDifferentiablePertubed2020} for instance), the decomposition~$\Omega := \Psi + \bbI_{\calC}$ is not given. 
    Therefore, when the polytope is not full-dimensional, we do not have access to a Legendre-type function.
    Instead, we have at our disposal a $\calC$-regularizer. 
    
    The following proposition extends classic results of Legendre-type functions to functions with non-full dimensional domains $\calC$, whose restrictions to the affine hull
    of $\calC$ are Legendre-type, while the next constructs a $ \Psi + \bbI_{\calC}$ decomposition.
    \begin{proposition}\label{prop:sub_dimensioal_domain_cvx_analysis}
        Let $\calC \subset \bbR^d$ be a non-empty convex compact set. 
        We consider a proper l.s.c. convex regularization function~$\Omega \in \Gamma_0(\bbR^d)$ with domain $\dom(\Omega)= \calC$. We assume that the restriction of $\Omega$ to $H = \aff(\calC)$, denoted as $\Omega_{|H}$, is Legendre-type (with respect to the metric of $H$, and not the one of $\bbR^d$).
        
        We denote by $V$ the direction of $H$ in $\bbR^d$, and we have the orthogonal sum ${\bbR^d = V \oplus V^\perp}$. We introduce $\Pi_V$, the orthogonal projection onto $V$ in $\bbR^d$. We have the following results:
        \begin{enumerate}
            \item \label{lem:Fenchel_conjugate_subdimension_diff} The Fenchel conjugate of $\Omega$, has full domain, \ie{} $\dom(\Omega^*)= \bbR^d$. Therefore, $\Omega$ is a $\calC$-regularizer (see Appendix~\ref{sec:regularizer_legendre_mirror}).
            Further, the function~$\Omega^*$ is differentiable over $\bbR^d$, and we have the property:
            \begin{equation}\label{eq:subdiff_diff}
                \nabla \Omega^*(\partial \Omega(y)) = y, \quad \forall y \in \relint(\calC).
            \end{equation}
            \item \label{lem:conjugate_constant_perp} Let $\theta \in \bbR^d$, decomposed as $\theta = \theta_V + \theta_{V^\perp}$, where $\theta_V = \Pi_V(\theta)$ and $\theta_{V^\perp} = \theta - \theta_V$, and $y_0 \in \calC$ be any point in $\calC$. The Fenchel conjugate of $\Omega$, denoted as $\Omega^*$, has an affine component over $V^\perp$:
            \begin{equation}\label{eq:affine_orthogonal}
                \Omega^*(\theta) = \Omega^*(\theta_V) + \langle \theta_{V^\perp}| y_0 \rangle.
            \end{equation}
            \item \label{lem:affine_subspace_subgrad} Let $y \in H$, the subdifferential of $\Omega$ at $y$ is given by:
            \begin{equation}\label{eq:subdiff_affine_subspace}
                \partial \Omega(y) = \partial(\Omega_{|H})(y) + V^\perp,
            \end{equation}
            where we have omitted the canonical injection from $H$ to $\bbR^d$ for notational simplicity. In particular, for $y \in \relint(\dom(\Omega))$, we have: \begin{equation}\label{eq:subdiff_relint_domain}
                \partial \Omega(y) = \{\nabla \Omega_{|H}(y)\} + V^\perp.
            \end{equation}
        \end{enumerate}
    \end{proposition}
    We illustrate Proposition~\ref{prop:sub_dimensioal_domain_cvx_analysis} in Figure~\ref{fig:sub_dimensional_space_analysis}, in the case $d=3$, $H$ is an affine hyperplane, and $V^\perp$ a straight line. Arrows represent the links between primal and dual variables, involving the subdifferential of $\Omega$ and the gradient of $\Omega^*$. The proof relies on classic convex duality results.
    \begin{figure}[ht]
        \centering
    \includegraphics[width=0.8\textwidth]{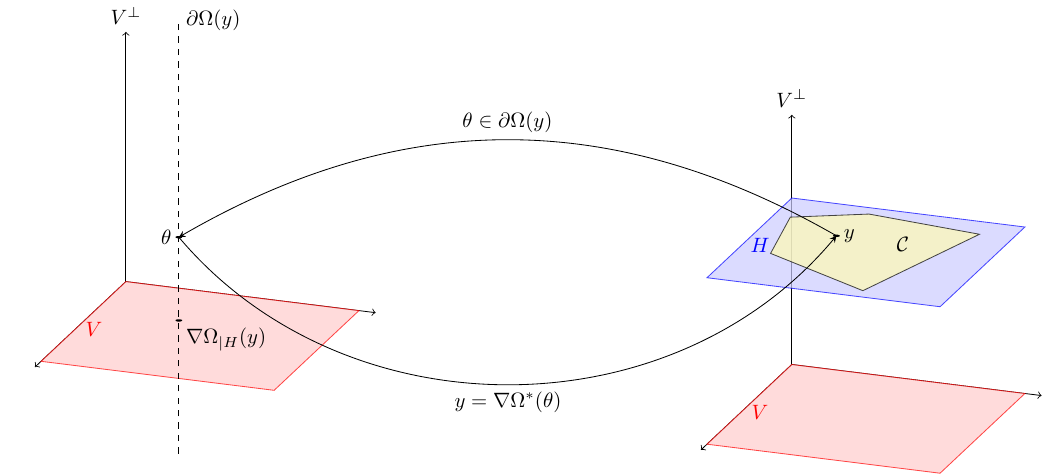}
        \caption{Primal-dual maps for non-full-dimensional $\dom(\Omega)$.}
        \label{fig:sub_dimensional_space_analysis}
    \end{figure}

    \begin{proof}[Proof of Proposition~\ref{prop:sub_dimensioal_domain_cvx_analysis}] Given the assumptions in the preamble of the proposition,
        \begin{enumerate}
            \item The function $\Omega$ belongs to the set of proper l.s.c. convex functions $\Gamma_0(\bbR^d)$, thus for any $\theta \in \bbR^d$, the supremum over the compact $\calC$ of $\langle \theta| \cdot \rangle - \Omega(\cdot)$ is finite and attained, thus $\dom(\Omega^*) = \bbR^d$.
        Recall that, as $\Omega$ is 
        in~$\Gamma_0(\bbR^d)$, using the computations of \textcite[Theorem 23.5]{Rockafellar+1970}, $\partial \Omega^*(\theta)=\argmax_{y} \langle \theta| y \rangle - \Omega(y)$. 
        As $\Omega$ is strictly convex, the argmax is reduced to a single point, and  $\Omega^*$ is differentiable over $\bbR^d$. 
        Therefore, we have for $(y ,\theta) \in (\bbR^d)^2$:
            \[\theta \in \partial \Omega(y) \iff y \in \partial \Omega^*(\theta) \iff y = \nabla \Omega^*(\theta).\]
        Therefore, for $y \in \relint(\calC)$, we have $\nabla \Omega^*(\partial \Omega(y)) = y$.
            \item Let $y_0 \in \calC$ and $\theta \in \bbR^d$ be decomposed as $\theta = \theta_V + \theta_{V^\perp}$.
        Note that for all $y \in \calC$ we have $\langle \theta_{V^\perp}| y \rangle = \langle \theta_{V^\perp}| y_0 \rangle$, since $\theta_{V^\perp}$ is orthogonal to the direction of the affine hull of $\calC$. 
        Thus, 
        \begin{align*}
            \Omega^*(\theta) & = \sup_{y \in \calC} \langle \theta_V + \theta_{V^\perp}| y \rangle - \Omega(y) = \langle \theta_{V^\perp}| y_0 \rangle + \sup_{y \in \calC} \langle \theta_V| y \rangle - \Omega(y),
        \end{align*}
        which yields the result.
        \item Let $(y, y') \in H^2$, $\theta \in \partial \Omega(y)$, by definition of the subgradients, we have
        \[\Omega(y') - \Omega(y) \geq \langle \theta| y'-y \rangle = \langle \Pi_V(\theta)| y'-y \rangle,\]
        since $y'-y$ belongs to $V$. Therefore we have shown that ${\Pi_V(\theta) \in \partial(\Omega_{|H})(y)}$.
        
        Conversely, let $y \in H$ and $\theta \in V$ be an element of $\partial(\Omega_{|H})(y)$, for $y' \in \bbR^d$ and $\tilde \theta \in V^\perp$,
        \[\Omega(y') \geq \Omega(y) + \langle \theta + \tilde \theta| y'-y \rangle,\]
        since either $y' \notin H$ and $\Omega(y') = + \infty$, or $y' \in H$ and $y'-y \in V$ therefore $\langle \tilde \theta| y'-y \rangle =0$. We have shown the first equality in Equation~\eqref{eq:subdiff_affine_subspace}.
        
        Equation~\eqref{eq:subdiff_relint_domain} comes from the fact that $\Omega_{|H}$ is Legendre-type, thus differentiable, and for $y \in \relint(\dom(\Omega))$, $\partial(\Omega_{|H})(y) = \{\nabla (\Omega_{|H})(y)\}$.
    \end{enumerate}    
    \end{proof}
    
    We now show that any $\Omega$ satisfying the assumption of Proposition~\ref{prop:sub_dimensioal_domain_cvx_analysis}, can be written as $\Omega= \Psi + \bbI_\calC$ for some Legendre-type function $\Psi$. 
    
    \begin{proposition}\label{prop:sum_sub_dim}
        Let $\calC \subset \bbR^d$ be a convex compact set, and  $\Omega$ be a proper l.s.c. convex regularization function in~$\Gamma_0(\bbR^d)$, with domain $\dom(\Omega)= \calC$. 
        We assume that the restriction of $\Omega$ to $H = \aff(\calC)$ is Legendre-type (with respect to the metric of $H$).
        Then, there exists a Legendre-type function $\Psi$, with $\calC \subset \cl\big(\inte(\dom(\Psi))\big)$, $\inte(\dom(\Psi)) \cap \calC \neq \emptyset$, and such that
        \[\Omega = \Psi + \bbI_{\calC}, \quad \text{and} \quad \dom(\Psi^*) = \bbR^d.\]
        Besides, let~$V$ be the direction of~$H$ in~$\bbR^d$, we have the direct sum~$\bbR^d = V \oplus V^\perp$. Given a vector $\theta \in \bbR^d$, there exists a vector $z \in V^\perp$ such that
        \[\nabla \Psi\big(\nabla \Omega^*(\theta) \big) = \theta + z.\]
    \end{proposition}
    This result can be seen as the converse, in a restricted setting (with more assumptions on $\Omega$), of a proposition by \textcite[Proposition 2.11]{juditskyUnifyingMirrorDescent2023}, where given a $\calC$-compatible mirror map $\Psi$, a $\calC$-regularizer~$\Omega$ is defined as $\Omega := \Psi + \bbI_{\calC}$. Indeed, we know with Proposition~\ref{prop:sub_dimensioal_domain_cvx_analysis} that $\Omega$ defined in the preamble of Proposition~\ref{prop:sum_sub_dim} is (in particular) a $\calC$-regularizer.    

    \begin{proof}[Proof of Proposition \ref{prop:sum_sub_dim}]
        Let $\Pi_V$ be the linear orthogonal projection onto $V$.
        W.l.o.g., we consider the case $H = V$, which means the affine hull of the domain of $\Omega$ is actually a vector subspace in $\bbR^d$. Extending to the affine case involves a translation. We define the following map:
        \begin{align*}
            \Psi: & \quad \bbR^d \rightarrow \bbR \\
            & \quad y \mapsto \Omega(\Pi_V(y)) + \frac{1}{2} ||y - \Pi_V(y)||^2_2, 
        \end{align*}
        where $\Pi_V(y)$ is seen as an element of $\bbR^d$ here. When it is the input of the restriction of~$\Omega$ to $V$, we see it as an element of $V$.
        Notice that by definition, $\Omega = \Psi + \bbI_\calC$. We are going to prove that $\Psi$ is a Legendre-type function. 
        We first show that $\Psi$ defined as such is essentially smooth by checking the three properties of the definition.
        \begin{enumerate}
            \item Since $\dom(\Omega) = \calC \subset V$ and $\|\cdot\|^2_2$ is defined over $\bbR^d$, the domain of $\Psi$ is ${\dom(\Psi) = \calC \oplus V^\perp}$, and $\inte(\dom(\Psi)) = \relint(\calC) \oplus V^\perp$, which is not empty. We therefore have
             \[\calC \subset \cl\big(\inte(\dom(\Psi))\big), \quad \text{and} \quad \inte(\dom(\Psi)) \cap \calC \neq \emptyset.\]
            
            \item By composition with linear projections and sum, $\Psi$ is differentiable over $\inte(\dom(\Psi))$. We denote by $J_{\Pi_V}$ the Jacobian of~$\Pi_V$, that can be seen as the canonical injection of~$V$ into~$\bbR^d$.
            Let now $(y, h) \in \inte(\dom(\Psi)) \times \bbR^d$ be two vectors such that $y+h \in \dom(\Psi)$,
            \begin{align*}
                \Psi(y+h) &= \Omega(\Pi_V(y+h)) + \frac{1}{2}||y+h - \Pi_V(y+h)||^2_2,\\
                & = \Omega(\Pi_V(y) + \Pi_V(h)) + \frac{1}{2}||y-\Pi_V(y) + h - \Pi_V(h)||^2_2,\\
                & = \Omega_{|V}(\Pi_V(y)) + \langle J_{\Pi_V} \nabla(\Omega_{|V})(\Pi_V(y))| \Pi_V(h) \rangle + o(||\Pi_V(h)||),\\
                & + \frac{1}{2}||y-\Pi_V(y)||^2_2 + \langle y-\Pi_V(y)| h-\Pi_V(h)\rangle + o(||h-\Pi_V(h)||),\\
                & = \Omega(\Pi_V(y)) + \langle J_{\Pi_V} \nabla(\Omega_{|V})(\Pi_V(y)) | \Pi_V(h) + h -  \Pi_V(h)\rangle + o(||h||),\\
                & + \frac{1}{2}||y-\Pi_V(y)||^2_2 + \langle y-\Pi_V(y)| h-\Pi_V(h) + \Pi_V(h)\rangle + o(||h||),\\
                & = \Psi(y) + \langle J_{\Pi_V} \nabla(\Omega_{|V})(\Pi_V(y)) +  y-\Pi_V(y) | h \rangle + o(||h||).
            \end{align*}
            In the computations above we use the linearity of $\Pi_V$, the fact that $\Omega$ and $\Omega_{|V}$ coincide over $V$, that $\Omega_{|V}$ is Legendre-type thus differentiable, and the orthogonal sum $\bbR^d = V \oplus V^\perp$.
            Therefore, we have shown that the gradient of $\Psi$ is given by:
            \[\nabla \Psi(y) = \underbrace{J_{\Pi_V}}_{\substack{\text{Canonical}\\\text{injection} \\ V\rightarrow \bbR^d}} \nabla \Omega_{|V}(\Pi_V(y)) + y - \Pi_V(y).\]
        \item The boundary of $\inte(\dom(\Psi))$ is \[\bdry\big(\inte(\dom(\Psi))\big) = \cl\big(\relint(\calC)\big) \backslash \relint(\calC) \oplus V^\perp.\]
        Indeed, $\inte(\dom(\Psi)) = \relint(\calC) \oplus V^\perp$ is isomorphic to $\relint(\calC) \times V^\perp$, $\bdry(V^\perp) = \emptyset$, $\cl(V^\perp) = V^\perp$, and for two sets $S_1$ and $S_2$
        \[\bdry(S_1\times S_2) = \big(\bdry(S_1)\times \cl(S_2) \big) \cup \big(\cl(S_1)\times \bdry(S_2)\big),\] 
        where the boundaries in the right-hand side above are computed with respect to the topology corresponding to each set. 
    
        Let now $\mu$ be in $ \bdry\big(\inte(\dom(\Psi))\big)$, and let $(\mu_i)_{i \in \bbN}$ be a sequence in~$\big(\inte(\dom(\Psi))\big)^{\bbN}$, such that 
        \[\lim _{i\to +\infty} \mu_i = \mu = \underbrace{\Pi_V(\mu)}_{\in \cl\big(\relint(\calC)\big) \backslash \relint(\calC)} + \underbrace{\mu - \Pi_V(\mu)}_{\in V^\perp}.\] 
        
        Since $\Pi_V$ is continuous, \[\lim_{i\to +\infty} \underbrace{\Pi_V(\mu_i)}_{\in \relint(\calC)} = \Pi_V(\mu), \quad \text{and} \quad \lim_{i\to +\infty} \underbrace{\mu_i - \Pi_V(\mu_i)}_{\in V^\perp} = \mu - \Pi_V(\mu).\]
        Now, using the fact that $\Omega_{|V}$ is Legendre-type, the expression of $\nabla \Psi$ above, and the reverse triangular inequality,
        \begin{align*}
            ||\nabla \Psi(\mu_i)|| &= || J_{\Pi_V} \nabla(\Omega_{|V})(\Pi_V(\mu_i)) + \mu_i - \Pi_V(\mu_i)||, \\
            & \geq \big| \underbrace{||J_{\Pi_V} \nabla(\Omega_{|V})(\Pi_V(\mu_i))||}_{\to + \infty} - \underbrace{||\mu_i - \Pi_V(\mu_i)||}_{\to ||\mu - \Pi_V(\mu)||} \big|.
        \end{align*}
        Therefore, we have shown that $\lim_{i \to + \infty} ||\nabla \Psi(\mu_i)|| = +\infty$.
    \end{enumerate}
        
    Let us finally show that $\Psi$ is strictly convex. We first remark that $\dom(\Psi)$ is convex since both $\calC$ and $V^\perp$ are. Let $(y_1, y_2)$ be in $ (\dom(\Psi))^2$, with $y_1 \neq y_2$, and let $t$ be in $(0,1)$. To ease notations, we denote $y_i^V=\Pi_V(y_i)$, and $y_i^\perp = y_i-\Pi_V(y_i)$,
    \begin{align*}
            \Psi(ty_1 + (1-t)y_2) &= \Omega(ty_1^V + (1-t)y_2^V) + \frac{1}{2} ||t y_1^\perp + (1-t)y_2^\perp||^2_2,\\
             & < t \Omega(y_1^V) + (1-t)\Omega(y_2^V) 
             + t \frac{1}{2} \|y_1^\perp\|^2 + (1-t)\frac{1}{2} \|y_2^\perp\|^2 ,\\
             & = t \Psi(y_1) + (1-t)\Psi(y_2).
        \end{align*}
      The first line is by linearity of the orthogonal projection onto $V$. 
      Further, as $y_1 \neq y_2$ we have 
    $y_1^V \neq y_2^V$ or $y_1^\perp \neq y_2^\perp$, thus the strict convexity of $\Omega$ over $\calC$ and of $\|\cdot\|^2_2$ yields the second line.
        We have therefore shown that $\Psi$ is a  Legendre-type function.
    
        \medskip
    
        We now consider its Fenchel conjugate~$\Psi^*$, and study its domain. Let $\theta \in \bbR^d$, decomposed as $\theta = \theta_V + \theta_{V^\perp}$, where $\theta_V = \Pi_V(\theta)$ and $\theta_{V^\perp} = \theta - \theta_V$,
        \begin{subequations}
            \begin{align}
                \Psi^*(\theta) &= \sup_{y \in \bbR^d} \{\langle \theta | y \rangle - \Psi(y)\},\\
                & =\sup_{y \in \bbR^d} \Big\{ \langle \theta_V | \Pi_V(y) \rangle - \Omega(\Pi_V(y)) + \langle \theta_{V^\perp} | y - \Pi_V(y) \rangle - \frac{1}{2} ||y - \Pi_V(y)||^2_2 \Big\},\\ 
                & =\sup_{\substack{y_V \in V, \\ y_{V^\perp} \in V^\perp}} \Big\{ \langle \theta_V | y_V \rangle - \Omega(y_V) + \langle \theta_{V^\perp} | y_{V^\perp} \rangle - \frac{1}{2} ||y_{V^\perp}||^2_2 \Big\},\\ 
                & = \Omega^*(\theta_V) + \frac{1}{2}||\theta_{V^\perp}||^2_2.
            \end{align}
        \end{subequations}
        Now, since $\Omega^*$ has full domain using Proposition~\ref{prop:sub_dimensioal_domain_cvx_analysis}, we have $\dom(\Psi^*) = \bbR^d$.
    
        \medskip
    
        We eventually show the property on the composition of the gradients of $\Psi$ and $\Omega^*$. First, we highlight that by Proposition~\ref{prop:sub_dimensioal_domain_cvx_analysis}, $\Omega^*$ is indeed differentiable over~$\bbR^d$. Its gradient corresponds to the regularized prediction defined by Equation~\eqref{eq:regularized_pred}, and belongs to the relative interior of the convex compact set~$\calC$.
        Let $\theta \in \bbR^d$ be a vector decomposed as $\theta = \theta_V + \theta_{V^\perp}$, where $\theta_V = \Pi_V(\theta)$ and $\theta_{V^\perp} = \theta - \theta_V$. Then we have 
        \[\nabla \Omega^*(\theta) = \nabla \Omega^*(\theta_V)  = \underbrace{J_{\Pi_V}}_{\substack{\text{Canonical}\\\text{injection} \\ V\rightarrow \bbR^d}} \nabla \Omega^*_{|V}(\theta_V).\]
        Therefore, applying the gradient of $\Psi$ leads to
        \begin{align*}
            \nabla \Psi \big(\nabla \Omega^*(\theta)\big) &= \nabla \Psi \big(\underbrace{J_{\Pi_V} \nabla \Omega^*_{|V}(\theta_V)}_{\in \relint(\calC)} + \underbrace{0}_{\in V^\perp}\big), \\
            & = J_{\Pi_V} \nabla \Omega_{|V} \big(\nabla \Omega^*_{|V}(\theta_V) \big) + 0,\\
            & = \theta_V = \theta \underbrace{- \theta_{V^\perp}}_{z \in V^\perp}.
        \end{align*}
        In the computations above, we use the expression of the gradient of~$\Psi$, and the fact that the restriction~$\Omega_{|V}$ of~$\Omega$ to~$V$ is a Legendre-type function with Fenchel conjugate~$\Omega^*_{|V}$.
        Therefore, we have shown that there exists a vector $z \in V^\perp$ such that $\nabla \Psi \big(\nabla \Omega^*(\theta)\big) = \theta + z$.
    \end{proof}
    
    \subsection{Regularizing in the distribution space}\label{subsec:distribution_space}
    
    Until now in this section, the Fenchel--Young loss has only been introduced in the context of the regularization (see Equation~\eqref{eq:regularized_pred}) of a linear optimization problem $\max_{\mu \in \calC} \langle \theta| \mu \rangle$. However, to deal with arbitrary minimization problems $\min_{y\in \calY} \; c(y)$ on a finite but combinatorial set $\calY$, it is convenient to consider regularization on distributions.
    The proposition below explores regularization on the distribution polytope. We recall that notations are introduced at the end of Section~\ref{sec:intro}.
    \begin{proposition}\label{prop:structuredPredictionWithGeneralizedNegentropy}
        Let $\Omega_{\Delta^{\calY}} \in \Gamma_0(\bbR^{\calY})$ be a proper l.s.c. convex function with domain $\Delta^{\calY}$. We drop the $\calY$ in the notation $\Omega_\Delta$ when $\calY$ is clear from context. We assume that the restriction of $\Omega_{\Delta}$ to $H_\Delta$, denoted as $\Omega_{\Delta|H_\Delta}$ is Legendre-type (with respect to the metric of $H_\Delta$, and not the one of $\bbR^{\calY}$). For $\mu \in \calC$, we define
        \begin{equation}\label{eq:omega_moment_from_distribution}
            \Omega_\calC(\mu) := \min\{\Omega_{\Delta}(q)\colon Yq = \mu \}.
        \end{equation}
        
        Let $\theta \in \bbR^d$ and $q \in \Delta^{\calY}$, we have the following properties:
        \begin{enumerate}
            \item $\langle s_\theta | q \rangle = \langle Y^\top \theta | q \rangle = \theta^\top Y q = \langle \theta | Y q\rangle = \theta^\top \mu_q$. \label{prop:objectivesEquality}
            \item $\Omega_{\calC}^*(\theta) = \Omega_{\Delta}^*(Y^\top \theta)$, therefore $\Omega_\calC^*$ has domain $\dom(\Omega_\calC^*) = \bbR^d$, it is differentiable over its domain and affine over $V^\perp$.\label{prop:FenchelDualsEquality}
            \item $\Omega_\Delta (q)\geq \Omega_\calC(\mu_q)$ and $\calL_{\Omega_\Delta}(s_\theta;q) \geq \calL_{\Omega_\calC}(\theta;\mu_q)$, both with equality if and only if 
            \[{q = \argmin\limits_{q' \in \Delta^{\calY} \colon Yq' = \mu_q}\Omega_\Delta(q')}.\]
            \item $\min\limits_{\theta'} \calL_{\Omega_\Delta}(s_{\theta'};q) \geq \min\limits_{\theta'} \calL_{\Omega_\calC}(\theta';\mu_q)$, with equality if and only if ${q = \argmin\limits_{q' \in \Delta^{\calY} \colon Yq' = \mu_q}\Omega_\Delta(q')}$.
            \item $\nabla_\theta \calL_{\Omega_\Delta}(s_\theta;q) = \nabla_\theta \calL_{\Omega_\calC}(\theta;\mu_q) = Y (\nabla\Omega_\Delta^*(s_\theta) - q) = \nabla\Omega_\calC^*(\theta) - \mu_q$.
        \end{enumerate}
    \end{proposition}
     This notion of regularization on distributions, and the way it induces a regularization on the moment space have already been introduced by \textcite[Section 7.1]{blondelLearningFenchelYoungLosses2020a}, but in the case of $\Omega_{\Delta}$ being a generalized negentropy \parencite{grunwaldGameTheoryMaximum2004}.
    
    \begin{proof}
        \begin{enumerate}
            \item Immediate.
            \item By definition, 
            \begin{align*}
            \Omega_\calC^*(\theta) = \max_{\mu}\Big( \theta^\top \mu + \max_{q \colon Yq = \mu}- \Omega(q)\Big) &= \max_{\mu, q\colon \mu = Yq}\underbrace{\theta^\top \mu}_{s_\theta^\top q} - \Omega(q), \\
            & = \max_{q} s_\theta^\top q - \Omega(q) = \Omega^*_\Delta(s_\theta).
            \end{align*} 
            Applying Proposition~\ref{prop:sub_dimensioal_domain_cvx_analysis}.\ref{lem:Fenchel_conjugate_subdimension_diff} to $\Omega_\Delta$, we get that ${\dom(\Omega_\Delta^*) = \bbR^{\calY}}$, and it is differentiable over $\bbR^{\calY}$. Composing with a linear map gives the domain and differentiability results. Applying Proposition~\ref{prop:sub_dimensioal_domain_cvx_analysis}.\ref{lem:conjugate_constant_perp} to $\Omega_\Delta^*$ with $V_\Delta^\perp = \vect(\mathbf{1})$, we have $\Omega^*_\Delta$ affine over $\vect(\mathbf{1})$. Besides, \[s_\theta \in \vect(\mathbf{1}) \iff \exists \alpha \in \bbR, \forall y \in \calY,  \langle \theta| y \rangle = \alpha \iff \theta \in V^\perp.\]
            \item An immediate consequence of the definition of $\Omega_\calC$ and the previous points.
            \item It follows from properties~\ref{prop:objectivesEquality} and~\ref{prop:FenchelDualsEquality} that the two minimization problems are equivalent up to a constant.
            \item Consequence of the equality of the losses up to a constant that does not depend on~$\theta$.
        \end{enumerate}
    \end{proof}
    
\section{Technical proofs for tractability and structured prediction}\label{sec:appendix_B}

\subsection{Proofs for Section~\ref{sec:structured_perturbation}}\label{subsec:appendix_structured_perturbation}

\propperturbation*
\begin{proof}[Proof of Proposition~\ref{prop:Perturbation}]
        Let $\varepsilon \in \bbR_{++}$,
        \begin{enumerate}
            \item Let $(\theta, Z) \in (\bbR^d)^2$, since $\calC$ is compact and non-empty, the maximum in the definition~\eqref{eq:perturbation_moment} of $F_{\varepsilon, \calC}$ is well-defined. 
            The expectation with respect to the normal distribution remains finite, and thus $\dom(F_{\varepsilon, \calC}) = \bbR^d$. 
            Further, the function $\theta \mapsto \max_{y \in \calC}(\theta + \varepsilon Z)^\top y$ is convex as a maximum of affine functions. 
            Finally, recall that finite convex functions are continuous. 
            \item The strict convexity of $F_{\varepsilon, \calC}$ over $V$ stems directly from the proof of \textcite[Proposition 2.2]{berthetLearningDifferentiablePertubed2020}. Let now $\theta = \theta_V + \theta_{V^\perp}$ be any vector in $\bbR^d$, and $y_0$ be any vector in $\calC$,
            \begin{align*}
                F_{\varepsilon, \calC}(\theta) = \bbE[\max_{y \in \calC}(\theta_V + \theta_{V^\perp} + \varepsilon \bfZ)^\top y] &= \bbE[\theta_{V^\perp}^\top y_0 + \max_{y \in \calC}(\theta_V + \varepsilon \bfZ)^\top y],\\
                & = \theta_{V^\perp}^\top y_0 + F_{\varepsilon, \calC}(\theta_V).
            \end{align*}
            \item As highlighted by \textcite[Proposition 3.1]{berthetLearningDifferentiablePertubed2020}, we can apply the technique of~\textcite[Lemma 1.5]{abernethyPerturbationTechniquesOnline2016} using the smoothness of the distribution of the noise variable $\bfZ$ to show the smoothness of $F_{\varepsilon, \calC}$. It involves a simple change of variable $u = \theta + \varepsilon Z$. The expression of the gradient comes from Danskin's lemma and swap of integration and differentiation.
            \item We first show that the domain of $F_{\varepsilon, \calC}^*$ is $\calC$. Let $y \in \bbR^d$, by definition \[F_{\varepsilon, \calC}^*(y) = \sup_{\theta \in \bbR^d}\{\theta^\top y - \bbE[\max_{y' \in \calC}(\theta + \varepsilon \bfZ)^\top y']\}.\] 
            \begin{itemize}
                \item If $y \in \bbR^d \backslash \calC$, we can separate it from $\calC$:
                there exists $\bar \theta \in \bbR^d$, and $\eta >0$ such that for all $y' \in \calC$, $\langle \bar \theta| y - y'\rangle \geq \eta$.
                Thus, for any $Z \in \bbR^d$, and $\lambda >0$, we have
                \begin{align*}
                    \langle \lambda \bar \theta + \varepsilon Z| y-y'\rangle &\geq \lambda \eta + \langle \varepsilon Z| y-y' \rangle
                     \geq \lambda \eta - ||\varepsilon Z||_2 ||y-y'||_2, \quad \forall y' \in \calC.
                \end{align*}
                Since $\calC$ is compact in $\bbR^d$, we can consider $D_{\calC, y} :=\sup_{y' \in \calC} ||y-y'||_2 < + \infty.$ 
                Minimizing with respect to $y' \in \calC$  yields
                \begin{align*}
                    \langle \lambda \bar \theta + \varepsilon Z| y \rangle - \max_{y' \in \calC} \langle \lambda \bar \theta + \varepsilon Z | y'\rangle & \geq \lambda \eta - ||\varepsilon Z|| D_{\calC, y}.
                \end{align*}
    
               Now, we recall that $\bfZ$ is centered with distribution $\nu$ (typically a multivariate standard normal distribution) with finite variance. 
               We thus denote $N_\nu:= \bbE_{\bfZ\sim \nu}[||\bfZ||_2]< + \infty$ and $\bbE[||\varepsilon \bfZ||_2] = |\varepsilon| N_\nu < + \infty$.
               Taking the expectation with respect to $\nu$ of the inequality above we get
                \begin{align*}
                    \langle \lambda \bar \theta| y \rangle - \bbE[\max_{y' \in \calC} \langle \lambda \bar \theta + \varepsilon \bfZ|y'\rangle] & \geq \lambda \eta - |\varepsilon|N_\nu D_{\calC, y}.
                \end{align*}
                Therefore, since $F_{\varepsilon, \calC}^*(y)\geq \langle \lambda \bar \theta| y \rangle - \bbE[\max_{y' \in \calC} \langle \lambda \bar \theta + \varepsilon \bfZ|y'\rangle]$ 
                and considering the limit $\lambda \to +\infty$ gives us $F_{\varepsilon, \calC}^*(y) = + \infty$.
                \item If $y \in \calC$, since $\bfZ$ is centered,
                \begin{align*}
                    F_{\varepsilon, \calC}^*(y) &= \sup_{\theta \in \bbR^d}\{\theta^\top y - \bbE[\max_{y' \in \calC}(\theta + \varepsilon \bfZ)^\top y']\},\\
                    &= \sup_{\theta \in \bbR^d}\{\bbE[(\theta+\varepsilon \bfZ)^\top y] - \bbE[\max_{y' \in \calC}(\theta + \varepsilon \bfZ)^\top y']\},\\
                    & = \sup_{\theta \in \bbR^d}\{\bbE[\underbrace{(\theta+\varepsilon \bfZ)^\top y - \max_{y' \in \calC}(\theta + \varepsilon \bfZ)^\top y'}_{\leq 0 \text{ since } y \in \calC}]\} < +\infty.\\
                \end{align*}
        \end{itemize}
            Therefore, we have shown that $\dom(F_{\varepsilon, \calC}^*) = \calC$. 
            Point $2.$ shows that~$(F_{\varepsilon, \calC})_{|V}$ is strictly convex over $V$. 
            Using the computations of point $3.$, we show that $(F_{\varepsilon, \calC})_{|V}$ is differentiable over~$V$. 
            It is thus a Legendre-type function with $\dom\big((F_{\varepsilon, \calC})_{|V}\big) = V$. 
            Indeed, point 3 of the essentially smooth definition holds vacuously. 
            To show that~$(F_{\varepsilon, \calC}^*)_{|H}$ is Legendre-type, we use the fact that it is the conjugate (up to a translation) of~$(F_{\varepsilon, \calC})_{|V}$ in $V$. 
            Then, Remark~\ref{rem:theorem_rockafellar_legendre} shows that~$(F_{\varepsilon, \calC}^*)_{|H}$ with domain $\calC$ is a Legendre-type function (with respect to the metric of $H$).
        \end{enumerate}
        This concludes the proof.
    \end{proof}

\thmsparseperturbation*
\begin{proof}[Proof of Theorem~\ref{prop:SparsePerturbation}]
    \begin{enumerate}
        \item 
        Let $s_1, s_2 \in \bbR^{\calY}$.
        As $ \Delta^{\calY}$ is a polytope, there exists $q^\sharp_1$ such that
        \begin{align*}
                    \max_{q_1 \in \Delta^{\calY}} \langle s_1 | q_1 \rangle - \max_{q_2 \in \Delta^{\calY}} \langle s_2 | q_2 \rangle
                    &= \langle s_1 | q_1^\sharp \rangle - \max_{q_2 \in \Delta^{\calY}} \langle s_2 | q_2 \rangle \leq \langle s_1 - s_2 | q_1^\sharp \rangle \\
                    &\leq \max_{q \in \Delta^{\calY}}\langle s_1 - s_2 | q \rangle \\
                    &\leq \max_{q \in \Delta^{\calY}}\| s_1 - s_2\|_2 \|q\|_2 = \| s_1 - s_2\|_2
        \end{align*}
        where the last equality comes from the definition of $\Delta^{\calY}$.
        Thus,
        \begin{align*}
            F_\Delta(s_1) - F_\Delta(s_2)  &=  \bbE_{\bfZ} \big[\max_{q \in \Delta^{\calY}} \langle s_1 + \varepsilon Y^\top \bfZ | q \rangle   - \max_{q \in \Delta^{\calY}} \langle s_2 + \varepsilon Y^\top \bfZ| q \rangle \big]
             \leq  ||s_1 - s_2||.
        \end{align*}
        By symmetry, we have shown that $F_\Delta$ is 1-Lipschitz continuous.
        \item The proof relies on the following technical lemma provided just below.

        \begin{lem}\label{lem:DistinctArgmax_nonzeromeasure}
    Let $\calY \subset \bbR^d$ be a finite set satisfying Assumption~\ref{rem:exposed_vertices}.
    For any vector $s \in \bbR^{\calY}$, we denote by $s(y) \in \bbR$ the component of $s$ indexed by $y \in \calY$. 
    We also consider $\varepsilon >0$ a positive real number, and $\bfZ$, a random variable with standard multivariate normal distribution over $\bbR^d$. 
    
    {
            For \(i\in\{1,2\}\), let \(f_i(y;Z):=s_i(y)+\varepsilon Z^\top y\). Then, for any distinct \(s_1,s_2\in V_\Delta\), there is not almost surely a point \(y^\star\in\calY\) that maximizes both \(f_i(\cdot;\bfZ)\) over \(\calY\). Then, for any distinct \(s_1,s_2\in V_\Delta\), the event \(\argmax_{y\in\calY}f_1(y;\bfZ)\cap\argmax_{y\in\calY}f_2(y;\bfZ)=\emptyset\) has positive probability.}
    
    \end{lem}
    
    Let $t \in (0, 1)$, the lemma above leads to
    \begin{align*}
        \bbP_{\bfZ} \Bigg[\max_{y \in \calY} \bigg(t f_1(y; \bfZ) + (1-t) f_2(y;\bfZ) \bigg)  < & \max_{y \in \calY} t f_1(y; \bfZ), \\ & + \max_{y \in \calY} (1-t) f_2(y;\bfZ) \Bigg] > 0.
    \end{align*}
    Since $F_\Delta$ is convex, this strict inequality with positive probability leads to the strict convexity of $F_\Delta$ over $V_\Delta$. 
    Last, using the decomposition $\bbR^{\calY} = V_\Delta + \vect(\mathbf{1})$, we get the affine property over $V_\Delta^\perp = \vect(\mathbf{1})$ with the same arguments as for Proposition~\ref{prop:Perturbation} Point 2.
        \item We first show that  the $\argmax$ in the definition of~$F_{\varepsilon, \Delta}$ is reduced to a singleton almost surely.
        Indeed, assume that for a pair of vectors~$(y,y')\in\calY^2$,
        \[P_{\bfZ}\big[ s(y) + \varepsilon \bfZ^\top y = s(y') + \varepsilon \bfZ^\top y'\big] > 0,\]
        then, as $\varepsilon\bfZ$ is a non-degenerate Gaussian, it implies that~$\langle \cdot| y-y'\rangle$ is constant (equal to $s(y')-s(y)$) on a ball of positive radius, and thus that $y=y'$.
        
       Note that the uniqueness of the $\argmax$ in $\calY$ implies the uniqueness of the corresponding $\argmax$ in the distribution space~$\Delta^{\calY}$ (see Equation~\eqref{eq:perturbation_distribution}). 
       Now, as the argmax is almost surely unique, using Danskin's theorem 
       \parencite[Theorem~10.31]{rockafellar1998variational}, and swapping integration with respect to the density of~$\bfZ$ and differentiation with respect to~$s$, we get:
            \[\nabla_s\bbE[\max_{q \in \Delta^{\calY}} (s +\varepsilon Y^\top \bfZ)^\top q] = \bbE[\argmax_{q \in \Delta^{\calY}} (s +\varepsilon Y^\top \bfZ)^\top q].\]
        \item The domain property is proved in the same way as for Proposition~\ref{prop:Perturbation} point 4, the rest also yields similarly.
    \end{enumerate}
    \end{proof}

    \begin{proof}[Proof of Lemma \ref{lem:DistinctArgmax_nonzeromeasure}] 
    Let $(s_1, s_2) \in (V_\Delta)^2, \, s_1 \neq s_2$ be two vectors. We are going to show that 
    \[\argmax_{y \in \calY} \{s_1(y) + \varepsilon Z^\top y \} \,\cap \, \argmax_{y \in \calY} \{s_2(y) + \varepsilon Z^\top y \} = \emptyset,\]
    
    \noindent for $Z$ almost everywhere with respect to the Lebesgue measure in an open ball in $\bbR^d$. Since $\bfZ$ follows a non-degenerate normal distribution, {this open ball has positive Gaussian measure, and} the result yields.
    
    Suppose first that 
    \(\argmax_{y \in \calY} s_1(y) \cap \argmax_{y \in \calY} s_2(y) = \emptyset.\) 
    Then, a ball centered on $0_{\bbR^d}$ with sufficiently small radius gives the result.
    
    Let now $y^*$ be a common maximizer of $s_1$ and $s_2$. Since $s_1$ and $s_2$ are elements of ${V_\Delta = \vect(\mathbf{1})^\perp}$, and they are distinct, they are not equal up to a constant. We can thus fix a $\bar{y} \in \calY$ such that \(s_1(\bar{y}) - s_1(y^*) \neq s_2(\bar{y}) - s_2(y^*).\)
    
    In particular, it implies that $\bar y$ does not belong to the intersection of the $\argmax$, 
    \(\bar y \notin \argmax_{y \in \calY} s_1(y) \cap \argmax_{y \in \calY} s_2(y).\)
    Let $\bar Z$ be in the relative interior of the normal cone of $\conv(\calY)$ at $\bar y$, such that the function g defined as 
    \(g: y \mapsto \varepsilon \bar Z^\top y,\) 
    is injective over $\calY$. It is possible since the normal cone is full dimensional, and the union of the hyperplanes where two dot products are equal is not full dimensional. Then, for $\lambda \in \bbR_+$ sufficiently large,
    \(\argmax_{y \in \calY} s_1(y) + \varepsilon \lambda \bar Z^\top y = \argmax_{y \in \calY} s_2(y) + \varepsilon \lambda \bar Z^\top y = \{\bar y \}.\)
    
    For $i \in \{1,2\}$, let us define $F_i$ as
    \(F_i : \lambda \in \bbR \mapsto \max_{y \in \calY}\big(s_i(y) + \varepsilon \lambda \bar Z^\top y \big) - s_i(y^*).\)
    Note that if there exists $\lambda^* \in \bbR$ such that $\argmax_{y \in \calY} \big(s_i(y) + \lambda^* \varepsilon \bar Z^\top y \big)$ for $i \in \{1,2\}$ are disjoint singletons, then considering a small enough open ball around the vector $\lambda^* \bar Z$ gives the result.  We now show that such a $\lambda^*$ exists.
    
    Since $\calY$ is finite, for $i\in\{1,2\}$, $F_i$ is a maximum of a finite number of affine functions in $\lambda$, it is therefore a piecewise affine function.
    Since $g$ is injective, for $i \in \{1,2\}$, there exists a collection of $k_i +1$ real numbers, $k_i \in \bbN, \, k_i >1$, denoted as  $0=\lambda_0^i< \ldots < \lambda_{k_i}^i$ and a unique collection $y^* = y_1^i, \ldots, y_{k_i}^i = \bar{y}$ of two-by-two distinct vectors such that
    \[F_i(\lambda) = s_i(y_j^i) + \varepsilon \lambda \bar Z^\top y_j^i - s_i(y^*), \quad \forall \lambda \in [\lambda_{j-1}^i, \lambda_j^i].\]
    
    Furthermore, the fact that $g$ is injective implies that $y_j^i$ is the unique maximizer over $\calY$ of $y \mapsto s_i(y) + \varepsilon \lambda \bar Z^\top y$ for $\lambda \in ]\lambda_{j-1}^i,\lambda_j^i[$.
    If the collections for $i \in \{1,2\}$ are not identical, the proof is finished. By contradiction, we assume that the two collections are identical and denote by ${0=\lambda_0< \ldots < \lambda_k}$ and  ${y^* = y_1, \ldots, y_k = \bar{y}}$ the common collections.
    
     Let $\tilde j$ be the smallest $j' \in [k]$ such that $s_1(y_{j'}) - s_1(y^*) \neq s_2(y_{j'}) - s_2(y^*)$. It exists since the inequality holds for $\bar{y}$. Without loss of generality, we can assume from now on that $s_1(y_{\tilde j}) - s_1(y^*) > s_2(y_{\tilde j}) - s_2(y^*)$. For $i \in \{1,2\}$, we define $\tilde{\lambda}_i$ as: 
     
     \[\tilde{\lambda}_i := \min \{\lambda \;|\; y_{\tilde j} \in \argmax_{y \in \calY} \big( s_i(y)+\varepsilon \lambda \bar Z^\top y \big)\}.\]
    
    We eventually get a contradiction by proving $\tilde \lambda_1 < \tilde \lambda_2$ with the following inequality.
     \begin{subequations}
         \begin{align}
             & s_2(y_{\tilde j}) - s_2(y^*) + \tilde \lambda_1 g(y_{\tilde j}), \\
             &< s_1(y_{\tilde j}) - s_1(y^*) + \tilde\lambda_1  g(y_{\tilde j}), && (\text{hypothesis right above}), \\
             &= s_1(y_{\tilde j-1}) - s_1(y^*) + \tilde\lambda_1 g(y_{\tilde j-1}), && \text{(Both $y_{\tilde j}$ and $y_{\tilde j-1}$ are optimal at the junction)}, \\
             &= s_2(y_{\tilde j-1})- s_2(y^*)  + \tilde\lambda_1 g(y_{\tilde j-1}), && \text{(by definition of $\tilde j$)}.
         \end{align}
     \end{subequations}
    Hence, for $\lambda = \tilde \lambda_1 + \eta$ with $\eta> 0$ sufficiently small, $y_{\tilde j}$ is the unique $\argmax$ of $s_1(y) + \lambda g(y)$
    and $y_{\tilde j-1} \neq y_{\tilde j}$ is the unique $\argmax$ of $s_2(y) + \lambda  g(y)$.
    \end{proof}

\propstrongconvexitysparseperturbation*
\begin{proof}[Proof of Proposition~\ref{prop:strongConvexitySparsePerturbation} (strong convexity of $\Omega_{\varepsilon,\Delta}$)]
    Let $(s_a, s_b) \in (\bbR^{\calY})^2$. By Theorem~\ref{prop:SparsePerturbation}, the gradient of $F_{\varepsilon, \Delta}$ is given by:
    \begin{align*}
        \nabla_s F_{\varepsilon, \Delta}(s) &= \bbE\left[\argmax_{q \in \Delta^{\calY}} (s + \varepsilon \bfZ)^\top q\right] 
        = \bbE\left[ \sum_{y \in \calY} e_y \oneindicator\left\{ y = \argmax_{y' \in \calY} (s(y') + \varepsilon \langle\bfZ|{y'}\rangle) \right\} \right],
    \end{align*}
    where $e_y$ denotes the standard basis vector in $\bbR^{\calY}$ corresponding to $y$. Note that the $\argmax$ is reduced to a singleton almost surely, as detailed in the proof of Theorem~\ref{prop:SparsePerturbation}. 
    Thus, we can bound the distance between the two gradients:
    \begin{align*}
        \|\nabla F_{\varepsilon, \Delta}(s_a) - \nabla F_{\varepsilon, \Delta}(s_b)\|_2
        &= \Big\| \bbE\Big[ \argmax_{y \in \calY} (s_a(y) + \varepsilon \langle\bfZ|y\rangle)  -  \argmax_{y \in \calY} (s_b(y) + \varepsilon \langle\bfZ|y\rangle) \Big] \Big\|_2 \\
        &\leq \bbE\Big[ \Big\| \argmax_{y \in \calY} (s_a(y) + \varepsilon \langle\bfZ|y\rangle)  - \argmax_{y \in \calY} (s_b(y) + \varepsilon \langle\bfZ|y\rangle) \Big\|_2 \Big] \\
        &\leq \sum_{(y_1, y_2) \in \calY^2} \|e_{y_1} - e_{y_2}\|_2 P_{y_1, y_2},
    \end{align*}

    where
    \begin{align*}
        P_{y_1, y_2} := \bbP\Big( &y_1 = \argmax_{y \in \calY} (s_a(y) + \varepsilon \langle\bfZ|y\rangle), y_2 = \argmax_{y \in \calY} (s_b(y) + \varepsilon \langle\bfZ|y\rangle) \Big).
    \end{align*}
    Since $e_{y_1}$ and $e_{y_2}$ are vertices of the standard simplex, for any $y_1 \neq y_2$, we have $\|e_{y_1} - e_{y_2}\|_2 = \sqrt{2}$.

   If the $\argmax$ shifts from $y_1$ under $s_a$ to $y_2$ under $s_b$, then the relative order of the scores of $y_1$ and $y_2$ must have swapped. In particular:
    \begin{align*}
        s_a(y_1) + \varepsilon \langle\bfZ|{y_1}\rangle &\geq s_a(y_2) + \varepsilon \langle\bfZ|{y_2}\rangle \implies \varepsilon(\langle\bfZ|{y_1}\rangle - \langle\bfZ|{y_2}\rangle) \geq s_a(y_2) - s_a(y_1) \\
        s_b(y_1) + \varepsilon \langle\bfZ|{y_1}\rangle &\leq s_b(y_2) + \varepsilon \langle\bfZ|{y_2}\rangle \implies \varepsilon(\langle\bfZ|{y_1}\rangle - \langle\bfZ|{y_2}\rangle) \leq s_b(y_2) - s_b(y_1)
    \end{align*}
    Therefore, the probability $P_{y_1, y_2}$ is upper-bounded by the probability that the random variable $\varepsilon(\langle\bfZ|{y_1}\rangle - \langle\bfZ|{y_2}\rangle)$ falls between $s_a(y_2) - s_a(y_1)$ and $s_b(y_2) - s_b(y_1)$. Up to reordering $a$ and $b$ to form a valid interval, we obtain:
    \begin{align*}
        P_{y_1, y_2} &\leq \bbP\left( \varepsilon(\langle\bfZ|{y_1}\rangle - \langle\bfZ|{y_2}\rangle) \text{ is strictly between } s_a(y_2) - s_a(y_1) \text{ and } s_b(y_2) - s_b(y_1) \right) \\
        &= \bbP\left( \frac{\langle\bfZ|{y_1}\rangle - \langle\bfZ|{y_2}\rangle}{\sqrt{2}} \in I_{y_1, y_2} \right),
    \end{align*}
    where $I_{y_1, y_2}$ is an interval of length:
    \begin{align*}
        \frac{1}{\varepsilon \sqrt{2}} |(s_b(y_2) - s_b(y_1)) - (s_a(y_2) - s_a(y_1))| &\leq \frac{2}{\varepsilon \sqrt{2}} \|s_a - s_b\|_{\infty} \\
        &= \frac{\sqrt{2}}{\varepsilon} \|s_a - s_b\|_{\infty} \leq \frac{\sqrt{2}}{\varepsilon} \|s_a - s_b\|_2.
    \end{align*}

    Since $\bfZ \sim \calN(0, I_d)$ is a standard Gaussian in $\bbR^d$ and $\langle\bfZ|y\rangle := y^\top\bfZ$, we have
    \[
        \frac{\langle\bfZ|{y_1}\rangle - \langle\bfZ|{y_2}\rangle}{\sqrt{2}} \sim \mathcal{N}\!\left(0, \frac{||y_1 - y_2||_2^2}{2}\right).
    \]
    Given a Gaussian $X \sim \calN(0,\sigma^2)$, we have $\bbP(X \in [\alpha, \beta]) \leq \frac{|\beta - \alpha|}{\sqrt{2\pi}\,\sigma}$ since the density is uniformly bounded by $1/(\sqrt{2\pi}\,\sigma)$.

    Applying this to our probability bound yields for $y_1 \neq y_2$:
    \begin{equation*}
        P_{y_1, y_2} \leq \frac{1}{\sqrt{2\pi}} \frac{\sqrt{2}}{||y_1-y_2||_2}\frac{\sqrt{2}}{\varepsilon} \|s_a - s_b\|_2.
    \end{equation*}

    Let $m_\calY := \min_{(y, y') \in \calY^2, y \neq y'}||y-y'||_2$ be the minimum distance between two distinct points in $\calY$.
    Substituting this back into the sum, which contains at most $|\calY|^2$ non-zero terms (since $P_{y,y}$ doesn't change relative distance but $\|e_y-e_y\|_2=0$, we just upper bound the sum generously by considering all $|\calY|^2$ pairs):
    \begin{align*}
        \|\nabla_s F_{\varepsilon, \Delta}(s_a) - \nabla_s F_{\varepsilon, \Delta}(s_b)\|_2 &\leq \sum_{(y_1, y_2) \in \calY^2} \sqrt{2} \left( \frac{1}{\sqrt{2\pi}} \frac{\sqrt{2}}{m_\calY} \frac{\sqrt{2}}{\varepsilon} \|s_a - s_b\|_2 \right) \\
        &\leq \frac{2|\calY|^2}{\varepsilon m_\calY \sqrt{\pi}} \|s_a - s_b\|_2 
    \end{align*}
    which gives the desired Lipschitz constant $L$.

    Finally, knowing that $F_{\varepsilon, \Delta}$ is a convex function with an $L$-Lipschitz continuous gradient, standard duality results in convex analysis dictate that its Fenchel conjugate $\Omega_{\varepsilon, \Delta}$ is $(1/L)$-strongly convex. This strong convexity is on $H_\Delta$ and with respect to $||\cdot||_2$. Therefore, $\Omega_{\varepsilon, \Delta}$ is strongly convex on $H_\Delta$ with respect to $||\cdot||_2$ with parameter:
    \(\mu = \frac{1}{L} = \frac{\varepsilon m_\calY \sqrt{\pi}}{2|\calY|^{2}}. \)

\end{proof}

    \propstructuredpredictionsparseperturbation*
\begin{proof}[Proof of Proposition~\ref{prop:structuredPredictionSparsePerturbation} (structured prediction regularizer with sparse perturbation)]
        By definition, $F_{\varepsilon, \calC}(\theta) = F_{\varepsilon, \Delta}(Y^\top \theta)$. Besides, since $\Omega_{\varepsilon, \calC}$ and $\Omega_{\varepsilon, \Delta}$ are proper convex lower-semicontinuous, they are bi-conjugate by Fenchel-Moreau theorem. Applying the computations of~\textcite[Corollary 15.28]{bauschkeConvexAnalysisMonotone2017} with $g = F_{\varepsilon, \Delta}$ and ${L : x \mapsto Y^\top x}$ leads to the result.
    \end{proof}

    \subsection{Proofs for Section~\ref{subsec:alternating_min_tractability}}\label{subsec:tractable_proofs}
    At this point, we now have all the elements needed to prove the two main propositions that underpin the tractability of updates~\eqref{eq:decomposition_w_dist}-\eqref{eq:coordination_w_dist}.
    
    \propcomputationsprimaldualdist*
    \begin{proof}[Proof of Proposition~\ref{prop:computations_primal_dual_dist}] \label{proof:computations_primal_dual_dist}
        To derive Equation~\eqref{eq:decomposition_w_dist} remark the following. First, $\calS_{\Omega_\Delta,N}$ is defined as the sum of~$S_{\Omega_\Delta}$, and since $\bar w^{(t)}$ is fixed, we obtain $N$ independent problems. Now, using the expression of~$S_{\Omega_\Delta}$, we see that~$q_i^{(t+1)}$ belongs to the minimizers of
        \[\Omega_{\Delta^{\calY(x_i)}}(\cdot) - \langle Y(x_i)^\top\varphi_{\bar w^{(t)}}(x_i) - \frac{1}{\kappa}\gamma_i | \cdot \rangle.\]
        Using Fenchel duality, we recognize~$\nabla \Omega_{\Delta^{\calY(x_i)}}^*\big(Y(x_i)^\top\varphi_{\bar w^{(t)}}(x_i) - \frac{1}{\kappa}\gamma_i\big)$. 
    
        For the dual update in Equation~\eqref{eq:coordination_w_dist}, using the expression of~$S_{\Omega_\Delta}$, and omitting the term that does not depend on~$w$, we can first recast Equation~\eqref{eq:coordination_product_w} as 
        \[\bar w^{(t+1)} \in \argmin_{w \in \calW} \frac{1}{N} \sum_{i=1}^N \calL_{\Omega_{\Delta(x_i)}}\big(Y(x_i)^\top \varphi_w(x_i); q_i^{(t+1)}\big). \]
        Now, we leverage the computations of Section~\ref{subsec:distribution_space}. In particular, based on the regularization function~$\Omega_{\Delta^{\calY(x)}}$ on distributions in~$\Delta^{\calY(x)}$, we define a regularization function~$\Omega_{\calC(x)}$ on the moment space~${\calC(x) = \conv(\calY(x))}$ as in Equation~\eqref{eq:omega_moment_from_distribution}. We then use Proposition~\ref{prop:structuredPredictionWithGeneralizedNegentropy}, more precisely Points~\ref{prop:objectivesEquality}-\ref{prop:FenchelDualsEquality}, to get Equation~\eqref{eq:coordination_w_dist}. 
    \end{proof}
    
    \propprimaldualperturbation*
    \begin{proof}[Proof of Proposition~\ref{prop:primal_dual_perturbation}]\label{proof:primal_dual_perturbation}
        For the primal update in Equation~\eqref{eq:primalUpdatePerturbation},
        we can reformulate Equation~\eqref{eq:decomposition_w_dist} in this perturbation setting
    \begin{subequations}
        \begin{align*}
            \mu_i^{(t+1)} &= Y(x_i)q_i^{(t+1)} = Y(x_i)\nabla F_{\varepsilon, \Delta(x_i)}\big(Y(x_i)^\top\varphi_{\bar w^{(t)}}(x_i) - \frac{1}{\kappa}\gamma_i \big),\\
            &= Y(x_i)\bbE_\bfZ\big[\argmax_{q_i \in \Delta^{\calY(x_i)}} \langle Y(x_i)^\top\varphi_{\bar w^{(t)}}(x_i) - \frac{1}{\kappa}\gamma_i + \varepsilon Y(x_i)^\top \bfZ| q_i \rangle \big],\\
            &= \bbE_\bfZ\big[Y(x_i)\argmin_{q_i \in \Delta^{\calY(x_i)}}\langle \big(\frac{1}{\kappa}c(x_i,y,\xi_i) - (\varphi_{\bar w^{(t)}}(x_i) + \varepsilon \bfZ)^\top y\big)_{y \in \calY(x_i)}| q_i \rangle \big],\\
            &= \bbE_\bfZ\big[\argmin_{y_i \in \calY(x_i)} c(x_i,y_i,\xi_i) - \kappa (\varphi_{\bar w^{(t)}}(x_i) + \varepsilon \bfZ)^\top y_i \big].
        \end{align*}
        \label{eq:decomposition_dist_pert}
    \end{subequations}
        In the computations above, we use Theorem~\ref{prop:SparsePerturbation} for the expression of the gradient of the function~$F_{\varepsilon, \Delta(x_i)}$, and the fact that the minimum of the linear optimization problem in~$q_i$ is attained (almost surely) at a vertex of the simplex~$\Delta^{\calY(x_i)}$, corresponding to a Dirac on a point~$y_i \in \calY(x_i)$. We see that the two constants~$\kappa$ and $\varepsilon$ play similar roles in this setting; a formal equivalence between tuning one parameter versus the other is established in Appendix~\ref{sec:kappa_eps_equivalence_appendix}. Up to a re-normalization of the statistical model~$\varphi_w$, we keep the hyper-parameter~$\varepsilon$ to tune the regularization scale.
    
        For the dual update in Equation~\eqref{eq:coordination_w_dist}, we simply use the expression of the Fenchel--Young loss in the perturbation setting and omit the terms that do not depend on~$w$.
    \end{proof}

\subsection{Proofs for Section~\ref{sec:loss_parameter_space} and consequences in Section~\ref{subsec:convergence}}\label{subsec:appendix_loss_parameter_space}

\propbaromegaproperties*
\begin{proof}[Proof of Proposition~\ref{prop:barOmegaProperties}]
Recall that \(H_i:=\aff(\calP_i)\), that \(V_i\) is its direction
space, and that \(\bar\calW=\sum_{i=1}^N A_iV_i\). We regard
\(\bar\calW\) as a Euclidean space with the inner product inherited
from \(\bbR^d\). Set
\[
\bar F(\bar w)
:=
\frac1N\sum_{i=1}^N
\Omega_{\calP_i}^*
\big(A_i^\top J_{\Pi_{\bar\calW}}\bar w\big),
\qquad
\Phi(\pi_\otimes)
:=
\frac1N\sum_{i=1}^N\Omega_{\calP_i}(\pi_i),
\]
and define \(B\pi_\otimes:=N^{-1}\sum_i\Pi_{\bar\calW}A_i\pi_i\).

For each \(i\), choose \(a_i\in H_i\) and define, for \(v\in V_i\), 
\(
\widetilde\Omega_i(v)
:=
\Omega_{\calP_i}(a_i+v).
\)
On \(V_i\),
\(\widetilde\Omega_i\) is closed, proper, and of Legendre type. Since its
domain \(\calP_i-a_i\) is compact, its conjugate is finite on all of
\(V_i\). 
Hence, by \textcite[Theorem~26.5]{Rockafellar+1970}, \(\widetilde\Omega_i^*\) is of Legendre type and, in particular, strictly convex on \(V_i\).

Moreover, for every \(s\in\bbR^{d_i}\),
\(
\Omega_{\calP_i}^*(s) = \langle s,a_i\rangle + \widetilde\Omega_i^*(\Pi_{V_i}s).
\)
Thus \(\Omega_{\calP_i}^*\) is finite and continuously differentiable on
\(\bbR^{d_i}\), affine along \(V_i^\perp\), and strictly convex along every
segment whose direction does not belong to \(V_i^\perp\).

Let \(\bar w_a\neq\bar w_b\) and set
\(u:=J_{\Pi_{\bar\calW}}(\bar w_a-\bar w_b)\).
If \(A_i^\top u\in V_i^\perp\) for every \(i\), then
\(\langle u,A_iv_i\rangle=0\) for every \(v_i\in V_i\). Hence
\(u\perp\sum_iA_iV_i=\bar\calW\). Since \(u\in\bar\calW\), this would
give \(u=0\), a contradiction. Thus at least one summand of \(\bar F\)
is strictly convex along \([\bar w_a,\bar w_b]\), while all the others
are convex. Therefore \(\bar F\) is strictly convex. Since it is
differentiable on its whole domain \(\bar\calW\), it is of Legendre
type. Furthermore, \(\bar F\) is closed, and hence
\(
\bar F=\bar F^{**}=\Omega_{\bar\calM}^*
\)
by \textcite[Theorem~12.2]{Rockafellar+1970}.

Now let \(p:=B\Phi\) be the image function
\[
p(\bar\nu)
:=
\inf\{\Phi(\pi_\otimes)\colon B\pi_\otimes=\bar\nu\}.
\]
Since \(\dom\Phi=\prod_i\calP_i\) is compact, \(p\) is proper and
lower semicontinuous, and its infimum is attained whenever it is
finite; it is convex by
\textcite[Theorem~5.7]{Rockafellar+1970}. Moreover,
\textcite[Theorem~16.3]{Rockafellar+1970} gives
\(p^*=\Phi^*\circ B^*\). Since
\(B^*\bar w=(N^{-1}A_i^\top J_{\Pi_{\bar\calW}}\bar w)_i\), the
scaling rule for conjugates yields
\[
p^*(\bar w)
=
\frac1N\sum_{i=1}^N
\Omega_{\calP_i}^*
\big(A_i^\top J_{\Pi_{\bar\calW}}\bar w\big)
=
\bar F(\bar w).
\]
Therefore, again by
\textcite[Theorem~12.2]{Rockafellar+1970},
\(\Omega_{\bar\calM}=\bar F^*=p^{**}=p\), which proves
\eqref{eq:expressionOfBarOmega}. It also gives
\[
\dom\Omega_{\bar\calM}
=
B\Big(\prod_{i=1}^N\dom\Omega_{\calP_i}\Big)
=
\bar\calM.
\]

Finally, let \(\bar\nu\in\relint(\bar\calM)\). By
\textcite[Theorem~26.5]{Rockafellar+1970},
\[
\bar w:=\nabla\Omega_{\bar\calM}(\bar\nu)
\quad\text{satisfies}\quad
\bar\nu=\nabla\bar F(\bar w).
\]
Define
\(
\pi_i
:=
\nabla\Omega_{\calP_i}^*
\big(A_i^\top J_{\Pi_{\bar\calW}}\bar w\big).
\)
Differentiating \(\bar F\) gives \(B\pi_\otimes=\bar\nu\).
Furthermore, Fenchel's equality
\parencite[Theorem~23.5]{Rockafellar+1970}, applied to each \(i\), gives
\[
\Phi(\pi_\otimes)
=
\langle\bar w,B\pi_\otimes\rangle-\bar F(\bar w)
=
\langle\bar w,\bar\nu\rangle-\bar F(\bar w)
=
\Omega_{\bar\calM}(\bar\nu).
\]
Thus \(\pi_\otimes\) attains the minimum in
\eqref{eq:expressionOfBarOmega}, with the claimed expression.
\end{proof}

\propfylandcalm*
\begin{proof}[Proof of Proposition~\ref{prop:FYLandCalM}]
    We first prove Equation~\eqref{eq:FYlandCalM_regsec}. 
    By definition, the parameter $w$ minimizes the empirical average of the Fenchel--Young losses. Expanding the loss, the objective to minimize is:
    $$ \frac{1}{N} \sum_{i=1}^N \calL_{\Omega_{\calP_i}}(A_i^\top w, \pi_i) = \frac{1}{N} \sum_{i=1}^N \Big( \Omega_{\calP_i}(\pi_i) + \Omega_{\calP_i}^*(A_i^\top w) - \langle A_i^\top w, \pi_i \rangle \Big). $$
    Because $\pi_i$ is fixed, minimizing this objective with respect to $w$ is equivalent to minimizing:
    $$ \frac{1}{N} \sum_{i=1}^N \Omega_{\calP_i}^*(A_i^\top w) - \Big\langle w, \frac{1}{N} \sum_{i=1}^N A_i \pi_i \Big\rangle. $$
    By introducing the aggregated moment $\nu = \frac{1}{N} \sum_{i=1}^N A_i \pi_i$, the unconstrained minimization problem for the full parameter $w$ is exactly the minimization of $\frac{1}{N} \sum_{i=1}^N \Omega_{\calP_i}^*(A_i^\top w) - \langle w, \nu \rangle$. 
    
    We restrict our attention to the identifiable parameter $\bar w = \Pi_{\bar \calW}(w)$ via the canonical injection $w = J_{\Pi_{\bar \calW}} \bar w$. The objective mapped to the identifiable space $\bar \calW$ becomes:
    $$ \underbrace{\frac{1}{N} \sum_{i=1}^N \Omega_{\calP_i}^*(A_i^\top J_{\Pi_{\bar \calW}} \bar w)}_{= \bar F(\bar w)} - \langle J_{\Pi_{\bar \calW}} \bar w, \nu \rangle. $$
    Using the adjoint property of the canonical injection $J_{\Pi_{\bar \calW}}$, we have $\langle J_{\Pi_{\bar \calW}} \bar w, \nu \rangle = \langle \bar w, \Pi_{\bar \calW}(\nu) \rangle = \langle \bar w, \bar \nu \rangle$. 
    Thus, $\bar w$ minimizes the strictly convex function $\bar F(\bar w) - \langle \bar w, \bar \nu \rangle$. 
    The first-order optimality condition yields:
    $$ \nabla \bar F(\bar w) - \bar \nu = 0 \implies \bar \nu = \nabla \bar F(\bar w). $$
    Because $\bar F$ is of Legendre type on the identifiable space, its gradient map is a bijection and we can invert it using the Fenchel conjugate $\Omega_{\bar \calM} = \bar F^*$. This gives $\bar w = \nabla \Omega_{\bar \calM}(\bar \nu)$, which is precisely $\Pi_{\bar \calW}(w) = \nabla \Omega_{\bar \calM}\big(\Pi_{\bar \calW}(\nu)\big)$, proving Equation~\eqref{eq:FYlandCalM_regsec}.

    We now prove Equation~\eqref{eq:crossJensenGapInCalM_regsec}.
    Let us evaluate the average Fenchel--Young loss on the identifiable space using $\bar w$:
    \begin{equation*}
    \begin{split}
        \frac{1}{N} \sum_{i=1}^N \calL_{\Omega_{\calP_i}}\big(A_i^\top J_{\Pi_{\bar \calW}}\bar w, \pi_i\big) &= \frac{1}{N} \sum_{i=1}^N \Omega_{\calP_i}(\pi_i) + \underbrace{\frac{1}{N} \sum_{i=1}^N \Omega_{\calP_i}^*(A_i^\top J_{\Pi_{\bar \calW}}\bar w)}_{= \bar F(\bar w)} \\
        &\quad - \Big\langle J_{\Pi_{\bar \calW}}\bar w, \frac{1}{N} \sum_{i=1}^N A_i \pi_i \Big\rangle.
    \end{split}
    \end{equation*}
    Substituting $\nu = \frac{1}{N} \sum_{i=1}^N A_i \pi_i$ and reusing the adjoint property $\langle J_{\Pi_{\bar \calW}}\bar w, \nu \rangle = \langle \bar w, \bar \nu \rangle$, the equation simplifies to:
    $$ \frac{1}{N} \sum_{i=1}^N \calL_{\Omega_{\calP_i}}\big(A_i^\top J_{\Pi_{\bar \calW}}\bar w, \pi_i\big) = \frac{1}{N} \sum_{i=1}^N \Omega_{\calP_i}(\pi_i) + \bar F(\bar w) - \langle \bar w, \bar \nu \rangle. $$
    From the first part of the proof, we established that $\bar w = \nabla \Omega_{\bar \calM}(\bar \nu)$ and $\bar \nu = \nabla \bar F(\bar w)$. Because $(\bar w, \bar \nu)$ form a conjugate pair linked by the gradient of the Legendre-type function $\bar F$, Fenchel's equality holds tightly:
    $$ \bar F(\bar w) + \Omega_{\bar \calM}(\bar \nu) = \langle \bar w, \bar \nu \rangle \implies \bar F(\bar w) - \langle \bar w, \bar \nu \rangle = - \Omega_{\bar \calM}(\bar \nu). $$
    Substituting this relation back into our expanded loss expression yields:
    $$ \frac{1}{N} \sum_{i=1}^N \calL_{\Omega_{\calP_i}}\big(A_i^\top J_{\Pi_{\bar \calW}}\bar w, \pi_i\big) = \frac{1}{N} \sum_{i=1}^N \Omega_{\calP_i}(\pi_i) - \Omega_{\bar \calM}(\bar \nu), $$
    which establishes Equation~\eqref{eq:crossJensenGapInCalM_regsec}, concluding the proof.
\end{proof}

We now specialize these results to the primal-dual setting of the paper, given in Section~\ref{subsec:convergence}.

\propcontextualregularization*
\begin{proof}[Proof of Proposition~\ref{prop:contextualRegularization}]
    The first four statements of the proposition follow directly as a corollary of Proposition~\ref{prop:barOmegaProperties}. 
    By setting the polytopes $\calP_i = \Delta^{\calY(x_i)}$, and the linear operators $A_i = \phi_i Y(x_i)$, the function $\bar{F}(\bar{w})$ defined in this section perfectly matches the aggregate dual regularizer $\bar{F}$ of Proposition~\ref{prop:barOmegaProperties}. 
    Since the restriction of $\Omega_{\Delta^{\calY(x_i)}}$ to its affine hull $H_{\Delta}$ is of Legendre type, it satisfies all structural requirements of Proposition~\ref{prop:barOmegaProperties}. 
    Consequently, $\bar{F}$ is of Legendre type, $\Omega_{\bar{\calM}} = \bar{F}^*$, the infimal convolution expression \eqref{eq:contextualRegularizationProp} holds, the domain is $\bar{\calM}$, and the minimum is attained at $q_i = \nabla \Omega_{\Delta^{\calY(x_i)}}^*(Y(x_i)^\top \phi_i^\top J_{\Pi_{\bar\calW}} \nabla \Omega_{\bar{\calM}}(\bar{\nu}))$.

    We now prove Equations~\eqref{eq:FYlandCalM} and~\eqref{eq:crossJensenGapInCalM}.
    These equations follow directly as a corollary of the Proposition~\ref{prop:FYLandCalM} using the same definition of $\calP_i$, $\Omega_{\calP_i}$, and $A_i$ as above. The unconstrained minimization of the coordination loss and the definition of the aggregated moment $\nu = \frac{1}{N} \sum_{i=1}^N \phi_i Y(x_i) q_i$ exactly match the premises of the proposition. Applying the proposition immediately yields the dual parameter mapping $\Pi_{\bar \calW}(w) = \nabla \Omega_{\bar \calM}\big(\Pi_{\bar \calW}(\nu)\big)$ of Equation~\eqref{eq:FYlandCalM} and the cross Jensen gap equality of Equation~\eqref{eq:crossJensenGapInCalM}, concluding the proof.
\end{proof}

\theoproximalpointoperator*
\begin{proof}[Proof of Theorem~\ref{theo:proximalPointOperator}]
    We proceed by induction. The base case $t=0$ holds by construction: both algorithms are initialized at $\bar \nu^{(0)} = \nabla \Omega_{\bar \calM}^*(\bar w^{(0)})$. Assume that at step $t$, $\bar \nu^{(t)} = \nabla \Omega_{\bar \calM}^*(\bar w^{(t)})$, where we denote the identifiable parameter as $\bar w^{(t)} = \Pi_{\bar \calW}(w^{(t)})$. Because $\Omega_{\bar \calM}$ is of Legendre type on the full-dimensional $\bar\calM$, we have $\bar w^{(t)} = \nabla \Omega_{\bar \calM}(\bar \nu^{(t)})$ and $D_{\Omega_{\bar \calM}}(\bar \nu \mid \bar \nu^{(t)}) = \calL_{\Omega_{\bar \calM}}(\bar w ^{(t)},
    \bar \nu)$.

    The update rule for the Bregman proximal algorithm is:
    $$ \bar \nu^{(t+1)} = \argmin_{\bar \nu \in \bar \calM}  f_\kappa(\bar \nu) + \kappa D_{\Omega_{\bar \calM}}\big(\bar \nu \mid \bar \nu^{(t)}\big)  = \argmin_{\bar \nu \in \bar \calM}  f_\kappa(\bar \nu) + \kappa \calL_{\Omega_{\bar \calM}}(\bar w ^{(t)},
    \bar \nu) . $$
    Let us now detail the minimization problem
    \begin{align*}
        &\min_{\bar \nu \in \bar \calM}  f_\kappa(\bar \nu) + \kappa \calL_{\Omega_{\bar \calM}}(\bar w ^{(t)},
        \bar \nu)\\
        &= \left|
            \begin{aligned}
                \min_{\bar \nu, \bar q} \frac{1}{N} \,&\sum_{i=1}^N \Big( \langle \gamma_i, q_i \rangle + \big(\kappa \Omega_{\Delta^{\calY(x_i)}}(q_i) + \kappa\overbrace{\Omega_{\Delta^{\calY(x_i)}}^*(Y_i^\top\phi_i^\top J_{\Pi_{\bar \calW}} \bar w^{(t)} )}^{\text{constant}} \\
                &\qquad - \kappa\langle Y_i^\top\phi_i^\top J_{\Pi_{\bar \calW}} \bar w^{(t)} | q_i \rangle \big) \Big) \\
                s.t.\colon\, & \frac{1}{N} \sum_{i=1}^N \Pi_{\bar \calW}\big(\phi_i Y(x_i) q_i\big) = \bar \nu 
            \end{aligned}    
            \right.
    \end{align*}
    where we have used the definition of $f_{\kappa}$ and $\Omega_{\bar \calM}^*(\bar w^{(t)})$ (Proposition~\ref{prop:barOmegaProperties}.1 and Equation~\eqref{eq:OmegaCalMstarDefinition}), brought the Fenchel--Young loss within the minimization in $q_{\otimes}$, merged the two minimizations in $\bar \nu$ and $q_{\otimes}$, and simplified $\kappa \Omega_{\bar \calM}(\bar \nu)$, which appears with a negative sign in $f_\kappa(\bar \nu)$ and a positive sign in the Fenchel--Young loss.
    Observe that since we minimize jointly in $q_{\otimes}$ and $\bar \nu$, we can just minimize in $q_{\otimes}$ without constraint and define $\bar \nu = \frac{1}{N} \sum_{i=1}^N \Pi_{\bar \calW}\big(\phi_i Y(x_i) q_i\big) $ a posteriori.
    Furthermore, the minimization decomposes by $i$, and we get
    $$ \begin{aligned}
        \tilde q_i &= \argmin_{q_i} \langle \gamma_i, q_i \rangle + \kappa\big( \Omega_{\Delta^{\calY(x_i)}}(q_i) - \langle Y_i^\top\phi_i^\top J_{\Pi_{\bar \calW}} \bar w^{(t)} | q_i \rangle \big) \\
        &= \argmin_{q_i \in \Delta^{\calY(x_i)}} \langle \gamma_i, q_i \rangle + \kappa\calL_{\Omega_{\Delta^{\calY(x_i)}}}(Y_i^\top\phi_i^\top J_{\Pi_{\bar \calW}} \bar w^{(t)},q_i) \\
        &= \argmin_{q_i \in \Delta^{\calY(x_i)}} \langle \gamma_i, q_i \rangle + \kappa \calL_{\Omega_{\Delta^{\calY(x_i)}}}(Y_i^\top\phi_i^\top w^{(t)},q_i)\\
        \bar \nu^{(t+1)} &= \frac{1}{N}\Pi_{\bar \calW} \sum_{i=1}^N \phi_iY_i\tilde q_i,\\
        \Pi_{\bar \calW}(\tilde w) &= \nabla \Omega_{\bar \calM }(\bar \nu^{(t+1)})
        \quad \text{where} \quad\tilde w \in \argmin \frac{1}{N}\sum_{i=1}^N \calL_{\Omega_{\Delta^{\calY(x_i)}}}(Y_i^\top\phi_i^\top w,\tilde q_i)
    \end{aligned} 
    $$
    where the first two minima are equal up to a constant by definition of Fenchel--Young losses, and the last two are again equal up to a constant by Proposition~\ref{prop:sub_dimensioal_domain_cvx_analysis}.2. The remaining equalities follow from Proposition~\ref{prop:contextualRegularization}. These are exactly the iterates of our alternating minimization, and we therefore obtain the induction hypothesis and the theorem.
\end{proof}
    
\propconvexityfkappa*
\begin{proof}[Proof of Proposition~\ref{prop:convexityfkappa}]
Set \(Aq_\otimes:=\bar\nu(q_\otimes)\) and define
\[
H(q_\otimes):=
\frac1N\sum_{i=1}^N\langle\gamma_i,q_i\rangle
+\kappa\left(
\frac1N\sum_{i=1}^N\Omega_{\Delta^{\calY(x_i)}}(q_i)
-\Omega_{\bar\calM}(Aq_\otimes)
\right)
+\bbI_{\Delta_\otimes}(q_\otimes).
\]
By assumption, the cross Jensen gap is convex; hence \(H\) is convex.
Moreover, by~\eqref{eq:definitionOfFkappa},
\(f_\kappa(\bar\nu)=\inf\{H(q_\otimes)\colon Aq_\otimes=\bar\nu\}\).
Thus \(f_\kappa\) is the image function of \(H\) under the linear map
\(A\), and is convex by \textcite[Theorem~5.7]{Rockafellar+1970}.
\end{proof}

\subsection{Proof of the surrogate error's bound}
\label{sec:bound_empirical_risk}

In this appendix, we prove Theorem~\ref{thm:bound_risk}. We first derive
an exact expression for the difference between the empirical cost and
the partial surrogate for one observation. Averaging this identity gives
the dataset-level bound, from which the comparison between the two
global minimizers follows.

\thmboundrisk*
\begin{proof}
We first consider a single observation and omit the dependence on \(x\)
and \(\xi\) to alleviate notation. Set
\(F:=\Omega_\Delta^*\) and \(s:=Y^\top\theta\).
As the partial surrogate~\eqref{eq:partial_surrogate} is obtained by
partially minimizing the surrogate~\eqref{eq:surrogate_risk_delta}, we have
\[
\begin{aligned}
\underline{\calS_{\Omega_\Delta}}(\theta)
&=
\min_{q\in\Delta^\calY}
\left\{
\langle\gamma\mid q\rangle
+
\kappa\big(
\Omega_\Delta(q)+F(s)-\langle s\mid q\rangle
\big)
\right\} \\
&=
\kappa F(s)
+
\kappa
\min_{q\in\Delta^\calY}
\left\{
\Omega_\Delta(q)-\langle s-\gamma/\kappa\mid q\rangle
\right\}
=
\kappa\Big(F(s)-F(s-\gamma/\kappa)\Big),
\end{aligned}
\]
where the last equality follows from the definition of the Fenchel
conjugate.

The policy associated with \(s\) is \(\nabla F(s)\), so its expected cost
is \(R_\Delta(\nabla F(s))=\langle\gamma\mid\nabla F(s)\rangle\)
(see~\eqref{eq:R_Delta}). Consequently,
\[
\begin{aligned}
R_\Delta(\nabla F(s))
-
\underline{\calS_{\Omega_\Delta}}(\theta)
&=
\kappa\left(
F(s-\gamma/\kappa)-F(s)
+
\langle\nabla F(s)\mid\gamma/\kappa\rangle
\right)
=
\kappa D_F(s-\gamma/\kappa\mid s).
\end{aligned}
\]
This quantity is nonnegative by convexity of \(F\). Since \(\nabla F\) is
\(1/L_x\)-Lipschitz-continuous, the descent lemma
(see \emph{e.g.}~\parencite[Lemma~3.1]{attouch_convergence_2013}), applied at
\(u=s-\gamma/\kappa\) and \(v=s\), gives
\(D_F(s-\gamma/\kappa\mid s)\leq\|\gamma/\kappa\|^2/(2L_x)\).
This proves~\eqref{eq:bound_risk_single}.

Applying this identity to each observation and averaging proves
\eqref{eq:bound_risk_sum}. In particular,
\(\underline{\calS_{\Omega_\Delta,N}}(\varphi_w)
\leq\calR_{\Omega_\Delta,N}(\varphi_w)\) for every \(w\in\calW\).
Finally, by optimality of \(w_{\calS}\) and the preceding one-sided
inequality,
\(
\underline{\calS_{\Omega_\Delta,N}}(\varphi_{w_{\calS}})
\leq
\underline{\calS_{\Omega_\Delta,N}}(\varphi_{w_{\calR}})
\leq
\calR_{\Omega_\Delta,N}(\varphi_{w_{\calR}}).
\)
It follows that
\[
\calR_{\Omega_\Delta,N}(\varphi_{w_{\calS}})
-
\calR_{\Omega_\Delta,N}(\varphi_{w_{\calR}})
\leq
\calR_{\Omega_\Delta,N}(\varphi_{w_{\calS}})
-
\underline{\calS_{\Omega_\Delta,N}}(\varphi_{w_{\calS}}),
\]
and~\eqref{eq:bound_empirical_risk} follows from
\eqref{eq:bound_risk_sum}.
\end{proof}

\subsection{Equivalence between tuning~$\kappa$ and tuning~$\varepsilon$ in the sparse perturbation case}\label{sec:kappa_eps_equivalence_appendix}

    We make precise the remark made in the proof of Proposition~\ref{prop:primal_dual_perturbation} (Appendix~\ref{subsec:tractable_proofs}) that the regularization weight~$\kappa$ and the perturbation scale~$\varepsilon$ ``play similar roles''. We show this in the non-contextual setting, where a single score vector~$\theta$ is shared by all scenarios: the two hyperparameters then act on the alternating scheme only through their product, so that fixing one to~$1$ and tuning the other explores exactly the same set of algorithm trajectories as tuning the other hyperparameter alone.

    \paragraph{Setting.}
    We place ourselves in the non-contextual alternating minimization scheme, with feasible set~$\calY$, costs $c_i(\cdot) := c(\cdot,\xi_i)$, and the sparse perturbation regularizer $\Omega_{\varepsilon,\calC} := F_{\varepsilon,\calC}^*$ shared by every scenario. Fix $\kappa>0$ and $\varepsilon>0$, and let $\big(\theta^{(t)}(\kappa,\varepsilon), \mu_i^{(t)}(\kappa,\varepsilon)\big)_{i\in[N],\,t\ge0}$ denote the iterates of this scheme, started from $\theta^{(0)}(\kappa,\varepsilon)=0$.

    \begin{proposition}\label{prop:kappa_eps_equivalence}
        For every $\lambda>0$, $i \in [N]$, and $t \geq 0$,
        \[ \mu_i^{(t)}(\kappa,\varepsilon) = \mu_i^{(t)}\Big(\frac{\kappa}{\lambda},\ \lambda\varepsilon\Big), \qquad \theta^{(t)}(\kappa,\varepsilon) = \frac{1}{\lambda}\,\theta^{(t)}\Big(\frac{\kappa}{\lambda},\ \lambda\varepsilon\Big). \]
        Consequently, the whole trajectory of the alternating scheme -- and hence the learned policies -- depends on $(\kappa,\varepsilon)$ only through the product~$\kappa\varepsilon$. In particular, fixing $\kappa=1$ and tuning~$\varepsilon$ explores exactly the same family of trajectories as fixing $\varepsilon=1$ and tuning~$\kappa$.
    \end{proposition}

    \begin{proof}
        \emph{Step 1: invariance of the primal update.} Fix $i\in[N]$, $t\ge0$, and let $\theta \in \bbR^d$. For every $\lambda>0$ and every realization of~$\bfZ$,
        \[
            c_i(y) - \kappa(\theta + \varepsilon \bfZ)^\top y \ = \ c_i(y) - \frac{\kappa}{\lambda}\big(\lambda\theta + \lambda\varepsilon \bfZ\big)^\top y, \qquad \forall y \in \calY,
        \]
        the two sides are \emph{identical} as functions of~$y$, so their $\argmin$ over $y \in \calY$, and thus its expectation under~$\bfZ$, coincide. Writing $\mu_i^{(t+1)}(\kappa,\varepsilon,\theta)$ for the corresponding primal update, we get
        \begin{equation}\label{eq:primal_invariance}
            \mu_i^{(t+1)}(\kappa,\varepsilon,\theta) \ = \ \mu_i^{(t+1)}\Big(\frac{\kappa}{\lambda},\ \lambda\varepsilon,\ \lambda\theta\Big), \qquad \forall \lambda>0.
        \end{equation}

        \emph{Step 2: scaling law of $F_{\varepsilon,\calC}$ and $\Omega_{\varepsilon,\calC}$.} By definition~\eqref{eq:perturbation_moment}, for every $\theta\in\bbR^d$ and $\varepsilon>0$,
        \[
            F_{\varepsilon,\calC}(\theta) = \bbE\big[\max_{y \in \calY}(\theta+\varepsilon \bfZ)^\top y\big] = \varepsilon\, \bbE\big[\max_{y \in \calY}(\theta/\varepsilon+\bfZ)^\top y\big] = \varepsilon\, F_{1,\calC}(\theta/\varepsilon).
        \]
        Taking Fenchel conjugates on both sides (using that conjugation of $\theta \mapsto \varepsilon g(\theta/\varepsilon)$ yields $\varepsilon g^*$ for any $g$),
        \[
            \Omega_{\varepsilon,\calC} = \varepsilon\,\Omega_{1,\calC}, \qquad \text{hence} \qquad \nabla\Omega_{\varepsilon,\calC |_V} = \varepsilon\,\nabla\Omega_{1,\calC|_V}.
        \]

        \emph{Step 3: dual update.} The coordination step reads, with $\bar\mu^{(t+1)} := \frac1N\sum_{i=1}^N \mu_i^{(t+1)}$,
        \[
            \theta^{(t+1)} \in \argmin_{\theta \in \bbR^d} \frac1N\sum_{i=1}^N\calL_{\Omega_{\varepsilon,\calC}}\big(\theta; \mu_i^{(t+1)}\big) = \frac1N\sum_{i=1}^N\Omega_{\varepsilon,\calC}\big(\mu_i^{(t+1)}\big) + F_{\varepsilon,\calC}(\theta) - \langle \theta \,|\, \bar\mu^{(t+1)}\rangle,
        \]
        using $\Omega_{\varepsilon,\calC}^* = F_{\varepsilon,\calC}$. Write $\theta = \theta_V + \theta_{V^\perp}$ with $\theta_V = \Pi_V(\theta)$. By Point~2 of Theorem~\ref{prop:SparsePerturbation}, for any fixed $y_0 \in \calC$,
        \[
            F_{\varepsilon,\calC}(\theta) = \langle y_0 \mid \theta_{V^\perp}\rangle + F_{\varepsilon,\calC}(\theta_V).
        \]
        Since $\bar\mu^{(t+1)} \in \calC \subset H = y_0 + V$, its $V^\perp$-component equals $y_0$, so $\langle \theta \mid \bar\mu^{(t+1)}\rangle = \langle \theta_V \mid \bar\mu^{(t+1)}\rangle + \langle \theta_{V^\perp} \mid y_0\rangle$. The two $\theta_{V^\perp}$-terms cancel in the objective, which reduces to minimizing $F_{\varepsilon,\calC}(\theta_V) - \langle \theta_V \mid \bar\mu^{(t+1)}\rangle$ over $\theta_V \in V$ alone. Since $F_{\varepsilon,\calC}|_V$ is Legendre-type (Theorem~\ref{prop:SparsePerturbation}, Point~4), the minimizer in $V$ is unique and given by gradient inversion; taking the canonical representative $\theta^{(t+1)} \in V$ and using Step~2,
        \[
            \theta^{(t+1)} = \nabla\Omega_{\varepsilon,\calC|_V}\big(\bar\mu^{(t+1)}\big) = \varepsilon\, \nabla\Omega_{1,\calC|_V}\big(\bar\mu^{(t+1)}\big) \in V.
        \]

        \emph{Step 4: induction.} We show by induction on~$t$ that $\mu_i^{(t)}(\kappa,\varepsilon) = \mu_i^{(t)}(1,\kappa\varepsilon)$ for every~$i$, and $\theta^{(t)}(\kappa,\varepsilon) = \theta^{(t)}(1,\kappa\varepsilon)/\kappa$. The base case $t=0$ holds since both trajectories start at~$0 \in V$. Note that Step~3 maps every iterate to $V$, so we consider $\theta^{(t)} \in V$ for all $t \geq 0$. Assume it holds at~$t$. By~\eqref{eq:primal_invariance} with $\lambda=\kappa$ and $\theta = \theta^{(t)}(\kappa,\varepsilon) = \theta^{(t)}(1,\kappa\varepsilon)/\kappa$,
        \[
            \mu_i^{(t+1)}(\kappa,\varepsilon) = \mu_i^{(t+1)}\Big(\kappa,\varepsilon,\ \tfrac{1}{\kappa}\theta^{(t)}(1,\kappa\varepsilon)\Big) = \mu_i^{(t+1)}\Big(1,\ \kappa\varepsilon,\ \theta^{(t)}(1,\kappa\varepsilon)\Big) = \mu_i^{(t+1)}(1,\kappa\varepsilon), \qquad \forall i\in[N].
        \]
        Averaging over~$i$ and applying Step~3 at $(\kappa,\varepsilon)$ and at $(1,\kappa\varepsilon)$ -- both evaluated, by the previous display, at the same $\bar\mu^{(t+1)}$ --,
        \[
            \theta^{(t+1)}(\kappa,\varepsilon) = \varepsilon\,\nabla\Omega_{1,\calC|_V}\big(\bar\mu^{(t+1)}\big), \qquad \theta^{(t+1)}(1,\kappa\varepsilon) = \kappa\varepsilon\,\nabla\Omega_{1,\calC|_V}\big(\bar\mu^{(t+1)}\big),
        \]
        so the two right-hand sides differ exactly by the factor~$\kappa$, \ie $\theta^{(t+1)}(\kappa,\varepsilon) = \theta^{(t+1)}(1,\kappa\varepsilon)/\kappa$, which closes the induction.

        Finally, for arbitrary $\lambda>0$, both $(\kappa,\varepsilon)$ and $(\kappa/\lambda,\lambda\varepsilon)$ have the same product $\kappa\varepsilon$, so the induction above (applied once to each pair against the common reference $(1,\kappa\varepsilon)$) yields $\mu_i^{(t)}(\kappa,\varepsilon)=\mu_i^{(t)}(1,\kappa\varepsilon)=\mu_i^{(t)}(\kappa/\lambda,\lambda\varepsilon)$, and similarly $\theta^{(t)}(\kappa,\varepsilon) = \theta^{(t)}(1,\kappa\varepsilon)/\kappa = \theta^{(t)}(\kappa/\lambda,\lambda\varepsilon)/\lambda$, which is the claimed identity.
    \end{proof}

    In particular, any point $(\kappa,\varepsilon)$ on the hyperbola $\kappa\varepsilon=c$ yields the same sequence of policies as the point~$(1,c)$, so restricting the search to $\kappa=1$ and tuning only~$\varepsilon$ entails no loss of generality on the set of policies reachable by the algorithm.

\section{Proof of convergence of the alternating minimization scheme}
\label{sec:convergenceProof}
This appendix starts by proving the convergence of Bregman PPA under some generic assumptions on the function minimized $f$ and the Bregman function $h$ in Section~\ref{sec:ConvergenceOfBregmanPPA}.
Section~\ref{sec:proofthm:convergence:speed} deduces from these results the convergence of our alternating minimization algorithm (Theorem~\ref{thm:convergence:speed}). 
Section~\ref{subsec:proposed-regularizers-satisfy-assumptions} finally shows that our regularizations satisfy the conditions of Theorem~\ref{thm:convergence:speed}, and are in particular real analytic (Proposition~\ref{prop:regsatisfyconv}).

\subsection{A generic Bregman proximal-point primal convergence result}
\label{sec:ConvergenceOfBregmanPPA}

Let \(\calU\subset\bbR^n\) be a non-empty compact convex set. Fix
\(u^{(0)}\in\operatorname{relint}(\calU)\), set
\(E:=\operatorname{aff}(\calU)-u^{(0)}\), and identify
\(\operatorname{aff}(\calU)\) with the Euclidean space \(E\) by translation.
For notational simplicity, we still write \(\calU\) for the translated set.
All interiors, boundaries, gradients, Hessians, and subdifferentials below are
understood in this relative geometry.

Let \(f:E\to\bbR\cup\{+\infty\}\) be proper and lower
semicontinuous, with \(\dom f=\calU\). Fix \(\kappa>0\), and consider a
sequence \((u^{(t)})_{t\ge0}\) and its dual sequence \((v^{(t)})_{t\ge0}\)
defined by
\begin{equation}
    \label{eq:generic-bregman-ppa}
    u^{(t+1)} \in \argmin_{u \in \calU}
    \{ f(u) + \kappa D_h(u \mid u^{(t)}) \},
    \qquad
    v^{(t)} := \nabla h(u^{(t)}).
\end{equation}

We assume:
\begin{itemize}
    \item \textbf{(H1) Non-degenerate local geometry:}
    \(h:E\to\bbR\cup\{+\infty\}\) is a closed proper convex Legendre-type
    function such that
    \(\operatorname{int}(\dom h)=\operatorname{relint}(\calU)\),
    \(\overline{\dom h}=\calU\), and \(\dom h^*=E\). Moreover, \(h^*\) is
    \(C^2\) on \(E\) and \(\nabla^2h^*(v)\succ0\) for every \(v\in E\).
    In particular, \(\nabla h^*\) maps \(E\) into
    \(\operatorname{relint}(\calU)\), and \(\nabla h\) is its \(C^1\) inverse
    on \(\operatorname{relint}(\calU)\).

    \item \textbf{(H2) Lipschitz dual map:}
    \(\nabla h^*\) is \(L\)-Lipschitz continuous. In particular,
\(h\) is \(1/L\)-strongly convex, hence
\(D_h(u\mid u')\ge 1/(2L)\|u-u'\|^2\), for $u,u'\in\operatorname{relint}(\calU)$.

    \item \textbf{(H3) Real analytic objective:}
    \(f\) is real analytic on \(\operatorname{relint}(\calU)\).

    \item \textbf{(H4) Interior proximal trajectory:}
    The sequence \((u^{(t)})_{t\ge0}\) stays in
    \(\operatorname{relint}(\calU)\).
\end{itemize}
Throughout, $\partial f$ denotes the limiting subdifferential
\parencite[Def.~8.3]{rockafellar1998variational}.

\begin{theorem}[Generic Bregman PPA convergence dichotomy]
\label{thm:generic_primal_convergence}
Assume \textnormal{(H1)}--\textnormal{(H4)}, and let \(\omega\) be the
accumulation set of \((u^{(t)})_{t\ge0}\). Then either
\(\omega\cap\operatorname{relint}(\calU)\neq\emptyset\), in which case
\((u^{(t)})\) converges with finite length to some
\(\tilde u\in\operatorname{relint}(\calU)\) with
\(0\in\partial f(\tilde u)\), and
\(f(u^{(t)})-f(\tilde u)=\mathcal O(1/t)\); or
\(\omega\subseteq\operatorname{rbd}(\calU)\), in which case
\(\|v^{(t)}\|\to+\infty\). In the first case, moreover,
\(v^{(t)}\to\nabla h(\tilde u)\) with finite length.
\end{theorem}

\paragraph{Proof strategy.}
The proof follows the Attouch--Bolte--Svaiter abstract convergence strategy
\parencite[Thm.~2.9]{attouch_convergence_2013}. 
Recall that \(f\) has the \emph{Kurdyka-\L ojasiewicz (KL)} property at \(\tilde u\) if there exist a
neighborhood \(U\) of \(\tilde u\), \(\eta>0\), and a concave \(C^1\)
desingularizer \(\varphi\), with \(\varphi(0)=0\) and \(\varphi'>0\), such
that
\(\varphi'(f(u)-f(\tilde u))\operatorname{dist}(0,\partial f(u))\ge1\)
whenever \(u\in U\) and \(f(\tilde u)<f(u)<f(\tilde u)+\eta\).
Attouch--Bolte--Svaiter \parencite{attouch_convergence_2013} assume that a function $f$ and sequence $(u^{(t)})_{t\geq 0}$ follows, for some $a,b > 0$ and every $t > 0$, a \emph{sufficient decrease (SD)} property $f(u^{(t)}) - f(u^{(t+1)}) \geq a ||u^{(t+1)} - u^{(t)}||^2$, a \emph{relative-error estimate (RE)} $\exists \xi^{(t+1)} \in \partial f(u^{(t+1)})$ with $||\xi^{(t+1)}|| \leq b ||u^{(t+1)} - u^{(t)}||$ and (KL). Combining (SD), (RE) and (KL) leads to finite length convergence toward a critical point.
In our setting, (SD) is global,
whereas (RE) is only local: it holds on compact subsets of
\(\operatorname{relint}(\calU)\), where \(\nabla h\) is Lipschitz. Thus, the theorem cannot be directly applied from the outset; instead, we first
localize the tail near an interior accumulation point, and repeat the standard KL capture argument:
starting from an iterate sufficiently close to \(\tilde u\), the KL estimate,
(SD), and local (RE) imply that the whole tail remains in the same interior
neighborhood and has finite length. Once this localization is obtained, the
usual KL convergence and value-rate arguments apply. The boundary alternative is
handled separately: if all accumulation points lie on \(\operatorname{rbd}(\calU)\),
then boundedness of the dual sequence would produce an interior accumulation
point through the map \(\nabla h^*\), a contradiction.

\begin{proof}[Proof of Theorem~\ref{thm:generic_primal_convergence}]
Set \(\Delta_{t+1}:=u^{(t+1)}-u^{(t)}\). Since \(f\) is lower
semicontinuous and finite on the compact set \(\calU\), it is bounded below.
Using \(u^{(t)}\) as a feasible point in \eqref{eq:generic-bregman-ppa} and
(H2), we obtain, with \(a:=\kappa/(2L)\),
\[
    f(u^{(t)})-f(u^{(t+1)})
    \ge \kappa D_h(u^{(t+1)}\mid u^{(t)})
    \ge a\|\Delta_{t+1}\|^2 .
    \tag{SD}
    \label{eq:generic_SD}
\]

Hence \(f(u^{(t)})\downarrow f_*:=\lim_t f(u^{(t)})\),
\(\sum_t\|\Delta_{t+1}\|^2<+\infty\), and \(\|\Delta_{t+1}\|\to0\).

Since \(u^{(t+1)}\in\operatorname{relint}(\calU)\), the first-order
optimality condition for \eqref{eq:generic-bregman-ppa}, together with the
exact sum rule (\parencite[Ex.~8.8]{rockafellar1998variational})
for the limiting subdifferential, gives
\[
    \xi^{(t+1)}
    :=
    \kappa\big(\nabla h(u^{(t)})-\nabla h(u^{(t+1)})\big)
    \in \partial f(u^{(t+1)}).
    \tag{RE0}
    \label{eq:generic_RE0}
\]
Moreover, by (H1), \(\nabla h\) is 
Lipschitz on every compact subset of
\(\operatorname{relint}(\calU)\).

Assume that there exists an accumulation point \(\tilde u\) of \((u^{(t)})\) in
\(\operatorname{relint}(\calU)\). 
Then by continuity of \(f\)
on \(\operatorname{relint}(\calU)\), \(f_*=f(\tilde u)\). Let \(\rho>0\) be
such that \(\bar B(\tilde u,2\rho)\subset\operatorname{relint}(\calU)\), and
let \(b\) be a Lipschitz constant of \(\kappa\nabla h\) on this ball. Whenever
\(u^{(t)},u^{(t+1)}\in\bar B(\tilde u,2\rho)\), \eqref{eq:generic_RE0}
therefore gives
\[
    \operatorname{dist}(0,\partial f(u^{(t+1)}))
    \le b\|\Delta_{t+1}\|.
    \tag{RE}
    \label{eq:generic_RE}
\]

By (H3), \(f\) is real analytic near
\(\tilde u\in\operatorname{relint}(\calU)\), hence it satisfies the KL
inequality there. Choose a KL neighborhood \(U\), a level \(\eta>0\), and a
desingularizer \(\varphi\). Shrinking \(\rho>0\) if necessary, assume
\(\bar B(\tilde u,2\rho)\subset U\cap\operatorname{relint}(\calU)\).

If \(f(u^{(t_0)})=f(\tilde u)\) for some \(t_0\), then
\eqref{eq:generic_SD} makes the sequence stationary from \(t_0\) onward, and
the capture conclusion is immediate. We therefore assume below that
\(f(u^{(t)})>f(\tilde u)\) for every \(t\).

Since \(\tilde u\) is an accumulation point, there is a subsequence
\(u^{(t_j)}\to\tilde u\). 
Along this subsequence,
\(\|\Delta_{t_j+1}\|\to0\),
\(f(u^{(t_j)})\to f(\tilde u)\), and
\(\varphi(f(u^{(t_j)})-f(\tilde u))\to0\). 
Hence we may choose
\(t_1=t_j\) large enough so that \(u^{(t_1)}\in B(\tilde u,\rho)\),
\(u^{(t_1+1)}\in B(\tilde u,2\rho)\),
\(0<f(u^{(t_1)})-f(\tilde u)<\eta\), and
\begin{equation}
    \|u^{(t_1)}-\tilde u\|
    +2\|\Delta_{t_1+1}\|
    +\frac ba\varphi(f(u^{(t_1)})-f(\tilde u))
    <2\rho .
    \tag{C}
    \label{eq:generic_capture_condition}
\end{equation}

We now prove the KL capture. Suppose that
\(u^{(t)},u^{(t+1)}\in\bar B(\tilde u,2\rho)\). Since
\(f(u^{(t+1)})>f(\tilde u)\), the KL inequality applies at \(u^{(t+1)}\).
Moreover, \(\|\Delta_{t+1}\|>0\): otherwise \eqref{eq:generic_RE} would give
\(\operatorname{dist}(0,\partial f(u^{(t+1)}))=0\), contradicting the KL
inequality because \(f(u^{(t+1)})>f(\tilde u)\).
Combining KL with \eqref{eq:generic_RE}, the concavity of \(\varphi\), and
the global sufficient decrease \eqref{eq:generic_SD} gives
\[
    \|\Delta_{t+2}\|^2
    \le
    \frac ba\|\Delta_{t+1}\|
    \big(
        \varphi(f(u^{(t+1)})-f(\tilde u))
        -
        \varphi(f(u^{(t+2)})-f(\tilde u))
    \big).
\]
The AM--GM inequality then yields
\[
    2\|\Delta_{t+2}\|
    \le
    \|\Delta_{t+1}\|
    +\frac ba
    \big(
        \varphi(f(u^{(t+1)})-f(\tilde u))
        -
        \varphi(f(u^{(t+2)})-f(\tilde u))
    \big).
    \tag{K}
    \label{eq:generic_KL_step}
\]

Let
\[
    T^*:=\sup\{T\ge t_1+1:
    u^{(s)}\in\bar B(\tilde u,2\rho)
    \text{ for all }t_1\le s\le T\}.
\]
Then \(T^*\ge t_1+1\). If \(T^*<+\infty\), summing
\eqref{eq:generic_KL_step} for \(t=t_1,\ldots,T^*-1\) gives
\[
    \sum_{s=t_1+1}^{T^*+1}\|\Delta_s\|
    \le
    2\|\Delta_{t_1+1}\|
    +\frac ba\varphi(f(u^{(t_1+1)})-f(\tilde u))
    \le
    2\|\Delta_{t_1+1}\|
    +\frac ba\varphi(f(u^{(t_1)})-f(\tilde u)).
\]
Therefore, by \eqref{eq:generic_capture_condition},
\[
    \|u^{(T^*+1)}-\tilde u\|
    \le
    \|u^{(t_1)}-\tilde u\|
    +
    \sum_{s=t_1+1}^{T^*+1}\|\Delta_s\|
    <2\rho,
\]
contradicting the maximality of \(T^*\). Hence \(T^*=+\infty\), the whole
tail remains in \(\bar B(\tilde u,2\rho)\), and as \(T^*=\infty\) in
the same estimate gives 
\(
\sum_{t\ge t_1}\|u^{(t+1)}-u^{(t)}\|<+\infty\).

Thus \(u^{(t)}\) is Cauchy and converges to \(\tilde u\), as a subsequence was already shown to converge to \(\tilde u\). Since
\(\xi^{(t+1)}\to0\), \(u^{(t+1)}\to\tilde u\), and
\(f(u^{(t+1)})\to f(\tilde u)\), the closedness of the limiting-subdifferential \parencite[Prop.~8.7]{rockafellar1998variational}
graph yields \(0\in\partial f(\tilde u)\).

It remains to record the value rate. If \(f(u^{(t_0)})=f(\tilde u)\) for some
\(t_0\), then \eqref{eq:generic_SD} makes the sequence stationary from \(t_0\)
onward. Otherwise, after the KL capture, the tail satisfies the standard
sufficient-decrease and relative-error assumptions in a fixed KL neighborhood.
Since \(f\) is analytic near \(\tilde u\), the KL inequality can be taken in
power form; equivalently, for some \(\theta\in[\tfrac12,1)\) and \(c>0\),
\(
    \operatorname{dist}(0,\partial f(u))
    \ge c\big(f(u)-f(\tilde u)\big)^\theta
\)
near \(\tilde u\). Applying the standard KL rate estimate
\parencite[Thm.~4]{frankel2015splitting} to
\(r_t:=f(u^{(t)})-f(\tilde u)\) gives a linear rate when
\(\theta=\tfrac12\), and
\(
    r_t=\mathcal O\!\left(t^{-1/(2\theta-1)}\right)
\)
when \(\theta>\tfrac12\). In particular, \(r_t=\mathcal O(1/t)\).

We now prove the boundary alternative. If
\(\omega\subseteq\operatorname{rbd}(\calU)\) and \((v^{(t)})\) had a bounded
subsequence, then, up to extraction, \(v^{(t_j)}\to\bar v\in E\). By (H1),
\(u^{(t_j)}=\nabla h^*(v^{(t_j)})\to\nabla h^*(\bar v)\in
\operatorname{relint}(\calU)\), contradicting
\(\omega\subseteq\operatorname{rbd}(\calU)\). Hence
\(\|v^{(t)}\|\to+\infty\).

Finally, in the capture case, the tail of \((u^{(t)})\) lies in a compact
subset of \(\operatorname{relint}(\calU)\), where \(\nabla h\) is Lipschitz.
Thus \(v^{(t)}=\nabla h(u^{(t)})\to\nabla h(\tilde u)\), and finite length of
\((u^{(t)})\) transfers to finite length of \((v^{(t)})\).
\end{proof}

\subsection{Proof of Theorem~\ref{thm:convergence:speed}}
\label{sec:proofthm:convergence:speed}

\thmconvergence*

To prove Theorem~\ref{thm:convergence:speed}, we apply Theorem~\ref{thm:generic_primal_convergence} with $\mathcal{U}=\bar{\mathcal{M}}$, the identifiable moment $\bar{\nu}$ a primal variable $u$, the identifiable parameter $\bar{w}$ dual variable $v$, the aggregate regularizer $\Omega_{\bar{\mathcal{M}}}$ as Bregman potential $h$, and $f_\kappa$ as $f$. 
By Theorem~\ref{theo:proximalPointOperator}, the sequence $\bar{\nu}^{(t)}$ generated by Algorithm~\ref{eq:alternating_product} matches the primal trajectory of the Bregman proximal point algorithm on $f_\kappa$, while the dual sequence matches $\bar{w}^{(t)}$. To apply Theorem~\ref{thm:generic_primal_convergence}, we must check that the generic conditions (H1)-(H4) hold for $\Omega_{\bar{\mathcal{M}}}$ and $f_\kappa$ under the assumptions of Theorem~\ref{thm:convergence:speed}. Moreover, we show that the dual sequence inherits the convergence properties of the primal one.
For every $i\in[N]$ write
$\Delta_i:=\Delta^{\calY(x_i)}$, $Y_i:=Y(x_i)$,
$\Omega_i:=\Omega_{\Delta_i}$, $\gamma_i:=\gamma(x_i,\xi_i)$, and
$A_i:=\phi_iY_i$, so that $Y_i^\top\phi_i^\top w=A_i^\top w$. 

\paragraph{Step 1: Verification of (H1), (H2) and (H4).} 

We first note that \(f_\kappa\), defined in
\eqref{eq:definitionOfFkappa}, satisfies the standing assumptions of
Theorem~\ref{thm:generic_primal_convergence}.

By~\eqref{eq:Gkappa}, \(G_\kappa\) is the infimal projection of a
lower-semicontinuous function over the compact set \(\Delta_\otimes\)
under a continuous linear constraint; it is therefore finite and lower
semicontinuous on \(\bar{\calM}\).
Moreover, Proposition~\ref{prop:barOmegaProperties} shows that
\(\Omega_{\bar{\calM}}\) is a lower-semicontinuous convex function with
domain equal to the polytope \(\bar{\calM}\); it is therefore continuous
relative to \(\bar{\calM}\) by~\parencite[Theorem~10.2]{Rockafellar+1970}.
Hence \(f_\kappa=G_\kappa-\kappa\Omega_{\bar{\calM}}\) is a proper
lower-semicontinuous function with domain \(\bar{\calM}\).

From Proposition~\ref{prop:barOmegaProperties}, $\dom(\Omega_{\bar \calM}) = \bar \calM$, and the Fenchel conjugate of our Bregman potential is $\Omega_{\bar{\mathcal{M}}}^*(\bar{w}) = \bar{F}(\bar{w}) = \frac{1}{N} \sum_{i=1}^N \Omega_{\calC_i}^*(\phi_i^\top J_{\Pi_{\bar\calW}} \bar{w})$. Because each $\Omega_{\calC_i}^*$ is twice continuously differentiable by specific assumption ($A_1$), their finite sum after linear precomposition, \(\bar F\), is also twice continuously differentiable. 
We now verify that $\nabla^2 \bar{F}(\bar{w}) \succ 0$ on $\bar{\calW}$. For any $u \in \bar{\calW}$, we have:
\[
    u^\top \nabla^2 \bar{F}(\bar{w}) u = \frac{1}{N} \sum_{i=1}^N (\phi_i^\top u)^\top \nabla^2 \Omega_{\calC_i}^*(\phi_i^\top J_{\Pi_{\bar\calW}} \bar{w}) (\phi_i^\top u) \ge 0.
\]
By assumption ($A_1$), $\nabla^2 \Omega_{\calC_i}^* \succ 0$ in the direction of $V_i$. Thus, the sum equals zero if and only if $\phi_i^\top u \in V_i^\perp$ for all $i \in [N]$. If $u \in \bar{\calW}$ satisfies this equality, then for any $v = \frac{1}{N}\sum_{i=1}^N \phi_i \theta_i \in \bar{\calW}$ (with $\theta_i \in V_i$), we have $\langle u, v \rangle = \frac{1}{N}\sum_{i=1}^N \langle \phi_i^\top u, \theta_i \rangle = 0$. This implies $u \in \bar{\calW} \cap \bar{\calW}^\perp = \{0\}$. Therefore, $\nabla^2 \bar{F} \succ 0$ on $\bar{\calW}$, which gives (H1). 
Furthermore, ($A_2$) gives that each $\nabla \Omega_i^*$ is Lipschitz continuous. Because \(\nabla\bar F\) is a finite sum of linear compositions of these maps, $\nabla \Omega_{\bar{\calM}}^* = \nabla \bar{F}$ is also Lipschitz continuous on $\bar{\calW}$, giving (H2). (H4) follows directly from Theorem~\ref{theo:proximalPointOperator}, where $\bar \nu^{(t)} = \nabla\Omega^*_{\bar\calM}(\bar w^{(t)}) \in \textrm{rel int}(\bar \calM)$ for all $t \geq 0$.

\paragraph{Step 2: Verification of (H3).}
We must show that $f_\kappa(\bar{\nu}) = G_\kappa(\bar{\nu}) - \kappa \Omega_{\bar{\calM}}(\bar{\nu})$ is real analytic on $\operatorname{relint}(\bar{\calM})$. 

We introduce the shifted dual potential
\begin{equation}\label{eq:shifted_dual_potential}
    \tilde{\bar{F}}(\bar{w}) := \frac{1}{N}\sum_{i=1}^N \Omega_i^*(A_i^\top J_{\Pi_{\bar\calW}} \bar{w} - \gamma_i/\kappa).
\end{equation}
Assumption ($A_3$) ensures that the maps $\theta \mapsto \Omega_i^*(Y_i^\top \theta - c)$ are real analytic for any constant shift $c$. Precomposing these with the linear mapping $\bar{w} \mapsto \phi_i^\top J_{\Pi_{\bar\calW}} \bar{w}$ preserves analyticity, meaning both $\bar{F}$ and $\tilde{\bar{F}}$ are real analytic on $\bar{\calW}$.

By definition of $f_\kappa$ in~\eqref{eq:definitionOfFkappa} and $ G_\kappa$ in~\eqref{eq:Gkappa}, we have $f_\kappa = G_\kappa - \kappa \Omega_{\bar{\calM}}$. Using Fenchel duality,
 \begin{align*}
    G_\kappa^*(\lambda) 
    = \sup_{q_\otimes \in \Delta_\otimes} \frac{1}{N} \sum_{i=1}^N \Big[ \langle A_i^\top \lambda, q_i \rangle - \langle \gamma_i, q_i \rangle - \kappa \Omega_i(q_i) \Big] 
    = \frac{\kappa}{N} \sum_{i=1}^N \Omega_i^*\Big(\frac{A_i^\top \lambda - \gamma_i}{\kappa}\Big).
 \end{align*}
Evaluating this at $\lambda = \kappa \bar{w}$ yields $G_\kappa^*(\kappa \bar{w}) = \kappa \tilde{\bar{F}}(\bar{w})$. Since $\tilde{\bar{F}}$ is real analytic, so is $G_\kappa^*$. As verified in Step 1, $\nabla^2 \bar{F}$ (and analogously by the shifted part of ($A_1$), $\nabla^2 \tilde{\bar{F}}$) are strictly positive definite, meaning their gradients have everywhere-nonsingular Jacobians.
By the analytic inverse function theorem
\parencite[I.B.7]{gunning2022analytic}, their respective conjugates $\Omega_{\bar{\calM}} = \bar{F}^*$ and $G_\kappa = (G_\kappa^*)^*$ are real analytic on $\operatorname{relint}(\bar{\calM})$. Consequently, their linear combination $f_\kappa = G_\kappa - \kappa \Omega_{\bar{\calM}}$ is real analytic, which gives (H3).

\paragraph{Step 3: Transfer of the primal guarantees to the dual iterates.}
We start with a lemma that transfers the convergence guarantees of the primal sequence to the dual one.
\begin{lem}[Dual identities and value transfer]\label{lem:value_transfer}
For every \(\lambda\),
\[
G_\kappa^*(\lambda)
=
\frac{\kappa}{N}\sum_{i=1}^N
\Omega_i^*\left(\frac{A_i^\top J\lambda-\gamma_i}{\kappa}\right)
=
\kappa\tilde{\bar F}(\lambda/\kappa).
\]
Moreover, for every \(t\ge0\),
\(
f_\kappa(\bar\nu^{(t+1)})
\le
\underline{S_{\Delta,N}}(\bar w^{(t)})
\le
f_\kappa(\bar\nu^{(t)}).
\)
Consequently, if
\(\bar\nu^{(t)}\to\bar\nu^*\in\operatorname{relint}(\bar{\calM})\) and
\(\bar w^*=\nabla\Omega_{\bar{\calM}}(\bar\nu^*)\), then
\(\underline{S_{\Delta,N}}(\bar w^*)=f_\kappa(\bar\nu^*)\), and any value rate
for \(f_\kappa(\bar\nu^{(t)})-f_\kappa(\bar\nu^*)\) transfers to
\(\underline{S_{\Delta,N}}(\bar w^{(t)})-\underline{S_{\Delta,N}}(\bar w^*)\).
\end{lem}
\begin{proof}
The formula for \(G_\kappa^*\) follows by eliminating the affine image variable
in~\eqref{eq:Gkappa}: 
\(G_\kappa^*(\lambda)=\sup_{q_\otimes\in\Delta_\otimes}
\tfrac1N\sum_i\{\langle A_i^\top J\lambda-\gamma_i,q_i\rangle
-\kappa\Omega_i(q_i)\}\), and the product structure of
\(\Delta_\otimes\) gives the stated sum of conjugates.
Using \(f_\kappa=G_\kappa-\kappa\Omega_{\bar{\calM}}\), Fenchel duality, and
the identity \(G_\kappa^*(\kappa\bar w)=\kappa\tilde{\bar F}(\bar w)\), we get
\[
\min_{\bar\nu\in\bar{\calM}}
\{f_\kappa(\bar\nu)+\kappa\calL_{\Omega_{\bar{\calM}}}(\bar w^{(t)},\bar\nu)\}
=
\underline{S_{\Delta,N}}(\bar w^{(t)}).
\]
Evaluating the minimum at
\(\bar\nu^{(t)}=\nabla\Omega_{\bar{\calM}}^*(\bar w^{(t)})\) gives the upper
bound, since the Fenchel--Young loss vanishes there. Moreover,
\(\calL_{\Omega_{\bar{\calM}}}(\bar w^{(t)},\bar\nu)
=D_{\Omega_{\bar{\calM}}}(\bar\nu\mid\bar\nu^{(t)})\), so the minimum is the
proximal objective in~\eqref{eq:proximal_point_f_kappa}, attained at
\(\bar\nu^{(t+1)}\). Positivity of the Bregman divergence gives the lower
bound. Since \(f_\kappa(\bar\nu^{(t)})\) decreases to \(f_\kappa(\bar\nu^*)\), the
sandwich implies
\(
\underline{S_{\Delta,N}}(\bar w^{(t)})\to f_\kappa(\bar\nu^*).
\)
Moreover,
\(\bar w^{(t)}=\nabla\Omega_{\bar{\calM}}(\bar\nu^{(t)})\to\bar w^*\), and
\(\underline{S_{\Delta,N}}\) is continuous. Hence
\(\underline{S_{\Delta,N}}(\bar w^*)=f_\kappa(\bar\nu^*)\). Finally,
\(
0\le
\underline{S_{\Delta,N}}(\bar w^{(t)})-\underline{S_{\Delta,N}}(\bar w^*)
\le
f_\kappa(\bar\nu^{(t)})-f_\kappa(\bar\nu^*),
\)
so any upper rate for the primal value gap transfers to the dual value gap.
\end{proof}

In the escape case, Theorem~\ref{thm:generic_primal_convergence} gives
\(\omega\subseteq\operatorname{rbd}(\bar{\calM})\) and
\(\|\bar w^{(t)}\|\to+\infty\). Since \(\bar{\calM}\) is compact and
\(\operatorname{rbd}(\bar{\calM})\) is closed, this is equivalent to
\(
\operatorname{dist}\big(\bar\nu^{(t)},\operatorname{rbd}(\bar{\calM})\big)
\to0.
\)

We now consider the confined case. The generic theorem gives
\(\bar\nu^{(t)}\to\bar\nu^*\in\operatorname{relint}(\bar{\calM})\) with finite
length. Since
\(\bar w^{(t)}=\nabla\Omega_{\bar{\calM}}(\bar\nu^{(t)})\), and
\(\nabla\Omega_{\bar{\calM}}\) is Lipschitz on a compact neighborhood of
\(\bar\nu^*\), the sequence \((\bar w^{(t)})\) also converges to
\(\bar w^*:=\nabla\Omega_{\bar{\calM}}(\bar\nu^*)\) with finite length.

By \eqref{eq:definitionOfFkappa} and~\eqref{eq:Gkappa},
\(f_\kappa=G_\kappa-\kappa\Omega_{\bar{\calM}}\). Since
\(\bar\nu^*\in\operatorname{relint}(\bar{\calM})\), the exact subdifferential
sum rule gives
\(0\in\partial f_\kappa(\bar\nu^*)\) iff
\(\kappa\bar w^*\in\partial G_\kappa(\bar\nu^*)\). By Fenchel duality and
differentiability of \(G_\kappa^*\), this is equivalent to
\(\bar\nu^*=\nabla G_\kappa^*(\kappa\bar w^*)\). Moreover,
Proposition~\ref{prop:barOmegaProperties} gives
\(\bar\nu^*=\nabla\Omega_{\bar{\calM}}^*(\bar w^*)=\nabla\bar F(\bar w^*)\),
while Lemma~\ref{lem:value_transfer} gives
\(\nabla G_\kappa^*(\kappa\bar w^*)=\nabla\tilde{\bar F}(\bar w^*)\). Hence
\(\nabla\bar F(\bar w^*)=\nabla\tilde{\bar F}(\bar w^*)\), which is exactly
\(\nabla\underline{S_{\Delta,N}}(\bar w^*)=0\).

\qed

\subsection{Proposed regularizers satisfy (A1)--(A3)}\label{subsec:proposed-regularizers-satisfy-assumptions}

We now show that our two main examples of regularization fall into the convergence framework above, thus proving Proposition~\ref{prop:regsatisfyconv}. 
The sparse-perturbation case requires the following two lemmas, which are of independent interest.
The only nonstandard point is real analyticity of the Gaussian-smoothed support function.

\begin{lem}[Gaussian non-degeneracy]
\label{lem:condition:a1:sparse}
$\nabla^2F_{\varepsilon,\calC}\succ0$ over $V$, the direction of
$\operatorname{aff}(\calC)$. More generally, for every continuous $c:\calC\to\bbR$, the shifted perturbed max $F_{\varepsilon,\calC,c}(\theta):=\bbE[\max_{y\in\calC}\{(\theta+\varepsilon\bfZ)^\top y-c(y)\}]$, where $\bfZ\sim \calN(O,I_d)$, has $\nabla^2F_{\varepsilon,\calC,c}\succ0$ over $V$.
\end{lem}
\begin{proof}
By \textcite[Prop.~3.1]{berthetLearningDifferentiablePertubed2020} (adapting
\textcite[Lemma~1.5]{abernethyPerturbationTechniquesOnline2016}),
$\nabla^2F_{\varepsilon,\calC}(\theta)=\frac1\varepsilon\bbE[y^*(\theta+\varepsilon\bfZ)\bfZ^\top]$
with $y^*(\theta)\in\argmax_{y\in\calC}\theta^\top y$ (any fixed measurable
selection; the tie set where the argmax is not a singleton is Lebesgue-null).
For $v\in V\setminus\{0\}$, $\|v\|=1$, decomposing $\bfZ=\bfr v+\bfZ_\perp$
with $\bfr\sim\calN(0,1)$ independent of $\bfZ_\perp$ reduces
$v^\top\nabla^2F_{\varepsilon,\calC}(\theta)v$ to
$\frac1\varepsilon\bbE_{\bfZ_\perp}\bbE_\bfr[g(\bfr)\bfr]$ with
$g(r):=v^\top y^*(\theta+\varepsilon\bfZ_\perp+\varepsilon rv)$ non-decreasing
in $r$ by monotonicity of the subdifferential of the support function
$F(\theta):=\max_{y\in\calC}\theta^\top y$
\parencite[Thm.~24.8]{Rockafellar+1970}; since $\bbE[\bfr]=0$ this makes
$\bbE_\bfr[g(\bfr)\bfr]\ge0$ pointwise in $\bfZ_\perp$, strictly so unless
$g$ is a.e.\ constant, which is excluded since $v\ne0$ and $V$ has
dimension $\ge1$ give distinct limits of $g$ at $\pm\infty$. The same argument applies to $F_{\varepsilon,\calC,c}$ by replacing the support function with the convex function $\theta\mapsto\max_{y\in\calC}\{\theta^\top y-c(y)\}$; the subdifferential is still monotone, and the linear term $\varepsilon rv^\top y$ dominates the bounded shift $c(y)$ as $r\to\pm\infty$.
\end{proof}

\begin{proposition}[Real analyticity of the perturbed max]
\label{prop:realanalytic}
Let $\varepsilon>0$, let $\calC\subset\bbR^d$ be non-empty compact, and let
$c:\calC\to\bbR$ be continuous. Then the shifted perturbed max, with  $\bfZ\sim \calN(O,I_d)$,
\[
    F_{\varepsilon,\calC,c}(\theta)
    :=
    \bbE\big[\max_{y\in\calC}\{(\theta+\varepsilon\bfZ)^\top y-c(y)\}\big]
\]
is real analytic on $\bbR^d$. In particular, taking $c\equiv0$,
$F_{\varepsilon,\calC}$ is real analytic on $\bbR^d$.
\end{proposition}

\begin{proof}
Define, 
for $\theta\in\bbC^d$,
\[
\tilde F_{\varepsilon,\calC}(\theta)
=
\frac1{(2\pi)^{d/2}\varepsilon^d}
\int_{\bbR^d}
F(u)
\exp\!\left(-\frac{(u-\theta)^\top(u-\theta)}{2\varepsilon^2}\right)
\,du, 
\]
with $F(u):=\max_{y\in\calC}\{u^\top y-c(y)\}$.
Note that, for $\theta\in\bbR^d$, $\tilde F_{\varepsilon,\calC}(\theta)=F_{\varepsilon,\calC,c}(\theta)$,
by the change of
variable $u=\theta+\varepsilon z$.

We first establish a compact-uniform domination bound. Since $\calC$ is
compact and $c$ continuous, $|F(u)|\le R\|u\|+\|c\|_\infty\le C_0(1+\|u\|)$
with $R:=\max_{y\in\calC}\|y\|<+\infty$, $\|c\|_\infty:=\max_{y\in\calC}|c(y)|<+\infty$,
and $C_0:=\max\{R,\|c\|_\infty\}$.
Let $K\subset\bbC^d$ be compact and write $\theta=\alpha+i\beta$, with
$\alpha,\beta\in\bbR^d$. On $K$, there exist constants
$A_K,B_K<+\infty$ such that $\|\alpha\|\le A_K$ and $\|\beta\|\le B_K$ for
all $\theta=\alpha+i\beta\in K$. 
By bilinearity, and since $u-\alpha,\beta\in\bbR^d$,
$(u-\theta)^\top(u-\theta)=\|u-\alpha\|^2-\|\beta\|^2
-2i(u-\alpha)^\top\beta .$
Hence
\[
\left|
\exp\!\left(-\frac{(u-\theta)^\top(u-\theta)}{2\varepsilon^2}\right)
\right|
=
\exp\!\left(\frac{\|\beta\|^2}{2\varepsilon^2}\right)
\exp\!\left(-\frac{\|u-\alpha\|^2}{2\varepsilon^2}\right).
\]
Moreover,
\(
\|u-\alpha\|^2
=
\|u\|^2-2u^\top\alpha+\|\alpha\|^2
\ge
\frac12\|u\|^2-\|\alpha\|^2
\ge
\frac12\|u\|^2-A_K^2.
\)
Therefore, uniformly for $\theta\in K$,
\(
\left|
F(u)
\exp\!\left(-\frac{(u-\theta)^\top(u-\theta)}{2\varepsilon^2}\right)
\right|
\le
C_K(1+\|u\|)
\exp\!\left(-\frac{\|u\|^2}{4\varepsilon^2}\right)
\)
for some finite constant $C_K$. The right-hand side is integrable on
$\bbR^d$. Hence $\tilde F_{\varepsilon,\calC}$ is well defined and continuous
on $\bbC^d$ by dominated convergence.

Fix $u\in\bbR^d$. Note that the map
\(
\theta\mapsto
\exp\!\left(-\frac{(u-\theta)^\top(u-\theta)}{2\varepsilon^2}\right)
\)
is entire on $\bbC^d$.
Thus, for $j\in[d]$, we have
\[
\frac{\partial}{\partial \theta_j}
\exp\!\left(-\frac{(u-\theta)^\top(u-\theta)}{2\varepsilon^2}\right)
=
\frac{u_j-\theta_j}{\varepsilon^2}
\exp\!\left(-\frac{(u-\theta)^\top(u-\theta)}{2\varepsilon^2}\right).
\]
On the same compact set $K$, let
$M_K:=\sup_{\theta\in K}\|\theta\|<+\infty$. Since
\(
|u_j-\theta_j|
\le
\|u\|+\|\theta\|
\le
\|u\|+M_K,
\)
the differentiated integrand is dominated by
\(
C_K'(1+\|u\|)(\|u\|+M_K)
\exp\!\left(-\frac{\|u\|^2}{4\varepsilon^2}\right),
\)
which is integrable on $\bbR^d$. Differentiation under the integral sign is
therefore justified, and
\[
\frac{\partial \tilde F_{\varepsilon,\calC}}{\partial \theta_j}(\theta)
=
\frac1{(2\pi)^{d/2}\varepsilon^d}
\int_{\bbR^d}
F(u)
\frac{u_j-\theta_j}{\varepsilon^2}
\exp\!\left(-\frac{(u-\theta)^\top(u-\theta)}{2\varepsilon^2}\right)
\,du.
\]
By Mattner's theorem on complex differentiation under the integral
\parencite{mattner2001complex}, applied coordinatewise, 
$\tilde F_{\varepsilon,\calC}$ is separately holomorphic. Since the compact-uniform
domination above also gives joint continuity, Osgood's lemma \parencite[I.A.2]{gunning2022analytic} 
yields joint
holomorphy on $\bbC^d$.

Consequently, $\tilde F_{\varepsilon,\calC}$ admits a convergent complex
Taylor expansion around every point of $\bbC^d$. Restricting this expansion
around each real point $\theta_0\in\bbR^d$ to real increments gives a
convergent real power-series expansion of $F_{\varepsilon,\calC,c}$ around
$\theta_0$. Hence $F_{\varepsilon,\calC,c}$ is real analytic on $\bbR^d$.

Finally, this yields (A\(_3\)) for the sparse perturbation: since
$\Omega^*_{\Delta}(s)=\bbE[\max_{y\in\calY}\{s_y+\varepsilon(Y^\top\bfZ)_y\}]$,
the composed map satisfies
$\Omega^*_{\Delta}(Y^\top\theta-c)=\bbE[\max_{y\in\calY}\{(\theta+\varepsilon\bfZ)^\top y-c_y\}]=F_{\varepsilon,\calY,c}(\theta)$,
real analytic for every fixed $c\in\bbR^{\calY}$ by the above.
\end{proof}

We can now prove Proposition~\ref{prop:regsatisfyconv}.

\regularizationconv*
\begin{proof}
\emph{Negentropy.} $\Omega_i^*$ is the log-sum-exp function; its gradient,
the softmax map, has Jacobian equal to a categorical covariance matrix,
hence operator norm at most $1$, so the softmax map is $1$-Lipschitz,
giving (A\(_2\)). The same Hessian is positive semidefinite with kernel
$\operatorname{span}(\one)$; since
$\Omega^*_{\calC_i}=\Omega^*_{\Delta_i}(Y_i^\top\cdot)$, for any direction
$v$ we get
$v^\top\nabla^2\Omega^*_{\calC_i}(\theta)v
=(Y_i^\top v)^\top\nabla^2\Omega^*_{\Delta_i}(Y_i^\top\theta)(Y_i^\top v)$,
which vanishes if and only if $Y_i^\top v\in\operatorname{span}(\one)$,
i.e., $\langle v,y-y'\rangle=0$ for all $y,y'\in\calY(x_i)$, i.e.,
$v\in V_i^\perp$. Hence $\nabla^2\Omega^*_{\calC_i}\succ0$ on $V_i$, which
is (A\(_1\)). The same kernel computation with $Y_i^\top\theta-c$ in place of $Y_i^\top\theta$ gives the shifted part of (A\(_1\)). Log-sum-exp is real analytic on
$\bbR^{\calY(x_i)}$, so for every $c$ the composed map
$\theta\mapsto\operatorname{logsumexp}(Y_i^\top\theta-c)$ is real analytic
(affine precomposition), which is (A\(_3\)).

\emph{Sparse perturbation.} (A\(_2\)) is
Proposition~\ref{prop:strongConvexitySparsePerturbation}; (A\(_1\)) is the
Gaussian non-degeneracy Lemma~\ref{lem:condition:a1:sparse} above including its shifted statement, and from Proposition~\ref{prop:realanalytic}, which gives
the required \(C^2\)-regularity; for
(A\(_3\)), the generalized Proposition~\ref{prop:realanalytic} gives, for
every $c\in\bbR^{\calY(x_i)}$, that
$\Omega^*_{\Delta_i}(Y_i^\top\theta-c)=F_{\varepsilon,\calY(x_i),c}(\theta)$
is real analytic.
\end{proof}

\section{Numerical experiments}\label{sec:appendix_numerical_exp}
 \subsection{Two-stage minimum weight spanning tree: problem and implementation details}\label{sec:appendix_spanning_tree_details}

    \paragraph{Problem formulation.}
    Let $G = (V, E)$ be an undirected graph, and let $\bfxi \in \Xi$ be an exogenous noise. For each edge $e \in E$ and scenario $\xi \in \Xi$, we denote by $c_e$ the scenario-independent first-stage cost of building $e$, and by $d_{e}(\xi)$ the scenario-dependent second-stage cost. The contextual stochastic two-stage minimum weight spanning tree problem is:

    \begin{equation}\label{eq:first_stage_min_weight_spanning_tree_md}
        \min_{\pi\in \calH}  \mathbb{E}_{\bfx, \bfy \sim \pi(\cdot | \bfx)} \bigg[\sum_{e \in E} c_e \bfy_e + \mathbb{E}_{\bfxi | \bfx} \big[ Q\big(\bfy; \bfxi\big)\big] \bigg],
    \end{equation}

    where the hypothesis class~$\calH$ is defined as
    \begin{align*}
        \calH &:=\{\pi \, | \, \pi(\cdot | x) \in \Delta^{\calY}\}, \quad \text{where} \\
        \calY &:= \{y \in \{0,1\}^{E} \, | \, \forall Y, \emptyset \subsetneq Y \subseteq V,  \sum\limits_{e \in E(Y)}y_e \leq |Y| - 1\}.
    \end{align*}
    The combinatorial space~$\calY$ corresponds to the first-stage solutions (forests on $G$). The second-stage recourse function $Q$ is defined as:

    \begin{subequations}
    \label{eq:second_stage_min_weight_spanning_tree_md}
    \begin{align}
    Q\big(y; \xi\big) := \min_z \quad & \sum_{e \in E}  d_{e}(\xi)z_e, \\
    \text{s.t.} \quad & \sum_{e \in E} y_e + z_{e} = |V| - 1,\label{eq:mtspt_first_tree_constraint_md}\\
    & \sum\limits_{e \in E(Y)}y_e + z_{e} \leq |Y| - 1, \quad  & \forall Y, \emptyset \subsetneq Y \subseteq V, \label{eq:mtspt_second_tree_constraint_md} \\
    & z_{e} \in \{0, 1\}, & \quad  \forall e \in E.\label{eq:mtspt_integer_second_stage_md}
    \end{align}
    \end{subequations}

    In problem~\eqref{eq:first_stage_min_weight_spanning_tree_md}, we look for \enquote{good} first-stage forests to be completed (in the second stage) into spanning trees, given a context realization $x \in \calX$.

    \begin{remark}
        In practice, the structure of the graph (and thus $\calY(x)$) can depend on the variable~$x$; we intentionally simplify the notation above.
    \end{remark}

    \paragraph{Policy model and algorithm.}

    \begin{figure}[h]
        \centering
        \includegraphics[width=0.9\textwidth]{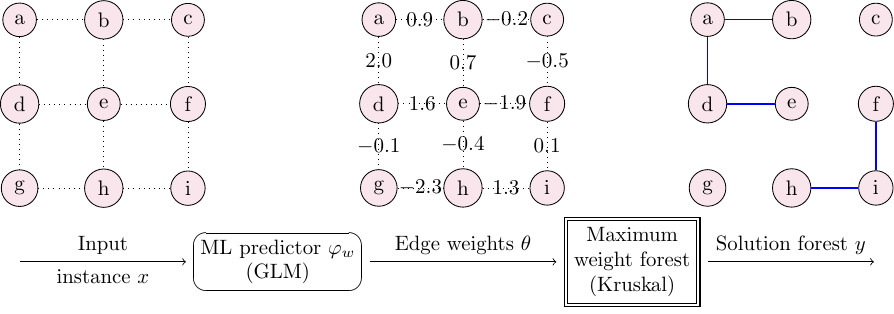}
        \caption{Two-stage minimum spanning tree: neural network pipeline.}\label{fig:two-stage-spanning-tree-pipeline_md}
    \end{figure}

    Each edge~$e$ of an instance (context)~$x$ is encoded by a feature vector~$\phi(x, e)$ detailed in the appendix to \textcite{dalleLearningCombinatorialOptimization2022}. The feature matrix is given as input to a generalized linear model (GLM)~$\varphi_w$, which predicts edge weights~$\theta_e$, illustrated in Figure~\ref{fig:two-stage-spanning-tree-pipeline_md}. We use the predicted edge weights~$\theta$ as the objective of a maximum weight forest problem; this problem is solved by Kruskal's algorithm, which plays the role of \texttt{oracle} in Algorithm~\ref{algo:primal_dual}.

    After tuning hyperparameters, we run $\texttt{nb\_iterations}=50$ outer iterations with perturbation scale $\varepsilon=10^{-4}$, and estimate expectations with $\texttt{nb\_samples}=20$ perturbation samples of $\bfZ$. The coordination step is trained for $\texttt{nb\_epochs}=30$ epochs with Adam initialized at $\texttt{lr\_init}=10^{-5}$.

\subsection{Contextual assortment: instance parameters and implementation details}
\label{sec:appendix_assortment_details}

\paragraph{Instance parameters.}
Each product feature vector has $d_p=3$ coordinates, consisting of price and two non-price features, and customer contexts have dimension $d_c=5$.
We report results over $5$ instance seeds and $5$ data seeds.
Each instance seed fixes the product catalog and preference matrix $B\in\mathbb{R}^{d_p\times d_c}$: product features are drawn uniformly on $[1,10]$, the first coordinate is interpreted as price, the price row of $B$ is fixed to $-0.7$, and the remaining entries are drawn uniformly on $[0,1]$.
For each data seed, we draw $100$ training customers, $100$ validation customers and evaluate on $1000$ independent test customers, with customer-context coordinates drawn from $\mathcal{N}(0,1)$.

\paragraph{Policy model and hyperparameters.}
The covariates consist of customer-context features and product features augmented by quantile-based summaries of the assortment instance.
The predictive model is a learned bilinear score: product-side and customer-side covariates are each mapped through a single \emph{tanh} layer into latent embeddings, whose interaction is parameterized by a learned bilinear form and passed through a \emph{softplus} link. 
Unless otherwise stated, we use $\kappa=1$, perturbation scale $\varepsilon=1$, $\texttt{nb\_samples}=10$, $\texttt{epochs}=30$, and learning rate $10^{-4}$.
In the scaling experiments, we report checkpoints at $T=50$ and $T=100$; validation-selected policies are denoted by $\bar{\pi}^T$.

\paragraph{Transfer experiment.}
For Table~\ref{tab:contextual_assortment_transfer}, the source and target catalogs both contain $10$ products.
We vary the share of target-catalog products that were unseen during training.
For SAA and kNN-SAA, prescriptions are restricted to old products that remain available in the target catalog.
The learned policies $\pi_{w^{(1)}}$ and $\bar{\pi}^{50}$ are evaluated on the full target catalog.

  \begin{table}[!ht]
  \centering
  \begin{tabular}{lcccccc}
    \toprule
    \textbf{Method} & \textbf{0\%} & \textbf{20\%} & \textbf{40\%} & \textbf{60\%} & \textbf{80\%} & \textbf{100\%} \\
    \midrule
    SAA & 7.14 & 7.03 & 7.03 & 6.18 & 5.41 & 0.00 \\
    kNN-10 & 7.27 & \textbf{7.35} & \textbf{7.14} & 6.24 & 5.40 & 0.00 \\
    $\pi_{w^{(1)}}$ & 4.54 & 4.56 & 3.85 & 4.20 & 3.77 & 3.63 \\
    $\bar{\pi}^{50}$ & \textbf{7.42} & 6.69 & 6.03 & \textbf{6.50} & \textbf{6.27} & \textbf{5.77} \\
    \bottomrule
  \end{tabular}
  \caption{Transfer benchmark as the share of unseen target-catalog products increases. Entries are mean test values pooled over $25$ instances.}
  \label{tab:contextual_assortment_transfer}
\end{table}
\end{appendices}
\end{document}